\documentclass{article} 
\usepackage{iclr2026_conference,times}
\iclrfinalcopy

\usepackage{amsmath,amsfonts,bm}

\def\eqref#1{equation~\ref{#1}}

\def\1{\bm{1}}

\DeclareMathAlphabet{\mathsfit}{\encodingdefault}{\sfdefault}{m}{sl}
\SetMathAlphabet{\mathsfit}{bold}{\encodingdefault}{\sfdefault}{bx}{n}

\usepackage{hyperref}
\usepackage{url}

\usepackage{subcaption}
\usepackage[utf8]{inputenc} 
\usepackage[T1]{fontenc}    
\usepackage{hyperref}       
\usepackage{url}            
\usepackage{booktabs}       
\usepackage{amsfonts}       
\usepackage{nicefrac}       
\usepackage{microtype}      
\usepackage{xcolor}         

\usepackage{graphicx}
\usepackage{booktabs}
\usepackage{colortbl}

\usepackage{amsthm} 
\usepackage{enumitem}
\usepackage{bbm}
\usepackage{amsmath,amssymb,amsfonts}
\usepackage{bm}

\usepackage{eqparbox}
\usepackage{graphicx}
\usepackage{xspace}
\usepackage{comment}
\usepackage{multirow}
\usepackage{verbatim}
\usepackage{arydshln}
\usepackage{tabularx}
\usepackage{tabulary}
\usepackage{booktabs}
\usepackage{makecell}
\usepackage{xspace}
\usepackage{circledsteps}
\usepackage{textcomp}
\usepackage{wasysym}

\usepackage{wrapfig}
\usepackage{algorithm}
\usepackage[noend]{algorithmic}

\definecolor{Skyblue}{rgb}{0.6, 0.6, 0.95 }
\definecolor{Green}{rgb}{0.0, 0.8, 0.0 }

\usepackage[hang,flushmargin]{footmisc}

\newcommand{\cmmnt}[1]{\ignorespaces} 

\newcommand{\algname}{\textsc{DualReform}}

\makeatletter
\newcommand*\bigcdot{\mathpalette\bigcdot@{.5}}
\newcommand*\bigcdot@[2]{\mathbin{\vcenter{\hbox{\scalebox{#2}{$\m@th#1\bullet$}}}}}
\makeatother

\usepackage{hyperref}

\usepackage{amsmath}
\usepackage{amssymb}
\usepackage{mathtools}
\usepackage{amsthm}

\usepackage{balance}

\usepackage[most]{tcolorbox}    	
\tcbuselibrary{theorems, skins, }
\newcounter{promptcounter}
\renewcommand{\thepromptcounter}{\arabic{promptcounter}}

\newtcolorbox{MyBox}[2][]{%
  enhanced,
  breakable,
  colback=gray!5,
  colframe=gray!80!black,
  boxrule=1pt,
  toprule=1.5pt,
  rounded corners,
  arc=2pt,
  top=0.2mm,
  bottom=0.2mm,
  left=3mm,
  right=3mm,
  fuzzy shadow={0pt}{-2pt}{-0.5pt}{0.5pt}{black!35},
  title={\normalsize Prompt~\thepromptcounter.~#2}, 
  #1 
}
\usepackage{amsthm} 
\theoremstyle{plain}
\newtheorem{theorem}{Theorem}[section]

\newtheorem{lemma}[theorem]{Lemma}

\newtheorem{definition}[theorem]{Definition}
\newtheorem{assumption}[theorem]{Assumption}
\theoremstyle{remark}

\title{References Indeed Matter?\\Reference-Free Preference Optimization for\\ Conversational Query Reformulation} 

\author{%
Doyoung Kim$^1$, Youngjun Lee$^1$, Joeun Kim$^1$, Jihwan Bang$^1$, Hwanjun Song$^1$, Susik Yoon$^2$,\\
\textbf{Jae-Gil Lee}$^1$\thanks{Corresponding author.}\\
$^1$ KAIST, $^2$ Korea University \\ 
{\tt\small \{dodokim,\,youngjun.lee,\,je.kim,\,jihwan.bang,\,songhwanjun,\,jaegil\}@kaist.ac.kr}, \\
{\tt\small susik@korea.ac.kr} 
}

\begin{document}
\maketitle
\begin{abstract}
Conversational query reformulation\,(CQR) has become indispensable for improving retrieval in dialogue-based applications. However, existing approaches typically rely on reference passages for optimization, which are \textit{impractical} to acquire in real-world scenarios. To address this limitation, we introduce a novel \textit{reference-free} preference optimization framework \algname{} that generates \textit{pseudo} reference passages from \textit{commonly-encountered} conversational datasets containing only queries and responses. \algname{} attains this goal through two key innovations: (1) \textit{response-based inference}, where responses serve as proxies to infer pseudo reference passages, and (2) \textit{response refinement via the dual-role of CQR}, where a CQR model refines responses based on the shared objectives between response refinement and CQR. Despite not relying on reference passages, \algname{} achieves 96.9--99.1\% of the retrieval accuracy attainable only with reference passages and surpasses the state-of-the-art method by up to 31.6\%.
\end{abstract}

\section{Introduction}
\label{sec:intro}

Retrieval-augmented generation (RAG) \cite{rag, selfrag, adaptive-rag, adaptive-rag2} is frequently employed to integrate external knowledge into the generation process of large language models\,(LLMs). One of the main components is to retrieve the passage most relevant to a specific query from an external data source. For this purpose, \textit{conversational query reformulation\,(CQR)}\,\cite{elgohary2019can, t5qr, qian2022explicit, vakulenko2021question, conqrr,  infocqr} is often used to facilitate the retrieval of the most relevant passage by reformulating the raw query.

In CQR, a query is reformulated using a language model\,(LM) which has generally been trained on a target conversational dataset. The training dataset comprises a collection of queries and corresponding responses, with each query linked to the \textit{reference passage} that represents the most ideal retrieval target\,\cite{qrecc, topiocqa}. As shown in Figure~\ref{fig:PO-CQR}(a), preference optimization\,\cite{rafailov2024direct} leverages the preferences over the candidates for the best reformulated query, where the rank of the reference passage in the retrieved passages for each candidate dictates the candidate's preference. For instance, since the reference passage is ranked higher for the candidate \textcircled{\scalebox{0.85}{A}} than for other candidates, a CQR model is fine-tuned to produce queries akin to \textcircled{\scalebox{0.85}{A}} during inference.


However, this \textit{reference-based} preference optimization\,\cite{retpo, adacqr} relies on an \textit{impractical} assumption that abundant reference passages are readily available. Most real-world conversational datasets, unfortunately, do not satisfy this assumption. Even worse, generating reference passages is very labor-intensive or expensive because annotators need to resolve coreferences\,(e.g., ``its'' in Figure~\ref{fig:PO-CQR}(a)) and apply domain-specific knowledge\,(e.g., ``gene editing'').

Therefore, in this paper, we introduce a novel \textit{reference-free} preference optimization framework, \algname{}, that eliminates the need for readily available reference passages. Instead, our approach generates \textit{pseudo} reference passages from \textit{commonly-encountered} conversational datasets, i.e., just a collection of queries and corresponding responses. Evidently, the primary challenge is accurately inferring pseudo reference passages, as the quality of pseudo supervision directly impacts model performance\,\cite{noisy_student, fixmatch, self-training-survery}. The novelty of our \algname{} framework lies in two ideas.


\begin{figure}[t!]
    \centering
    \captionsetup[subfigure]{skip=0pt}
    \begin{subfigure}[c]{0.49\linewidth}
        \includegraphics[width=\linewidth]{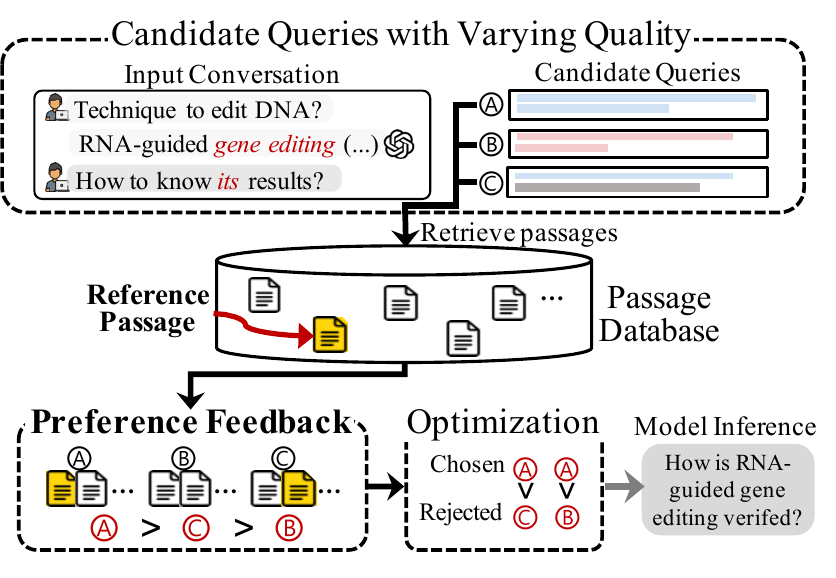}
        \caption{Preference optimization framework for CQR.}
    \end{subfigure}%
    \hfill
    \begin{subfigure}[c]{0.49\linewidth}
        \raisebox{0.16cm}{\includegraphics[width=\linewidth]{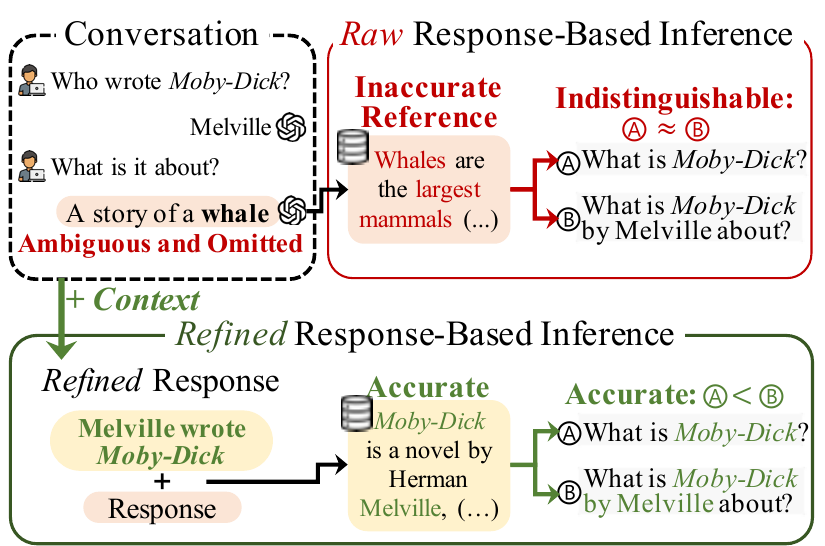}}
        \caption{Key idea of \algname{}.}
    \end{subfigure}
    \vspace*{-0.25cm}
    \caption{Overview of \algname{}. (a) Preference optimization framework for CQR with reference passages as a key component for generating preference feedback over candidate queries. (b) Key idea of \algname{}: Inferring pseudo reference passages through response refinement, addressing ambiguities and omissions in raw responses by incorporating conversational context.}
    \label{fig:PO-CQR}
    \vspace*{-0.6cm}
\end{figure}

\vspace*{-0.15cm}
(1) {\textbf{Response-Based Inference}}: While a (pseudo) reference passage is associated with a query, we propose to use its corresponding \textit{response} to infer its pseudo reference passage. A pseudo reference passage for a query is an external piece of information that enhances the quality of its response. Note that, for the majority of conversational datasets employed in training CQR models, the responses are already accessible, despite not being optimized by these passages. So, why not utilize the responses to infer pseudo reference passages? We assert that the responses serve as excellent proxies or weak supervisions for pseudo reference passages.

\vspace*{-0.15cm}
(2) {\textbf{Response Refinement by the \textit{Dual-Role} of CQR}}: Although the above idea appears intuitive, using raw responses may not yield accurate pseudo reference passages owing to potential ambiguities and omissions\,(e.g., ``A story of whale'' in Figure~\ref{fig:PO-CQR}(b)). Thus, we propose to use \textit{refined} responses rather than merely the raw responses. Such refined responses will clarify ambiguities and omissions by integrating pertinent conversational context (e.g., ``Melvile wrote Moby-Dick'' in Figure~\ref{fig:PO-CQR}(b)).

\vspace*{-0.1cm}
Here, we exploit the \textit{CQR model} to refine the raw responses. While the primary input for CQR is a query, this response refinement exactly aligns with the objective of CQR, namely, rephrasing a given query to enhance its relevance to the (pseudo) reference passage\,\cite{retpo}. \algname{} leverages this \textit{dual-role} of CQR, utilizing it for query reformulation during inference and for generating pseudo reference passages (through response reformulation) during training. One might argue that an LLM is suitable for this purpose, but we contend that taking advantage of the dual role offers numerous advantages. Most importantly, higher retrieval accuracy can be attained through preference optimization thanks to more accurate inference of pseudo reference passages. In addition, the overall procedure is simplified without an additional LLM; monetary cost is saved by not relying on commercial LLMs.

\vspace*{-0.1cm}
In summary, \algname{}, featured by the dual-role of CQR, eliminates the need for reference passages, thereby enhancing its applicability to various conversational datasets. As far as we know, this is the first work that addresses reference-free preference optimization for CQR. Despite not using reference passages at all, \algname{} demonstrates a retrieval accuracy remarkably close (96.9--99.1\%) to the optimal level only achievable with reference passages. Moreover, it outperforms the state-of-the-art methods with an average improvement of 15.7\% in retrieval accuracy.

\vspace*{-0.1cm}
\section{Related Work}
\label{sec:related_works}



\subsection{Preference Optimization}
\label{subsec:po}
Preference optimization methods\,\cite{schulman2017proximal, yang2024preference, rafailov2024direct, spo, cpo} aim to align the outputs of LMs with preferences by leveraging comparisons between outputs, such as rankings, rather than relying solely on labeled data or supervised targets. A prominent method is direct preference optimization\,\cite{rafailov2024direct}, which optimizes LMs without relying on an explicit reward model, resulting in a computationally efficient and robust framework. Please refer to extensive surveys\,\cite{wirth2017survey, jiang2024survey}.

\subsection{Conversational Query Reformulation}
\label{subsec:CQR}

CQR methods typically employ LMs to generate self-contained queries by incorporating conversational contexts in three directions. 

In prompt engineering methods, LLM-IQR\,\cite{infocqr} and LLM4CS\,\cite{llm4cs} leverage LLMs to generate self-contained queries by carefully designing prompts that extract the relevant context from conversation. HyDE\,\cite{hyde} extends this direction by generating synthetic passages related to the query. 

In supervised fine-tuning methods, T5QR\,\cite{t5qr} fine-tunes a T5-base model\,\cite{t5} for query reformulation using human-annotated reformulated queries. ConvGQR\,\cite{convgqr} additionally fine-tunes an LM for query expansion, augmenting queries with potential responses, while aligning query embeddings to reference passages. However, they require creating high-quality reformulated queries, which are labor-intensive\,\cite{song2024learning} and often misalign with the retrieval objective\,\cite{retpo}.

Addressing this issue, preference optimization methods, such as RetPo\,\cite{retpo} and AdaCQR\,\cite{adacqr}, align the LM with the retrieval objective by leveraging preferences among candidates for reformulated queries. However, they assume the existence of abundant reference passages, prohibiting their utilization for most conversational datasets that lack such references.

\section{Preliminaries}
\label{sec:preliminaries}

\subsection{Conversational Query Reformulation}
\label{subsec:cqr}

A conversational session is represented as a sequence of query-response turns \( \mathcal{T} \!=\! \{(x_t, a_t, {G}_t)\}_{t=1}^N \), where \( x_t \!=\! (\mathcal{H}_{<t}, q_t) \) is the input comprising the query-response history \( \mathcal{H}_{<t}\!=\! \{(q_i, a_i)\}_{i=1}^{t-1} \) and the current query \( q_t \). Here, \(a_t\) is the response to \( q_t \), and \( {G}_t \!=\! \{g_t^j\}_j \) is the set of reference passages to \( q_t \).

Given an input \( x_t \), CQR excutes a query reformulation function \({\rm CQR}(\cdot; \theta)\), parametrized by a model \( \theta \), to generate a self-contained query, which is then passed to a retrieval system \( R(\cdot) \) to retrieve relevant passages, i.e., \( \hat{G}_t \!=\! R({\rm CQR}(x_t; \theta)) \). Formally, for a conversation session \( \mathcal{T} \), the goal of CQR is to maximize the reformulation quality,
\begin{equation}
\label{eq:cqr}
J_{\theta}(\mathcal{T}) = \frac{1}{|\mathcal{T}|} \sum_{t=1}^{|\mathcal{T}|} \mathcal{M}(\hat{G}_t, G_t),
\end{equation}
where \( \mathcal{M}(\hat{G}_t, G_t) \) is a metric (e.g., Recall@\(k\)) that evaluates the retrieval quality by comparing the retrieved passages \( \hat{G}_t \) with \( {G}_t\).

\subsection{Reference-Based Preference Optimization for CQR}
\label{subsec:supervised preference optimization}

\noindent \textbf{Preference Feedback Generation.}
Preference optimization methods rely on the reference passages \( G_t \) to produce \textit{preference feedback}.
For a given input \( x_t \), an LLM is used to produce a set of candidate query reformulations \( \{ \tilde{q}_t^i \}_{i=1}^M \) with varying quality. The preference feedback \( {\rm pref}(G_t) \) is then defined as a sorted list of these candidates based on their relevance to the reference passages \( G_t \),
\begin{equation}
\label{eq:pref}
{\rm pref}(G_t) = {\rm sort}\big(\{\tilde{q}_t^i\}_{i=1}^M, \; \text{by decreasing } s(\tilde{q}_t^i \mid G_t)\big),
\end{equation}
where \(s(\tilde{q}_t^i \mid G_t)\) is the retrieval score, indicating how accurately \(\tilde{q}_t^i\) retrieves passages containing \(G_t\).

\smallskip
\noindent \textbf{Preference Optimization.}
Preference optimization proceeds through two steps: supervised fine-tuning\,(SFT) and direct preference optimization\,(DPO).
First, the SFT step trains the model $\theta$ on the top-ranked one \( \tilde{q}_t^1 \) from \( {\rm pref}(G_t) \), by minimizing the negative log-likelihood,
\begin{equation}
\label{eq:sft}
\ell_{\text{sft}} (x_t, G_t; \theta) = - \mathbb{E}_{\tilde{q}_t^i \sim {\rm pref}(G_t)} \left[ \mathbbm{1}_{[i = 1]} \log P(\tilde{q}_t^i \mid {x}_t; \theta) \right],
\end{equation}
where \( P(\tilde{q} \mid {x}; \theta) \) is the probability of \( \tilde{q} \) given the input ${x}$. 
Next, the DPO step optimizes the model to learn pairwise preferences from reformulation pairs \( (\tilde{q}_t^i, \tilde{q}_t^j) \) such that \( i < j \) in \( {\rm pref}(G_t) \), by maximizing the preference likelihood,
\begin{equation}
\label{eq:dpo}
\ell_{\text{pref}} (x_t, G_t; \theta) = \mathbb{E}_{\tilde{q}_t^i, \tilde{q}_t^j \sim {\rm pref}(G_t)} \left[ \mathbbm{1}_{[i < j]} {\rm{r}} (\tilde{q}_t^i, \tilde{q}_t^j; x_t, \theta) \right],
\end{equation}
where \( {\rm{r}} (\tilde{q}_t^i, \tilde{q}_t^j; x_t, \theta) \) represents the likelihood that the model $\theta$ ranks \( \tilde{q}_t^i \) higher than \( \tilde{q}_t^j \).

\section{\algname{}: ``Reference-Free'' Preference Optimization}
\label{sec:methodology}

\subsection{Problem Statement}
\label{subsec:problem}

Our \textit{reference-free} preference optimization framework, \algname{}, accommodates \textit{commonly-encountered} scenarios where a conversation \(\mathcal{T}_U = \{(x_t, a_t)\}_{t=1}^{N}\) does \emph{not} include reference passages \({G}_t\). Instead, \algname{} generates \textit{pseudo} reference passages \(\tilde{{G}}_t\) to establish $\mathcal{T}_P =\{(x_t, a_t, \tilde{{G}}_t)\}_{t=1}^{N}$ to enable preference optimization. Then, the key challenge is how to accurately build the set of pseudo reference passages \(\tilde{\mathcal{G}} = \{\tilde{{G}}_t\}_{t=1}^{N} \) such that the preference-optimized model \(\theta_{\tilde{\mathcal{G}}}\) maximizes the retrieval performance on a target dataset, i.e.,
\begin{equation}
\label{eq:weak-sup-obj}
\tilde{\mathcal{G}}^* = \underset{\tilde{\mathcal{G}}}{\arg\max} \, J_{\theta_{\tilde{\mathcal{G}}}}(\mathcal{T}_P).
\end{equation}

\subsection{Response Refinement by CQR's Dual-Role}
\label{subsec:response_refinement}

\begin{wrapfigure}{R}{0.5\linewidth}
\vspace{-0.7cm}
\begin{center}
\includegraphics[width=7.0cm]{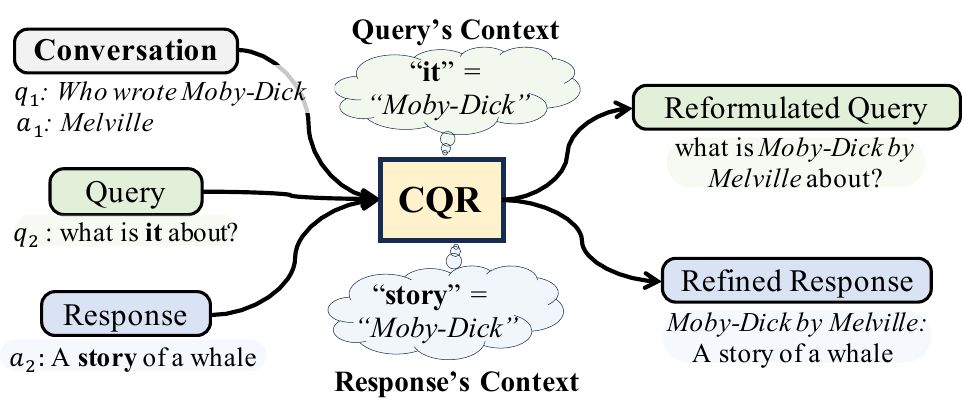}
\end{center}
\vspace{-0.5cm}
\caption{Dual-role of the CQR model.}
\vspace{-0.35cm}
\label{fig:dual_role}
\end{wrapfigure}

Because using raw responses harms the quality of pseudo reference passages,
we leverage the CQR model not only for query reformulation but also for \textit{response refinement}, thus introducing its \textit{dual role}. As shown in Figure \ref{fig:dual_role}, the CQR model identifies and integrates the key context\,(e.g., ``\textit{Moby-Dick}'') from the conversation\,\citep{convgqr, retpo}, demonstrating a capacity advantageous for both query reformulation and response refinement. This dual role is a natural extension arising from the inherent alignment between the objective of response refinement and the retrieval objective in Eq.~(\ref{eq:cqr}), where both aim to maximize the relevance to underlying reference passages.



\begin{wraptable}{r}{0.37\textwidth}
\vspace{-1.3em} 
\caption{Comparison of single- and dual-role configurations. LLM and CQR use a common language model.
} 
\def\arraystretch{0.7}
\centering
\vspace*{-0.25cm}
\resizebox{0.99\linewidth}{!}{%
\begin{tabular}[c]
{@{}c|cccc@{}}
\arrayrulecolor{black}\specialrule{1.5pt}{0.75pt}{2.5pt}
\multirow{2}{*}{\makecell[c]{Roles}} 
& \multicolumn{1}{c}{{ \textbf{Single Role}}} 
& \multicolumn{1}{c}{{ \textbf{Dual Role}}} 
\\ 
& \multicolumn{1}{c}{{Llama}}
& \multicolumn{1}{c}{{\textbf{\algname{}}}}
\\
\arrayrulecolor{black}\specialrule{1pt}{1.5pt}{1pt}
\arrayrulecolor{black}\specialrule{1pt}{1pt}{4.0pt}
\multirow{1}{*}[0.3em]{\textbf{Response Ref.}}
& \begin{tabular}[c]{@{}c@{}}{ {LLM} }\end{tabular}
&  \multirow{1}{*}[0.2em]{ {\textbf{CQR}} } 
\\ \addlinespace[0.5ex]\cdashline{1-4}\addlinespace[0.9ex]
\multirow{1}{*}[0.3em]{{Query Ref.}}                                
& \multicolumn{2}{c}{{CQR}} 
\\ \arrayrulecolor{black}\specialrule{1.5pt}{1pt}{1.0pt}
\end{tabular} 
}
\vspace*{-0.23cm}
\label{tbl:cqr_config}
\vspace*{-0.3cm}
\end{wraptable} 
We validate the effectiveness of the CQR's dual-role by comparing \textit{single-} and \textit{dual-role} configurations. As presented in Table~\ref{tbl:cqr_config}, the single-role configuration employs an LLM for response refinement, whereas the dual-role configuration utilizes the CQR model.

\noindent \textbf{Theoretical Evidence.} 
Refining responses with a CQR model, optimized for high retrieval accuracy of reference passages, yields more accurate pseudo reference passages, thereby improving the CQR's retrieval accuracy. We provide a formal justification of this optimization using the generalization bound from the pseudo label denoising theory\,\citep{wei2021theoretical}.
\begin{assumption}[\sc {Expansion and Separation\,\citep{wei2021theoretical}}]
We assume (i) $c$-expansion: each example is, on average, reachable to $c$ neighbors, i.e., $\mathbb{E}_{x} [ |\mathcal{N}(x)| ]=c$ with $c > 3$, where $\mathcal{N}(x)$ denotes the neighborhood of $x$; and (ii) $\mu$-separation: the average proportion of neighbors requiring different reference passages is $\mu$, which is negligibly small\,(e.g., $1/poly(d)$, the inverse of a polynomial in dimension). 
\label{assume:expansion}
\end{assumption}
These assumptions, commonly used in prior studies\,\citep{wei2021theoretical,park2023robust,cai2021theory}, consider data distributions in which the examples with the same reference passages are close by $c$-expansion whereas those with different reference passages are well separated by $\mu$-separation.
Under these assumptions, we derive the training error bounds for the single- and dual-role CQR models in Lemmas~\ref{lemma:err-single} and~\ref{lemma:err-dual}.

\begin{lemma}[\textsc{Single-Role Bound}]
\label{lemma:err-single}
Suppose that Assumption~\ref{assume:expansion} holds. Then, the training error of a CQR model $\theta_{single}$ trained under the single-role configuration is bounded by the quality of pseudo reference passages generated by the LLM such that
\begin{equation}
\begin{gathered}
{Err}(\theta_{single}) \le {{2}\over{c-1}} {Err}(\theta_{LLM}) + {{2c}\over{c-1}} \mu.
\end{gathered}
\label{eq:error_bound_single}
\end{equation}
\end{lemma} 

\begin{lemma}[\sc {Dual-Role Bound}]
\label{lemma:err-dual}
Suppose that Assumption~\ref{assume:expansion} holds. Then, the training error of a dual-role CQR model $\theta_{dual}$ under the dual-role configuration is bounded by
\begin{equation}
\begin{gathered}
{Err}(\theta_{dual}) \le \left({{2}\over{c-1}}\right)^2 {Err}(\theta_{LLM}) + \left({{2c}\over{c-1}}\right)\left({{c+1}\over{c-1}}\right) \mu.
\end{gathered}
\label{eq:error_bound_dual}
\end{equation}
\end{lemma} 
These bounds extend the pseudo label denoising theorem\,\citep{wei2021theoretical}, with complete proofs presented in Appendix~\ref{subsec:lemma}. Combining these lemmas, we compare the two configurations in Theorem~\ref{thm:comparison}.
\begin{theorem}[\textsc{Error Bound Difference}] 
Let $\overline{Err}(\theta)$ denote the theoretical upper bound on the training error for the model $\theta$. Under Assumption~\ref{assume:expansion}, the bound of the dual-role configuration is smaller than that of the single-role configuration, i.e., $\overline{Err}(\theta_{dual}) < \overline{Err}(\theta_{single})$.
\label{thm:comparison} 
\end{theorem}
\begin{proof}
\vspace*{-0.3cm}
Because \(\mu\) is negligible and \( ( \frac{2}{c - 1} )^2 < \frac{2}{c - 1}\) for \(c > 3\), the dual-role configuration achieves a smaller upper bound. The complete proof is available in Appendix~\ref{subsec:thm}.
\vspace*{-0.2cm}
\end{proof}

\begin{wrapfigure}{R}{0.48\linewidth}
\vspace{-0.7cm}
\begin{center}
\includegraphics[width=7.0cm]{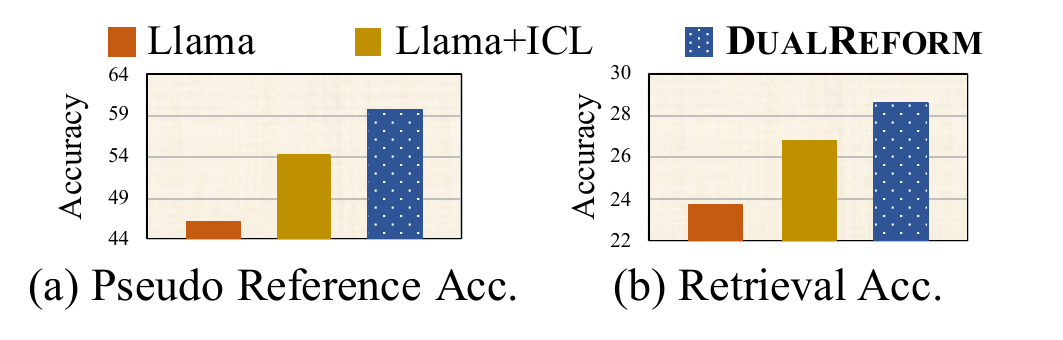}
\end{center}
\vspace{-0.5cm}
\caption{Empirical comparison of single- and dual-role variants. All variants employ Llama3.1-8b-inst as their backbones and use the prompt detailed in Prompt~\ref{prompt_temp}. \textit{Llama+ICL} variant additionally employs \emph{in-context learning}.
}
\vspace{-0.35cm}
\label{fig:sum-result}
\end{wrapfigure}

\noindent \textbf{Empirical Evidence.} 
Figure~\ref{fig:sum-result} empirically supports Theorem~\ref{thm:comparison}, showing that the refined responses by the CQR model yield more accurate pseudo references and higher retrieval accuracy than those by the single-role variants. \textit{Pseudo reference accuracy}, in Figure~\ref{fig:sum-result}(a), assesses response refinement by measuring the agreement between pseudo and ground-truth\protect\footnotemark~reference passages; and \textit{retrieval accuracy} in Figure~\ref{fig:sum-result}(b), as defined in Eq.~(\ref{eq:cqr}), assesses query reformulation by comparing passages retrieved via CQR against ground-truth reference passages. In the next section (see Figure~\ref{fig:refine_exs}), we will further demonstrate that the CQR model focuses on the context relevant to the response, while the single-role variant introduces less relevant context (e.g., including already mentioned movies when asking for unmentioned ones). 

\definecolor{deepblue}{RGB}{0, 0, 110}
\definecolor{red}{RGB}{230, 0, 0}

  
  

\footnotetext{The ground-truth information is used only for evaluation purposes, but it is \textit{not} used in \algname{}.}

\begin{wrapfigure}{R}{0.5\linewidth}
\vspace{-0.7cm}
\begin{center}
\includegraphics[width=7.0cm]{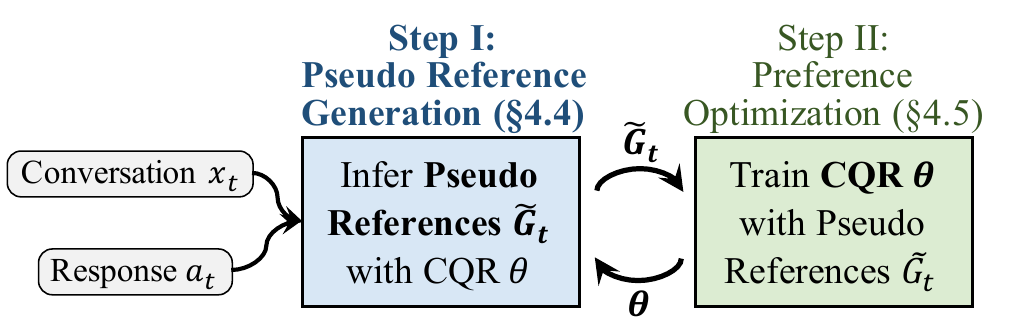}
\end{center}
\vspace{-0.5cm}
\caption{Overall flow of \algname{}.}
\vspace{-0.35cm}
\label{fig:our-flow}
\end{wrapfigure}

\subsection{Overview of \algname{}}
\label{subsec:overview}
Figure~\ref{fig:our-flow} illustrates the reference-free preference optimization framework of \algname{}, driven by the CQR's dual role. It iteratively alternates between the two roles: as a \textit{response refiner}, the CQR model helps generate pseudo reference passages \(\tilde{G}_t\)~($\S$~\ref{subsec:pseudo-generation}), which subsequently guide the optimization of the CQR model \(\theta\) as a \textit{query reformulator} to better align with the retrieval objective~($\S$~\ref{subsec:pseudo-optim}). The improved CQR model is reintroduced for further response refinement, forming a self-reinforcing cycle\,\citep{self-training-survery} that continually enhances both pseudo references and retrieval performance.
Algorithm \ref{algo:overall} details each step.



\begin{wrapfigure}{R!}{0.51\linewidth}
\begin{minipage}{\linewidth}
\vspace{-0.85cm}
\begin{algorithm}[H]
\footnotesize
\caption{\algname{}}
\label{algo:overall}
\begin{algorithmic}[1]
{\footnotesize
    \STATE {\textbf{Input}: Conversational dataset $\mathcal{T}_U$, CQR model $\theta$, Number of iterations $n_\text{iters}$ } \\
    
    \STATE {$i \gets 0$}
    \WHILE{$i < n_{\text{iters}}$} 
    {
        \STATE \textcolor{blue}{\textit{/* \textsc{Pseudo Reference Generation $\S$ \ref{subsec:pseudo-generation}} */}} \\
        \STATE $\mathcal{T}_P \!\gets\! \emptyset$ 
        \FOR{each $(x_t, a_t) \in \mathcal{T}_U$}
            
            \STATE {$\tilde{a}_t \!\gets\! \text{``''}$ { /* Initialize as an empty string */}}
            \IF{$i > 0$}
                \STATE {$\tilde{a}_t \!\gets\!\! {\rm RefineResponse}(x_t,\! a_t,\! \theta)$}\! { /* Eq.~(\ref{eq:response-refinement}) */}\!\!\!
            \ENDIF
            \STATE {$\tilde{G}_t \!\gets\! {\rm RetrieveReference}(a_t,\! \tilde{a}_t)$}\! { /* Eq.~(\ref{eq:pseudo-ref}) */}\!\!\!
            \STATE $\mathcal{T}_P \!\gets\! \mathcal{T}_P \cup \{ (x_t, a_t, \tilde{G}_t) \}$
        \ENDFOR
        \STATE \textcolor{blue}{\textit{/* \textsc{Preference Optimization $\S$ \ref{subsec:pseudo-optim}} */}} \\
        \STATE $\theta \gets\! {\rm Optimize}(\mathcal{T}_P, \theta)$ { /* Eq.~(\ref{eq:pseudo-optim-obj}) */}
        \STATE $i \gets i + 1$
}
\ENDWHILE
\STATE {\textbf{return} $ \theta $}
}
\end{algorithmic}
\end{algorithm}
\end{minipage}
\vspace{-0.5cm}
\end{wrapfigure}

\subsection{Step I: Pseudo Reference Generation}
\label{subsec:pseudo-generation}
For the \(t\)-th conversation turn, we define \textit{pseudo reference passages} \(\tilde{G}_t\) in Definition \ref{def:pseudo-ref}.

\begin{definition}[\sc {Pseudo Reference Passages}]
\label{def:pseudo-ref} 
Given a response \(a_t\) at the \(t\)-th conversation turn, and its refined counterpart \(\tilde{a}_t\), the \textit{pseudo reference passages} \(\tilde{G}_t\) are a set of passages, retrieved by $R(\cdot)$, which are most relevant to \( a_t \) and \(\tilde{a}_t\). Formally,
\begin{equation}
\label{eq:pseudo-ref}
\tilde{G}_t = R(a_t~ || ~\tilde{a}_t),
\end{equation}  
where \( || \) denotes string concatenation of \( a_t \) and \(\tilde{a}_t\).
\end{definition}

\smallskip
\noindent\textbf{Response Refinement via CQR.} 
Since the CQR model expects a query format as its input, each raw response \(a_t\) is converted into a query using a template function \(T(\cdot)\). 
Formally, given an input \( x_t = (\mathcal{H}_{<t}, q_t) \) and a raw response \( a_t \), 
the refined response \(\tilde{a}_t\) then becomes  
\begin{equation}
\label{eq:response-refinement}
\tilde{a}_t \!\!\;=\;\!\!
\begin{cases}
{\rm CQR}\bigl((x_t, T(a_t)); \theta^*\bigr) 
& \!\text{if \(\theta^*\) is trained}\\
\epsilon~~\text{(i.e., empty string)}
& \!\text{otherwise}, 
\end{cases}
\end{equation}
where \(\theta^*\) denotes the parameters of the trained CQR model from the previous iteration, held fixed in the current iteration. When $\theta^*$ is not sufficiently trained during the initialization period, an empty string \(\epsilon\) is returned and no refinement is applied. 



\newcommand{\newprompt}{\refstepcounter{promptcounter}}
\newprompt


\smallskip
\begin{minipage}[t]{0.57\textwidth}
\noindent\textbf{Query-Forming Template Function.} 
When converting a response into a query format, we exploit the CQR model's mechanism as well. In addition to query reformulation, recent CQR models\,\citep{retpo, convgqr} also perform query expansion that adds a potential response. 
To harness the overall mechanism, we design the template function \(T(\cdot)\) for a given response \(a_t\) in Prompt~\ref{prompt_temp}. 
\end{minipage} 
\hspace{0.03\textwidth}
\begin{minipage}[t]{0.4\textwidth}
\vspace{-0.2cm}
\centering
\begin{MyBox}[label={prompt_temp}]{Template~\(T(a_t)\)}
{\normalsize Can you clearly \textit{state the main points} of the \textit{last response} (\{$a_t$\}), contextualizing them and resolving coreferences?}
\end{MyBox}
\vspace{0.1cm}
\end{minipage} \\ 
This template prompts the CQR model to extract the context pertinent to \(a_t\) by rephrasing \textit{``last response \((\{a_t\})\)''} and to append a potential response to ``\textit{state the main points}'' of \(a_t\).



\subsection{Step II: Preference Optimization}
\label{subsec:pseudo-optim}

Once pseudo reference passages \(\{\tilde{G}_t\}_{t=1}^N\) are ready, \algname{} conducts preference optimization on $\mathcal{T}_P =\{(x_t, a_t, \tilde{G}_t)\}_{t=1}^{N}$, pairing each conversational turn with its corresponding \(\tilde{G}_t\).

\smallskip
\noindent\textbf{Preference Feedback Generation.} 
We use pseudo reference passages \(\tilde{G}_t\) to derive preference feedback \({\rm pref}(\tilde{G}_t)\) by ranking the candidate query reformulations \(\{\tilde{q}_t^i\}_{i=1}^M\) according to their retrieval scores as in Eq.~(\ref{eq:pref}). The retrieval score \(s(\tilde{q}_t^i \mid \tilde{G}_t)\) quantifies how accurately \(\tilde{q}_t^i\) retrieves passages comprising \(\tilde{G}_t\). This score is derived from a combination of multiple retrieval metrics, including Recall@k, MRR, and NDCG, to account for distinct aspects of retrieval quality\,\citep{eval_metrics}. See Appendix \ref{sec:ret-score} for its complete definition.

\smallskip
\noindent\textbf{Preference Optimization.} 
We then apply the standard preference optimization process consisting of SFT and DPO, following Eq.~(\ref{eq:sft}) and Eq.~(\ref{eq:dpo}), guided by the preference feedback. The training objective is to optimize $\theta_{\tilde{G}}$ such that
\begin{equation}
\label{eq:pseudo-optim-obj}
\theta_{\tilde{G}} = \underset{\theta}{\arg\min} \sum_{t=1}^{|\mathcal{T}_P|} \ell_{\text{pref}}(x_t, \tilde{G}_t; \theta) \\
\end{equation}
where $\theta$ is initially set to the parameter values obtained by SFT, i.e., $\underset{\theta}{\arg\min} \sum_{t=1}^{|\mathcal{T}_P|} \ell_{\text{sft}}(x_t, \tilde{G}_t; \theta)$.

\section{Evaluation}
\label{sec:evaluation}

\subsection{Experiment Setting}
\noindent\textbf{Dataset Preparation.}
We evaluate the efficacy of \algname{} in realistic CQR deployment scenarios where reference passages are \textit{unavailable} for target conversational datasets. To reflect the diversity of target datasets in the real world, we conduct experiments on three benchmarks: \textit{QReCC}\,\citep{qrecc} and \textit{TopiOCQA}\,\citep{topiocqa}, which focus on general-domain topics with similar conversational contexts, and \textit{SciConvQA}, a new benchmark focusing on specialized scientific domains with diverse conversational contexts.

\smallskip
\noindent{\underline{SciConvQA:}} This is our proprietary conversational dataset, constructed using scientific journal data\footnote{\url{https://aida.kisti.re.kr/data/b22c73ed-fa19-47b0-87b3-a509df8380e5}} provided by the Korea Institute of Science and Technology Information\footnote{\url{https://www.kisti.re.kr/}}, a
government-funded research institute. The dataset follows the conversation generation protocol of TopiOCQA\,\citep{topiocqa} and will be publicly released upon acceptance. Additional details, including an example conversation, and visual comparisons with QReCC and TopiOCQA, are provided in Appendix~\ref{sec:sciconvqa}.


\smallskip
\noindent\textbf{Algorithms.}
We compare \algname{} against (i) \textit{LLM-based} CQR methods, LLM-IQR\,\citep{infocqr}, HyDE-LLM\,\citep{hyde}, and LLM4CS-CoT\,\citep{llm4cs}; (ii) \textit{SFT-only} CQR methods, T5QR\,\citep{t5qr}; and (iii) \textit{reference-based} CQR methods, ConvGQR\,\citep{convgqr}, HyDE-FT\,\citep{hyde}, and RetPo\,\citep{retpo}. The first and second categories do not need reference passages, but the second category needs optimally reformulated queries, which are even more costly to obtain. HyDE is configured as either HyDE-LLM or HyDE-FT. 

To run \textit{reference-based} baselines on a target dataset \textit{devoid of reference passages}, we pre-train their CQR models on a \textit{source} dataset having reference passages and transfer the models to the target dataset. Two transfer scenarios are intended for a \textit{large} domain gap between a source and a target, \textit{QReCC\scalebox{.8}{~(source)}\,$\rightarrow$\,SciConvQA\scalebox{.8}{~(target)}}, and a \textit{small} domain gap,  \textit{QReCC\scalebox{.8}{~(source)}\,$\rightarrow$\,TopiOCQA\scalebox{.8}{~(target)}}. 

Additionally, we include the {\textit{Upper Bound\,(RetPo$^{\dagger}$)}} baseline that performs preference optimization by RetPo using genuine reference passages of each target dataset. This Upper~Bound baseline is used to estimate the ideal performance of CQR for a target dataset, although it is not practically usable owing to the necessity of reference passages. Importantly, \textit{none} of the baselines except Upper~Bound accesses reference passages.



\smallskip
\noindent\textbf{Metrics.}
We evaluate (1) \textit{pseudo reference accuracy}, assessing the agreement between our pseudo and ground-truth reference passages; (2) \textit{retrieval accuracy}, evaluating the agreement between CQR-retrieved passages and ground-truth reference passage; and (3) \textit{response generation accuracy}, measuring the quality of LLM-generated responses with CQR-retrieved passages. We measure pseudo reference and retrieval accuracy using MRR, NDCG@3, and Recall@$k$\,\citep{convgqr, retpo}, and generation accuracy using LLMeval, ROUGE, and BertScore\,\citep{adaptive-rag2, bergen}.
More details on the evaluation metrics are provided in Appendix~\ref{sec:app_metrics}.

\smallskip
\noindent\textbf{Retriever Systems.} 
Following \citep{convgqr, infocqr, retpo}, we employ BM25\,\citep{bm25} for sparse retrieval and GTR\,\citep{gtr} for dense retrieval. 

\noindent\textbf{Implementation Details.}
We train all baselines, except RetPo, using their official repositories. Due to the absence of released code, we implement RetPo by adopting our strategy for preference feedback generation and confirm that it achieves better performance than the original paper. For \algname{}, pseudo reference passages are updated once per epoch throughout three epochs, following \citep{noisy_student}. 
The top-3 relevant passages are chosen as pseudo reference passages for each conversation turn. The mean and standard error of three repetitions with different random initializations are reported. Additional implementation details are provided in Appendix~\ref{sec:app_impl}.



\def\arraystretch{1.00}
\begin{table*}[t!]
\caption{Retrieval accuracy of \algname{} compared with representative CQR baselines on the target conversational datasets: SciConvQA and TopiOCQA. The best and second-best results~(excluding Upper Bound) are highlighted in bold and underlined, respectively.}
\vspace{-0.6em}
\small
\centering
\resizebox{0.90\linewidth}{!}{%
\begin{tabular}[c]
{@{}ccc|cccc|ccccc@{}}
\arrayrulecolor{black}\specialrule{1.3pt}{0.75pt}{1.0pt}
\multirow{2}{*}{\hspace{-0.0cm}\makecell[c]{\textbf{Data}}\hspace{-0.0cm}}
& \multicolumn{2}{c|}{\multirow{2}{*}{\makecell[c]{\textbf{Query}\\\textbf{Reformulations}}}}
& \multicolumn{4}{c|}{{ \textbf{Sparse Retriever}}} 
& \multicolumn{4}{c}{{ \textbf{Dense Retriever}}}\\
& & & \multicolumn{1}{c}{~\textbf{MRR}~}
& \multicolumn{1}{c}{\textbf{NDCG}}
& \multicolumn{1}{c}{~\textbf{R@5}~}
& \multicolumn{1}{c|}{\textbf{R@20}}
& \multicolumn{1}{c}{~\textbf{MRR}~}
& \multicolumn{1}{c}{\textbf{NDCG}}
& \multicolumn{1}{c}{~\textbf{R@5}~}
& \multicolumn{1}{c}{\textbf{R@20}}
\\ 
\arrayrulecolor{black}\specialrule{1pt}{1.5pt}{0.7pt} 
\arrayrulecolor{black}\specialrule{1pt}{0.7pt}{3.0pt}
\multirow{10}{*}{\hspace{-0.0cm}\rotatebox{90}{\textbf{SciConvQA}}} 
& \multicolumn{2}{c|}{{ {Upper Bound}}} 
& \begin{tabular}[c]{@{}c@{}}{\hspace{-0.10cm}{20.89}\scalebox{0.6}{$\pm$0.30}\hspace{-0.20cm}}\end{tabular}
& \begin{tabular}[c]{@{}c@{}}{\hspace{-0.10cm}{18.91}\scalebox{0.6}{$\pm$0.58}\hspace{-0.20cm}}\end{tabular}
& \begin{tabular}[c]{@{}c@{}}{\hspace{-0.10cm}{30.09}\scalebox{0.6}{$\pm$0.62}\hspace{-0.20cm}}\end{tabular}
& \begin{tabular}[c]{@{}c@{}}{\hspace{-0.10cm}{43.45}\scalebox{0.6}{$\pm$0.50}\hspace{-0.20cm}}\end{tabular}
& \begin{tabular}[c]{@{}c@{}}{\hspace{-0.10cm}{23.73}\scalebox{0.6}{$\pm$0.70}\hspace{-0.20cm}}\end{tabular}
& \begin{tabular}[c]{@{}c@{}}{\hspace{-0.10cm}{22.49}\scalebox{0.6}{$\pm$0.80}\hspace{-0.20cm}}\end{tabular}
& \begin{tabular}[c]{@{}c@{}}{\hspace{-0.10cm}{31.95}\scalebox{0.6}{$\pm$1.00}\hspace{-0.20cm}}\end{tabular}
& \begin{tabular}[c]{@{}c@{}}{\hspace{-0.10cm}{45.14}\scalebox{0.6}{$\pm$0.67}\hspace{-0.20cm}}\end{tabular}
\\ \addlinespace[0.1ex]\cline{2-11}\addlinespace[0.5ex]
& \multicolumn{2}{c|}{{ \texttt{Original Query}}} 
& \begin{tabular}[c]{@{}c@{}}{\hspace{-0.10cm}{4.85}\scalebox{0.6}{$\pm$0.00}\hspace{-0.20cm}}\end{tabular}
& \begin{tabular}[c]{@{}c@{}}{\hspace{-0.10cm}{4.36}\scalebox{0.6}{$\pm$0.00}\hspace{-0.20cm}}\end{tabular}
& \begin{tabular}[c]{@{}c@{}}{\hspace{-0.10cm}{6.70}\scalebox{0.6}{$\pm$0.00}\hspace{-0.20cm}}\end{tabular}
& \begin{tabular}[c]{@{}c@{}}{\hspace{-0.10cm}{10.85}\scalebox{0.6}{$\pm$0.00}\hspace{-0.20cm}}\end{tabular}
& \begin{tabular}[c]{@{}c@{}}{\hspace{-0.10cm}{6.24}\scalebox{0.6}{$\pm$0.00}\hspace{-0.20cm}}\end{tabular}
& \begin{tabular}[c]{@{}c@{}}{\hspace{-0.10cm}{5.57}\scalebox{0.6}{$\pm$0.00}\hspace{-0.20cm}}\end{tabular}
& \begin{tabular}[c]{@{}c@{}}{\hspace{-0.10cm}{8.17}\scalebox{0.6}{$\pm$0.00}\hspace{-0.20cm}}\end{tabular}
& \begin{tabular}[c]{@{}c@{}}{\hspace{-0.10cm}{14.18}\scalebox{0.6}{$\pm$0.00}\hspace{-0.20cm}}\end{tabular}
\\ \addlinespace[0.5ex]\cdashline{2-11}\addlinespace[0.7ex]
& \multicolumn{2}{c|}{{ {LLM-IQR}}} 
& \begin{tabular}[c]{@{}c@{}}{\hspace{-0.10cm}{14.23}\scalebox{0.6}{$\pm$0.09}\hspace{-0.20cm}}\end{tabular}
& \begin{tabular}[c]{@{}c@{}}{\hspace{-0.10cm}{13.15}\scalebox{0.6}{$\pm$0.10}\hspace{-0.20cm}}\end{tabular}
& \begin{tabular}[c]{@{}c@{}}{\hspace{-0.10cm}{19.83}\scalebox{0.6}{$\pm$0.08}\hspace{-0.20cm}}\end{tabular}
& \begin{tabular}[c]{@{}c@{}}{\hspace{-0.10cm}{29.55}\scalebox{0.6}{$\pm$0.04}\hspace{-0.20cm}}\end{tabular}
& \begin{tabular}[c]{@{}c@{}}{\hspace{-0.10cm}{16.23}\scalebox{0.6}{$\pm$0.12}\hspace{-0.20cm}}\end{tabular}
& \begin{tabular}[c]{@{}c@{}}{\hspace{-0.10cm}{15.09}\scalebox{0.6}{$\pm$0.15}\hspace{-0.20cm}}\end{tabular}
& \begin{tabular}[c]{@{}c@{}}{\hspace{-0.10cm}{22.49}\scalebox{0.6}{$\pm$0.15}\hspace{-0.20cm}}\end{tabular}
& \begin{tabular}[c]{@{}c@{}}{\hspace{-0.10cm}{32.75}\scalebox{0.6}{$\pm$0.12}\hspace{-0.20cm}}\end{tabular}
\\
& \multicolumn{2}{c|}{{ {HyDE-LLM}}} 
& \begin{tabular}[c]{@{}c@{}}{\hspace{-0.10cm}{13.78}\scalebox{0.6}{$\pm$0.54}\hspace{-0.20cm}}\end{tabular}
& \begin{tabular}[c]{@{}c@{}}{\hspace{-0.10cm}{12.56}\scalebox{0.6}{$\pm$0.32}\hspace{-0.20cm}}\end{tabular}
& \begin{tabular}[c]{@{}c@{}}{\hspace{-0.10cm}{20.68}\scalebox{0.6}{$\pm$1.06}\hspace{-0.20cm}}\end{tabular}
& \begin{tabular}[c]{@{}c@{}}{\hspace{-0.10cm}{33.47}\scalebox{0.6}{$\pm$1.36}\hspace{-0.20cm}}\end{tabular}
& \begin{tabular}[c]{@{}c@{}}{\hspace{-0.10cm}{16.70}\scalebox{0.6}{$\pm$0.69}\hspace{-0.20cm}}\end{tabular}
& \begin{tabular}[c]{@{}c@{}}{\hspace{-0.10cm}{15.38}\scalebox{0.6}{$\pm$0.66}\hspace{-0.20cm}}\end{tabular}
& \begin{tabular}[c]{@{}c@{}}{\hspace{-0.10cm}{23.28}\scalebox{0.6}{$\pm$0.70}\hspace{-0.20cm}}\end{tabular}
& \begin{tabular}[c]{@{}c@{}}{\hspace{-0.10cm}{35.36}\scalebox{0.6}{$\pm$1.68}\hspace{-0.20cm}}\end{tabular}
\\ 
& \multicolumn{2}{c|}{{ {LLM4CS-CoT}}} 
& \begin{tabular}[c]{@{}c@{}}\hspace{-0.20cm}{\underline{16.53}\scalebox{0.6}{$\pm$0.18}\hspace{-0.20cm}}\end{tabular}
& \begin{tabular}[c]{@{}c@{}}\hspace{-0.20cm}{\underline{15.29}\scalebox{0.6}{$\pm$0.27}\hspace{-0.20cm}}\end{tabular}
& \begin{tabular}[c]{@{}c@{}}\hspace{-0.20cm}{\underline{22.49}\scalebox{0.6}{$\pm$0.30}\hspace{-0.20cm}}\end{tabular}
& \begin{tabular}[c]{@{}c@{}}\hspace{-0.20cm}{\underline{33.55}\scalebox{0.6}{$\pm$0.38}\hspace{-0.20cm}}\end{tabular}
& \begin{tabular}[c]{@{}c@{}}\hspace{-0.20cm}{\underline{18.25}\scalebox{0.6}{$\pm$0.22}\hspace{-0.20cm}}\end{tabular}
& \begin{tabular}[c]{@{}c@{}}\hspace{-0.20cm}{\underline{16.75}\scalebox{0.6}{$\pm$0.18}\hspace{-0.20cm}}\end{tabular}
& \begin{tabular}[c]{@{}c@{}}\hspace{-0.20cm}{\underline{24.48}\scalebox{0.6}{$\pm$0.22}\hspace{-0.20cm}}\end{tabular}
& \begin{tabular}[c]{@{}c@{}}\hspace{-0.20cm}{\underline{36.30}\scalebox{0.6}{$\pm$0.02}\hspace{-0.20cm}}\end{tabular}
\\ 
& \multicolumn{2}{c|}{{ {T5QR}}} 
& \begin{tabular}[c]{@{}c@{}}{\hspace{-0.10cm}{11.45}\scalebox{0.6}{$\pm$0.72}\hspace{-0.20cm}}\end{tabular}
& \begin{tabular}[c]{@{}c@{}}{\hspace{-0.10cm}{10.64}\scalebox{0.6}{$\pm$0.73}\hspace{-0.20cm}}\end{tabular}
& \begin{tabular}[c]{@{}c@{}}{\hspace{-0.10cm}{15.88}\scalebox{0.6}{$\pm$1.02}\hspace{-0.20cm}}\end{tabular}
& \begin{tabular}[c]{@{}c@{}}{\hspace{-0.10cm}{23.49}\scalebox{0.6}{$\pm$1.56}\hspace{-0.20cm}}\end{tabular}
& \begin{tabular}[c]{@{}c@{}}{\hspace{-0.10cm}{14.51}\scalebox{0.6}{$\pm$0.74}\hspace{-0.20cm}}\end{tabular}
& \begin{tabular}[c]{@{}c@{}}{\hspace{-0.10cm}{13.44}\scalebox{0.6}{$\pm$0.77}\hspace{-0.20cm}}\end{tabular}
& \begin{tabular}[c]{@{}c@{}}{\hspace{-0.10cm}{19.42}\scalebox{0.6}{$\pm$1.12}\hspace{-0.20cm}}\end{tabular}
& \begin{tabular}[c]{@{}c@{}}{\hspace{-0.10cm}{29.65}\scalebox{0.6}{$\pm$1.44}\hspace{-0.20cm}}\end{tabular}
\\ \addlinespace[0.5ex]\cdashline{2-11}\addlinespace[0.7ex]
& \multirow{3}{*}{\hspace{-0.0cm}{\makecell[c]{\scriptsize\textit{QReCC}\\ \small $\downarrow$ \\ \scriptsize\textit{SciConvQA}}}\hspace{-0cm}} 
& \multirow{1}{*}{\hspace{-0.2cm}\makecell[c]{ConvGQR}\hspace{-0.cm}}
& \begin{tabular}[c]{@{}c@{}}{\hspace{-0.10cm}{13.55}\scalebox{0.6}{$\pm$0.23}\hspace{-0.20cm}}\end{tabular}
& \begin{tabular}[c]{@{}c@{}}{\hspace{-0.10cm}{12.51}\scalebox{0.6}{$\pm$0.27}\hspace{-0.20cm}}\end{tabular}
& \begin{tabular}[c]{@{}c@{}}{\hspace{-0.10cm}{18.68}\scalebox{0.6}{$\pm$0.39}\hspace{-0.20cm}}\end{tabular}
& \begin{tabular}[c]{@{}c@{}}{\hspace{-0.10cm}{28.71}\scalebox{0.6}{$\pm$0.30}\hspace{-0.20cm}}\end{tabular}
& \begin{tabular}[c]{@{}c@{}}{\hspace{-0.10cm}{14.67}\scalebox{0.6}{$\pm$0.10}\hspace{-0.20cm}}\end{tabular}
& \begin{tabular}[c]{@{}c@{}}{\hspace{-0.10cm}{13.72}\scalebox{0.6}{$\pm$0.16}\hspace{-0.20cm}}\end{tabular}
& \begin{tabular}[c]{@{}c@{}}{\hspace{-0.10cm}{20.03}\scalebox{0.6}{$\pm$0.08}\hspace{-0.20cm}}\end{tabular}
& \begin{tabular}[c]{@{}c@{}}{\hspace{-0.10cm}{29.90}\scalebox{0.6}{$\pm$0.58}\hspace{-0.20cm}}\end{tabular}
\\
& & \multirow{1}{*}{\hspace{-0.2cm}\makecell[c]{HyDE-FT}\hspace{-0.cm}}
& \begin{tabular}[c]{@{}c@{}}{\hspace{-0.10cm}{11.09}\scalebox{0.6}{$\pm$0.47}\hspace{-0.20cm}}\end{tabular}
& \begin{tabular}[c]{@{}c@{}}{\hspace{-0.10cm}{9.85}\scalebox{0.6}{$\pm$0.52}\hspace{-0.20cm}}\end{tabular}
& \begin{tabular}[c]{@{}c@{}}{\hspace{-0.10cm}{15.89}\scalebox{0.6}{$\pm$1.07}\hspace{-0.20cm}}\end{tabular}
& \begin{tabular}[c]{@{}c@{}}{\hspace{-0.10cm}{27.74}\scalebox{0.6}{$\pm$1.74}\hspace{-0.20cm}}\end{tabular}
& \begin{tabular}[c]{@{}c@{}}{\hspace{-0.10cm}{12.16}\scalebox{0.6}{$\pm$1.03}\hspace{-0.20cm}}\end{tabular}
& \begin{tabular}[c]{@{}c@{}}{\hspace{-0.10cm}{11.09}\scalebox{0.6}{$\pm$1.15}\hspace{-0.20cm}}\end{tabular}
& \begin{tabular}[c]{@{}c@{}}{\hspace{-0.10cm}{16.63}\scalebox{0.6}{$\pm$1.90}\hspace{-0.20cm}}\end{tabular}
& \begin{tabular}[c]{@{}c@{}}{\hspace{-0.10cm}{27.50}\scalebox{0.6}{$\pm$2.65}\hspace{-0.20cm}}\end{tabular}
\\
& & \multirow{1}{*}{\hspace{-0.2cm}\makecell[c]{RetPo}\hspace{-0.cm}}
& \begin{tabular}[c]{@{}c@{}}{\hspace{-0.10cm}{15.60}\scalebox{0.6}{$\pm$0.23}\hspace{-0.20cm}}\end{tabular}
& \begin{tabular}[c]{@{}c@{}}{\hspace{-0.10cm}{14.36}\scalebox{0.6}{$\pm$0.24}\hspace{-0.20cm}}\end{tabular}
& \begin{tabular}[c]{@{}c@{}}{\hspace{-0.10cm}{21.61}\scalebox{0.6}{$\pm$0.33}\hspace{-0.20cm}}\end{tabular}
& \begin{tabular}[c]{@{}c@{}}{\hspace{-0.10cm}{32.63}\scalebox{0.6}{$\pm$0.50}\hspace{-0.20cm}}\end{tabular}
& \begin{tabular}[c]{@{}c@{}}{\hspace{-0.10cm}{17.95}\scalebox{0.6}{$\pm$0.08}\hspace{-0.20cm}}\end{tabular}
& \begin{tabular}[c]{@{}c@{}}{\hspace{-0.10cm}{16.74}\scalebox{0.6}{$\pm$0.08}\hspace{-0.20cm}}\end{tabular}
& \begin{tabular}[c]{@{}c@{}}{\hspace{-0.10cm}{24.31}\scalebox{0.6}{$\pm$0.36}\hspace{-0.20cm}}\end{tabular}
& \begin{tabular}[c]{@{}c@{}}{\hspace{-0.10cm}{35.45}\scalebox{0.6}{$\pm$0.56}\hspace{-0.20cm}}\end{tabular}
\\ \addlinespace[0.3ex]\cdashline{2-11}\addlinespace[0.8ex]
& \multicolumn{2}{c|}{{ \textbf{\algname{}}}} 
& \begin{tabular}[c]{@{}c@{}}{\hspace{-0.10cm}\textbf{20.06}\scalebox{0.6}{$\pm$0.19}\hspace{-0.20cm}}\end{tabular}
& \begin{tabular}[c]{@{}c@{}}{\hspace{-0.10cm}\textbf{18.42}\scalebox{0.6}{$\pm$0.27}\hspace{-0.20cm}}\end{tabular}
& \begin{tabular}[c]{@{}c@{}}{\hspace{-0.10cm}\textbf{28.64}\scalebox{0.6}{$\pm$0.28}\hspace{-0.20cm}}\end{tabular}
& \begin{tabular}[c]{@{}c@{}}{\hspace{-0.10cm}\textbf{43.14}\scalebox{0.6}{$\pm$0.20}\hspace{-0.20cm}}\end{tabular}
& \begin{tabular}[c]{@{}c@{}}{\hspace{-0.10cm}\textbf{23.53}\scalebox{0.6}{$\pm$0.09}\hspace{-0.20cm}}\end{tabular}
& \begin{tabular}[c]{@{}c@{}}{\hspace{-0.10cm}\textbf{22.04}\scalebox{0.6}{$\pm$0.23}\hspace{-0.20cm}}\end{tabular}
& \begin{tabular}[c]{@{}c@{}}{\hspace{-0.10cm}\textbf{31.19}\scalebox{0.6}{$\pm$0.18}\hspace{-0.20cm}}\end{tabular}
& \begin{tabular}[c]{@{}c@{}}{\hspace{-0.10cm}\textbf{45.21}\scalebox{0.6}{$\pm$0.22}\hspace{-0.20cm}}\end{tabular}
\\ 
\arrayrulecolor{black}\specialrule{1.3pt}{1.5pt}{1.5pt} 
\multirow{10}{*}{\hspace{-0.0cm}\rotatebox{90}{\textbf{TopiOCQA}}} 
& \multicolumn{2}{c|}{{ {Upper Bound}}} 
& \begin{tabular}[c]{@{}c@{}}{\hspace{-0.10cm}{29.19}\scalebox{0.6}{$\pm$0.30}\hspace{-0.20cm}}\end{tabular}
& \begin{tabular}[c]{@{}c@{}}{\hspace{-0.10cm}{27.19}\scalebox{0.6}{$\pm$0.22}\hspace{-0.20cm}}\end{tabular}
& \begin{tabular}[c]{@{}c@{}}{\hspace{-0.10cm}{39.72}\scalebox{0.6}{$\pm$0.87}\hspace{-0.20cm}}\end{tabular}
& \begin{tabular}[c]{@{}c@{}}{\hspace{-0.10cm}{56.89}\scalebox{0.6}{$\pm$0.97}\hspace{-0.20cm}}\end{tabular}
& \begin{tabular}[c]{@{}c@{}}{\hspace{-0.10cm}{40.58}\scalebox{0.6}{$\pm$0.22}\hspace{-0.20cm}}\end{tabular}
& \begin{tabular}[c]{@{}c@{}}{\hspace{-0.10cm}{39.15}\scalebox{0.6}{$\pm$0.25}\hspace{-0.20cm}}\end{tabular}
& \begin{tabular}[c]{@{}c@{}}{\hspace{-0.10cm}{53.42}\scalebox{0.6}{$\pm$1.04}\hspace{-0.20cm}}\end{tabular}
& \begin{tabular}[c]{@{}c@{}}{\hspace{-0.10cm}{72.69}\scalebox{0.6}{$\pm$1.22}\hspace{-0.20cm}}\end{tabular}
\\ \addlinespace[0.3ex]\cline{2-11}\addlinespace[0.5ex]
& \multicolumn{2}{c|}{{ \texttt{Original Query}}} 
& \begin{tabular}[c]{@{}c@{}}{\hspace{-0.10cm}{2.09}\scalebox{0.6}{$\pm$0.00}\hspace{-0.20cm}}\end{tabular}
& \begin{tabular}[c]{@{}c@{}}{\hspace{-0.10cm}{1.77}\scalebox{0.6}{$\pm$0.00}\hspace{-0.20cm}}\end{tabular}
& \begin{tabular}[c]{@{}c@{}}{\hspace{-0.10cm}{2.90}\scalebox{0.6}{$\pm$0.00}\hspace{-0.20cm}}\end{tabular}
& \begin{tabular}[c]{@{}c@{}}{\hspace{-0.10cm}{5.21}\scalebox{0.6}{$\pm$0.00}\hspace{-0.20cm}}\end{tabular}
& \begin{tabular}[c]{@{}c@{}}{\hspace{-0.10cm}{5.95}\scalebox{0.6}{$\pm$0.00}\hspace{-0.20cm}}\end{tabular}
& \begin{tabular}[c]{@{}c@{}}{\hspace{-0.10cm}{5.52}\scalebox{0.6}{$\pm$0.00}\hspace{-0.20cm}}\end{tabular}
& \begin{tabular}[c]{@{}c@{}}{\hspace{-0.10cm}{8.35}\scalebox{0.6}{$\pm$0.00}\hspace{-0.20cm}}\end{tabular}
& \begin{tabular}[c]{@{}c@{}}{\hspace{-0.10cm}{12.05}\scalebox{0.6}{$\pm$0.00}\hspace{-0.20cm}}\end{tabular}
\\ \addlinespace[0.5ex]\cdashline{2-11}\addlinespace[0.7ex]
& \multicolumn{2}{c|}{{ {LLM-IQR}}} 
& \begin{tabular}[c]{@{}c@{}}{\hspace{-0.10cm}{17.21}\scalebox{0.6}{$\pm$0.09}\hspace{-0.20cm}}\end{tabular}
& \begin{tabular}[c]{@{}c@{}}{\hspace{-0.10cm}{15.66}\scalebox{0.6}{$\pm$0.11}\hspace{-0.20cm}}\end{tabular}
& \begin{tabular}[c]{@{}c@{}}{\hspace{-0.10cm}{24.34}\scalebox{0.6}{$\pm$0.11}\hspace{-0.20cm}}\end{tabular}
& \begin{tabular}[c]{@{}c@{}}{\hspace{-0.10cm}{37.72}\scalebox{0.6}{$\pm$0.06}\hspace{-0.20cm}}\end{tabular}
& \begin{tabular}[c]{@{}c@{}}{\hspace{-0.10cm}{32.53}\scalebox{0.6}{$\pm$0.13}\hspace{-0.20cm}}\end{tabular}
& \begin{tabular}[c]{@{}c@{}}{\hspace{-0.10cm}{33.26}\scalebox{0.6}{$\pm$0.16}\hspace{-0.20cm}}\end{tabular}
& \begin{tabular}[c]{@{}c@{}}{\hspace{-0.10cm}{45.61}\scalebox{0.6}{$\pm$0.18}\hspace{-0.20cm}}\end{tabular}
& \begin{tabular}[c]{@{}c@{}}{\hspace{-0.10cm}{61.63}\scalebox{0.6}{$\pm$0.14}\hspace{-0.20cm}}\end{tabular}
\\
& \multicolumn{2}{c|}{{ {HyDE-LLM}}} 
& \begin{tabular}[c]{@{}c@{}}{\hspace{-0.10cm}{20.65}\scalebox{0.6}{$\pm$1.29}\hspace{-0.20cm}}\end{tabular}
& \begin{tabular}[c]{@{}c@{}}{\hspace{-0.10cm}{18.59}\scalebox{0.6}{$\pm$0.87}\hspace{-0.20cm}}\end{tabular}
& \begin{tabular}[c]{@{}c@{}}{\hspace{-0.10cm}{28.26}\scalebox{0.6}{$\pm$1.95}\hspace{-0.20cm}}\end{tabular}
& \begin{tabular}[c]{@{}c@{}}{\hspace{-0.10cm}{44.34}\scalebox{0.6}{$\pm$1.03}\hspace{-0.20cm}}\end{tabular}
& \begin{tabular}[c]{@{}c@{}}{\hspace{-0.10cm}{33.39}\scalebox{0.6}{$\pm$1.48}\hspace{-0.20cm}}\end{tabular}
& \begin{tabular}[c]{@{}c@{}}{\hspace{-0.10cm}{32.02}\scalebox{0.6}{$\pm$1.34}\hspace{-0.20cm}}\end{tabular}
& \begin{tabular}[c]{@{}c@{}}{\hspace{-0.10cm}{45.54}\scalebox{0.6}{$\pm$0.05}\hspace{-0.20cm}}\end{tabular}
& \begin{tabular}[c]{@{}c@{}}{\hspace{-0.10cm}{61.49}\scalebox{0.6}{$\pm$0.46}\hspace{-0.20cm}}\end{tabular}
\\ 
& \multicolumn{2}{c|}{{ {LLM4CS-CoT}}} 
& \begin{tabular}[c]{@{}c@{}}{\hspace{-0.10cm}\underline{26.81}\scalebox{0.6}{$\pm$0.66}\hspace{-0.20cm}}\end{tabular}
& \begin{tabular}[c]{@{}c@{}}{\hspace{-0.10cm}\underline{25.40}\scalebox{0.6}{$\pm$0.38}\hspace{-0.20cm}}\end{tabular}
& \begin{tabular}[c]{@{}c@{}}{\hspace{-0.10cm}\underline{37.12}\scalebox{0.6}{$\pm$1.12}\hspace{-0.20cm}}\end{tabular}
& \begin{tabular}[c]{@{}c@{}}{\hspace{-0.10cm}\underline{53.36}\scalebox{0.6}{$\pm$1.80}\hspace{-0.20cm}}\end{tabular}
& \begin{tabular}[c]{@{}c@{}}{\hspace{-0.10cm}\underline{37.55}\scalebox{0.6}{$\pm$0.75}\hspace{-0.20cm}}\end{tabular}
& \begin{tabular}[c]{@{}c@{}}{\hspace{-0.10cm}\underline{35.92}\scalebox{0.6}{$\pm$1.04}\hspace{-0.20cm}}\end{tabular}
& \begin{tabular}[c]{@{}c@{}}{\hspace{-0.10cm}\underline{52.45}\scalebox{0.6}{$\pm$0.35}\hspace{-0.20cm}}\end{tabular}
& \begin{tabular}[c]{@{}c@{}}{\hspace{-0.10cm}\underline{69.51}\scalebox{0.6}{$\pm$0.04}\hspace{-0.20cm}}\end{tabular}
\\ 
& \multicolumn{2}{c|}{{ {T5QR}}} 
& \begin{tabular}[c]{@{}c@{}}{\hspace{-0.10cm}{12.14}\scalebox{0.6}{$\pm$0.11}\hspace{-0.20cm}}\end{tabular}
& \begin{tabular}[c]{@{}c@{}}{\hspace{-0.10cm}{10.39}\scalebox{0.6}{$\pm$0.12}\hspace{-0.20cm}}\end{tabular}
& \begin{tabular}[c]{@{}c@{}}{\hspace{-0.10cm}{17.61}\scalebox{0.6}{$\pm$0.10}\hspace{-0.20cm}}\end{tabular}
& \begin{tabular}[c]{@{}c@{}}{\hspace{-0.10cm}{31.16}\scalebox{0.6}{$\pm$0.17}\hspace{-0.20cm}}\end{tabular}
& \begin{tabular}[c]{@{}c@{}}{\hspace{-0.10cm}{28.36}\scalebox{0.6}{$\pm$0.07}\hspace{-0.20cm}}\end{tabular}
& \begin{tabular}[c]{@{}c@{}}{\hspace{-0.10cm}{27.42}\scalebox{0.6}{$\pm$0.05}\hspace{-0.20cm}}\end{tabular}
& \begin{tabular}[c]{@{}c@{}}{\hspace{-0.10cm}{39.78}\scalebox{0.6}{$\pm$0.07}\hspace{-0.20cm}}\end{tabular}
& \begin{tabular}[c]{@{}c@{}}{\hspace{-0.10cm}{53.21}\scalebox{0.6}{$\pm$0.07}\hspace{-0.20cm}}\end{tabular}
\\ \addlinespace[0.5ex]\cdashline{2-11}\addlinespace[0.7ex]
& \multirow{3}{*}{\hspace{-0.0cm}{\makecell[c]{\scriptsize\textit{QReCC}\\ $\downarrow$ \\ \scriptsize\textit{TopiOCQA} }}\hspace{-0cm}} 
& \multirow{1}{*}{\hspace{-0.2cm}\makecell[c]{ConvGQR}\hspace{-0.cm}}
& \begin{tabular}[c]{@{}c@{}}{\hspace{-0.10cm}{12.86}\scalebox{0.6}{$\pm$0.16}\hspace{-0.20cm}}\end{tabular}
& \begin{tabular}[c]{@{}c@{}}{\hspace{-0.10cm}{11.28}\scalebox{0.6}{$\pm$0.12}\hspace{-0.20cm}}\end{tabular}
& \begin{tabular}[c]{@{}c@{}}{\hspace{-0.10cm}{17.89}\scalebox{0.6}{$\pm$0.23}\hspace{-0.20cm}}\end{tabular}
& \begin{tabular}[c]{@{}c@{}}{\hspace{-0.10cm}{30.91}\scalebox{0.6}{$\pm$0.47}\hspace{-0.20cm}}\end{tabular}
& \begin{tabular}[c]{@{}c@{}}{\hspace{-0.10cm}{22.58}\scalebox{0.6}{$\pm$0.14}\hspace{-0.20cm}}\end{tabular}
& \begin{tabular}[c]{@{}c@{}}{\hspace{-0.10cm}{21.33}\scalebox{0.6}{$\pm$0.19}\hspace{-0.20cm}}\end{tabular}
& \begin{tabular}[c]{@{}c@{}}{\hspace{-0.10cm}{33.60}\scalebox{0.6}{$\pm$0.07}\hspace{-0.20cm}}\end{tabular}
& \begin{tabular}[c]{@{}c@{}}{\hspace{-0.10cm}{48.20}\scalebox{0.6}{$\pm$0.26}\hspace{-0.20cm}}\end{tabular}
\\
& & \multirow{1}{*}{\hspace{-0.2cm}\makecell[c]{HyDE-FT}\hspace{-0.cm}}
& \begin{tabular}[c]{@{}c@{}}{\hspace{-0.10cm}{11.22}\scalebox{0.6}{$\pm$1.70}\hspace{-0.20cm}}\end{tabular}
& \begin{tabular}[c]{@{}c@{}}{\hspace{-0.10cm}{10.05}\scalebox{0.6}{$\pm$1.70}\hspace{-0.20cm}}\end{tabular}
& \begin{tabular}[c]{@{}c@{}}{\hspace{-0.10cm}{16.25}\scalebox{0.6}{$\pm$3.10}\hspace{-0.20cm}}\end{tabular}
& \begin{tabular}[c]{@{}c@{}}{\hspace{-0.10cm}{25.38}\scalebox{0.6}{$\pm$4.20}\hspace{-0.20cm}}\end{tabular}
& \begin{tabular}[c]{@{}c@{}}{\hspace{-0.10cm}{19.28}\scalebox{0.6}{$\pm$0.63}\hspace{-0.20cm}}\end{tabular}
& \begin{tabular}[c]{@{}c@{}}{\hspace{-0.10cm}{17.80}\scalebox{0.6}{$\pm$0.80}\hspace{-0.20cm}}\end{tabular}
& \begin{tabular}[c]{@{}c@{}}{\hspace{-0.10cm}{27.01}\scalebox{0.6}{$\pm$0.88}\hspace{-0.20cm}}\end{tabular}
& \begin{tabular}[c]{@{}c@{}}{\hspace{-0.10cm}{38.85}\scalebox{0.6}{$\pm$1.38}\hspace{-0.20cm}}\end{tabular}
\\
& & \multirow{1}{*}{\hspace{-0.2cm}\makecell[c]{RetPo}\hspace{-0.cm}}
& \begin{tabular}[c]{@{}c@{}}{\hspace{-0.10cm}{23.20}\scalebox{0.6}{$\pm$0.33}\hspace{-0.20cm}}\end{tabular}
& \begin{tabular}[c]{@{}c@{}}{\hspace{-0.10cm}{21.41}\scalebox{0.6}{$\pm$0.28}\hspace{-0.20cm}}\end{tabular}
& \begin{tabular}[c]{@{}c@{}}{\hspace{-0.10cm}{32.18}\scalebox{0.6}{$\pm$0.20}\hspace{-0.20cm}}\end{tabular}
& \begin{tabular}[c]{@{}c@{}}{\hspace{-0.10cm}{48.54}\scalebox{0.6}{$\pm$0.22}\hspace{-0.20cm}}\end{tabular}
& \begin{tabular}[c]{@{}c@{}}{\hspace{-0.10cm}{35.67}\scalebox{0.6}{$\pm$0.05}\hspace{-0.20cm}}\end{tabular}
& \begin{tabular}[c]{@{}c@{}}{\hspace{-0.10cm}{34.28}\scalebox{0.6}{$\pm$0.04}\hspace{-0.20cm}}\end{tabular}
& \begin{tabular}[c]{@{}c@{}}{\hspace{-0.10cm}{49.99}\scalebox{0.6}{$\pm$0.31}\hspace{-0.20cm}}\end{tabular}
& \begin{tabular}[c]{@{}c@{}}{\hspace{-0.10cm}{67.47}\scalebox{0.6}{$\pm$0.10}\hspace{-0.20cm}}\end{tabular}
\\ \addlinespace[0.3ex]\cdashline{2-11}\addlinespace[0.8ex]
& \multicolumn{2}{c|}{{ \textbf{\algname{} }}} 
& \begin{tabular}[c]{@{}c@{}}{\hspace{-0.10cm}\textbf{28.39}\scalebox{0.6}{$\pm$0.39}\hspace{-0.20cm}}\end{tabular}
& \begin{tabular}[c]{@{}c@{}}{\hspace{-0.10cm}\textbf{26.08}\scalebox{0.6}{$\pm$0.46}\hspace{-0.20cm}}\end{tabular}
& \begin{tabular}[c]{@{}c@{}}{\hspace{-0.10cm}\textbf{39.57}\scalebox{0.6}{$\pm$1.05}\hspace{-0.20cm}}\end{tabular}
& \begin{tabular}[c]{@{}c@{}}{\hspace{-0.10cm}\textbf{56.62}\scalebox{0.6}{$\pm$1.24}\hspace{-0.20cm}}\end{tabular}
& \begin{tabular}[c]{@{}c@{}}{\hspace{-0.10cm}\textbf{40.14}\scalebox{0.6}{$\pm$0.21}\hspace{-0.20cm}}\end{tabular}
& \begin{tabular}[c]{@{}c@{}}{\hspace{-0.10cm}\textbf{38.33}\scalebox{0.6}{$\pm$0.28}\hspace{-0.20cm}}\end{tabular}
& \begin{tabular}[c]{@{}c@{}}{\hspace{-0.10cm}\textbf{53.78}\scalebox{0.6}{$\pm$0.58}\hspace{-0.20cm}}\end{tabular}
& \begin{tabular}[c]{@{}c@{}}{\hspace{-0.10cm}\textbf{71.99}\scalebox{0.6}{$\pm$0.30}\hspace{-0.20cm}}\end{tabular}
\\ 
\arrayrulecolor{black}\specialrule{1.3pt}{1.0pt}{1.0pt}
\end{tabular}
}
\vspace{-0.6cm}
\label{tbl:overall_perf}
\end{table*}

\subsection{Main Results}
Table~\ref{tbl:overall_perf} compares \algname{} against CQR baselines for the two target datasets.

\noindent\textbf{Significance of Reference-Free Preference Optimization.} Both Upper Bound and RetPo employ preference optimization, yet they exhibit contrasting results depending on the availability of reference passages from target datasets. The Upper Bound achieves the strongest results by using reference passages, whereas RetPo struggles without them even when the target dataset shares similar conversational contexts\,(QReCC\,$\rightarrow$\,TopiOCQA). This result calls for an effective approach to preference optimization in reference-free scenarios.

\noindent\textbf{\algname{}: An Effective Reference-Free Preference Optimization Framework.} \algname{} consistently outperforms the baselines and achieves performance close to the Upper Bound across datasets and retrieval systems. On average, it achieves improvement of 15.70\% over LLM4CS-CoT, the strongest baseline, and reaches 98.23\% of the Upper Bound's performance. This result indicates the efficacy of \algname{} as a reference-free preference optimization approach.

\noindent\textbf{Robustness across Diverse CQR Domains.} 
\algname{} maintains strong performance across general domains\,(TopiOCQA) and specialized domains\,(SciConvQA), while the performance of the baselines varies considerably per domain. For example, \algname{} outperforms LLM4CS-CoT by 5.12\% on TopiOCQA, and larger improvement of 26.28\% on SciConvQA. This result reveals the domain sensitivity of the baselines and highlight \algname{}'s robust effectiveness for diverse CQR domains, facilitated by the reference-free preference optimization on target datasets. More results on QReCC are provided in Appendix~\ref{sec:app_qrecc}.

\begin{wraptable}{r}{0.38\textwidth}
\vspace{-4.8em} 
\caption{Comparison of response refinement methods.} 
\def\arraystretch{1.3}
\centering
\vspace*{-0.4cm}
\resizebox{0.99\linewidth}{!}{%
\begin{tabular}[c]
{@{}cc|cc|ccc@{}}
\arrayrulecolor{black}\specialrule{1.2pt}{0.75pt}{2.5pt}
\multirow{2}{*}{\hspace{-0.0cm}{\makecell[c]{ Data }}\hspace{-0cm}}  & \multirow{2}{*}{\hspace{-0.0cm}{\makecell[c]{\textbf{Refine}\\ \textbf{Methods}}}\hspace{-0cm}} 

& \multicolumn{2}{c|}{{\textbf{Pseudo Acc.}}} 
& \multicolumn{2}{c}{{\textbf{Retrieval Acc.}}} 
\\ 
& & \multicolumn{1}{c}{{MRR}\hspace{-0.20cm}}
& \multicolumn{1}{c|}{{R@5}\hspace{-0.20cm}}
& \multicolumn{1}{c}{{MRR}\hspace{-0.20cm}}
& \multicolumn{1}{c}{{R@5}\hspace{-0.20cm}}
\\
\arrayrulecolor{black}\specialrule{1pt}{1.5pt}{1pt}
\arrayrulecolor{black}\specialrule{1pt}{1pt}{3.0pt}
\multirow{3}{*}{\vspace{0.4cm}\rotatebox{90}{\small SciConvQA}} 
& \multirow{1}{*}{{Llama}}
& \begin{tabular}[c]{@{}c@{}}{{36.55}\hspace{-0.20cm}}\end{tabular}
& \begin{tabular}[c]{@{}c@{}}{{46.13}\hspace{-0.20cm}}\end{tabular}
& \begin{tabular}[c]{@{}c@{}}{{17.39}\hspace{-0.20cm}}\end{tabular}
& \begin{tabular}[c]{@{}c@{}}{{23.69}\hspace{-0.20cm}}\end{tabular}
\\
&  \multirow{1}{*}{{Llama+ICL}}                                
& \begin{tabular}[c]{@{}c@{}}{{44.82}\hspace{-0.20cm}}\end{tabular}
& \begin{tabular}[c]{@{}c@{}}{{54.13}\hspace{-0.20cm}}\end{tabular}
& \begin{tabular}[c]{@{}c@{}}{{19.02}\hspace{-0.20cm}}\end{tabular}
& \begin{tabular}[c]{@{}c@{}}{{26.77}\hspace{-0.20cm}}\end{tabular}
\\ 
&  \multirow{1}{*}{{{\algname{}}}}            
& \begin{tabular}[c]{@{}c@{}}{\textbf{50.05}\hspace{-0.20cm}}\end{tabular}
& \begin{tabular}[c]{@{}c@{}}{\textbf{59.75}\hspace{-0.20cm}}\end{tabular}
& \begin{tabular}[c]{@{}c@{}}{\textbf{20.06}\hspace{-0.20cm}}\end{tabular}
& \begin{tabular}[c]{@{}c@{}}{\textbf{28.64}\hspace{-0.20cm}}\end{tabular}
\\ \addlinespace[0.3ex]\cdashline{1-7}\addlinespace[0.8ex] 
\multirow{3}{*}{\vspace{0.4cm}\rotatebox{90}{\small TopiOCQA}} 
& \multirow{1}{*}{{Llama}}              
& \begin{tabular}[c]{@{}c@{}}{{35.95}\hspace{-0.20cm}}\end{tabular}
& \begin{tabular}[c]{@{}c@{}}{{42.38}\hspace{-0.20cm}}\end{tabular}
& \begin{tabular}[c]{@{}c@{}}{{25.80}\hspace{-0.20cm}}\end{tabular}
& \begin{tabular}[c]{@{}c@{}}{{37.19}\hspace{-0.20cm}}\end{tabular}
\\
&  \multirow{1}{*}{{Llama+ICL}}                                
& \begin{tabular}[c]{@{}c@{}}{{50.53}\hspace{-0.20cm}}\end{tabular}
& \begin{tabular}[c]{@{}c@{}}{{59.31}\hspace{-0.20cm}}\end{tabular}
& \begin{tabular}[c]{@{}c@{}}{{27.92}\hspace{-0.20cm}}\end{tabular}
& \begin{tabular}[c]{@{}c@{}}{{37.79}\hspace{-0.20cm}}\end{tabular}
\\ 
&  \multirow{1}{*}{{{\algname{}}}}            
& \begin{tabular}[c]{@{}c@{}}{\textbf{56.50}\hspace{-0.20cm}}\end{tabular}
& \begin{tabular}[c]{@{}c@{}}{\textbf{66.79}\hspace{-0.20cm}}\end{tabular}
& \begin{tabular}[c]{@{}c@{}}{\textbf{28.39}\hspace{-0.20cm}}\end{tabular}
& \begin{tabular}[c]{@{}c@{}}{\textbf{39.57}\hspace{-0.20cm}}\end{tabular}
\\
\arrayrulecolor{black}\specialrule{1.2pt}{1pt}{1.0pt}
\end{tabular} 
}
\vspace*{-1.1cm}
\label{tbl:effct_cqr}
\end{wraptable}
\subsection{Analysis of Pseudo Reference Generation}


\subsubsection{Effect of Response Refinement through\\CQR's Dual Role}
Table~\ref{tbl:effct_cqr} builds upon the analysis in Figure~\ref{fig:sum-result}, comparing \algname{} and its single-role variants. 
Overall, these variants underperform compared to \algname{}, with Llama+ICL exhibiting declines of 11.62\% in pseudo reference accuracy and 5.07\% in retrieval accuracy. This result indicates \algname{}'s capability to accurately refine responses for pseudo reference generation by leveraging the alignment between response refinement and retrieval objective. More results are provided in Appendix~\ref{sec:app_refine_methods}.

\subsubsection{Effect of Iterative Optimization}
\begin{wraptable}{r}{0.34\textwidth}
\vspace{-5.4em} 
\caption{Effect of \emph{iterative} optimization within \algname{}.} 
\def\arraystretch{1.3}
\centering
\vspace*{-0.4cm}
\resizebox{0.99\linewidth}{!}{%
\begin{tabular}[c]
{@{}cc|cc|ccc@{}}
\arrayrulecolor{black}\specialrule{1.2pt}{0.75pt}{2.5pt}
\multirow{2}{*}{\hspace{-0.cm}{ Data }\hspace{-0.cm}}  & \multirow{2}{*}{\hspace{-0.2cm}\makecell[c]{\textbf{Pseudo Ref.}\\\textbf{Updates}}\hspace{-0.1cm}} 
& \multicolumn{2}{c|}{{\textbf{Pseudo Acc.}}} 
& \multicolumn{2}{c}{{\textbf{Retrieval Acc.}}} 
\\ 
& & \multicolumn{1}{c}{{MRR}\hspace{-0.20cm}}
& \multicolumn{1}{c|}{{R@5}\hspace{-0.20cm}}
& \multicolumn{1}{c}{{MRR}\hspace{-0.20cm}}
& \multicolumn{1}{c}{{R@5}\hspace{-0.20cm}}
\\
\arrayrulecolor{black}\specialrule{1pt}{1.5pt}{1pt}
\arrayrulecolor{black}\specialrule{1pt}{1pt}{3.0pt}
\multirow{4}{*}{\vspace{0.4cm}\rotatebox{90}{\small SciConvQA}} 
& \multirow{1}{*}{1}                                
& \begin{tabular}[c]{@{}c@{}}{{38.28}\hspace{-0.20cm}}\end{tabular}
& \begin{tabular}[c]{@{}c@{}}{{47.73}\hspace{-0.20cm}}\end{tabular}
& \begin{tabular}[c]{@{}c@{}}{{17.55}\hspace{-0.20cm}}\end{tabular}
& \begin{tabular}[c]{@{}c@{}}{{26.05}\hspace{-0.20cm}}\end{tabular}
\\
&  \multirow{1}{*}{2}                                
& \begin{tabular}[c]{@{}c@{}}{{46.21}\hspace{-0.20cm}}\end{tabular}
& \begin{tabular}[c]{@{}c@{}}{{54.96}\hspace{-0.20cm}}\end{tabular}
& \begin{tabular}[c]{@{}c@{}}{{19.53}\hspace{-0.20cm}}\end{tabular}
& \begin{tabular}[c]{@{}c@{}}{{26.92}\hspace{-0.20cm}}\end{tabular}
\\ 
&  \multirow{1}{*}{{3}}            
& \begin{tabular}[c]{@{}c@{}}{\textbf{50.05}\hspace{-0.20cm}}\end{tabular}
& \begin{tabular}[c]{@{}c@{}}{\textbf{59.75}\hspace{-0.20cm}}\end{tabular}
& \begin{tabular}[c]{@{}c@{}}{\textbf{20.06}\hspace{-0.20cm}}\end{tabular}
& \begin{tabular}[c]{@{}c@{}}{\textbf{28.64}\hspace{-0.20cm}}\end{tabular}
\\ \addlinespace[0.3ex]\cdashline{1-7}\addlinespace[0.8ex] 
\multirow{4}{*}{\vspace{0.4cm}\rotatebox{90}{\small TopiOCQA}} 
& \multirow{1}{*}{1}                                
& \begin{tabular}[c]{@{}c@{}}{{39.33}\hspace{-0.20cm}}\end{tabular}
& \begin{tabular}[c]{@{}c@{}}{{44.14}\hspace{-0.20cm}}\end{tabular}
& \begin{tabular}[c]{@{}c@{}}{{25.87}\hspace{-0.20cm}}\end{tabular}
& \begin{tabular}[c]{@{}c@{}}{{37.43}\hspace{-0.20cm}}\end{tabular}
\\
&  \multirow{1}{*}{2}                                
& \begin{tabular}[c]{@{}c@{}}{{55.24}\hspace{-0.20cm}}\end{tabular}
& \begin{tabular}[c]{@{}c@{}}{{65.32}\hspace{-0.20cm}}\end{tabular}
& \begin{tabular}[c]{@{}c@{}}{{28.08}\hspace{-0.20cm}}\end{tabular}
& \begin{tabular}[c]{@{}c@{}}{{39.25}\hspace{-0.20cm}}\end{tabular}
\\ 
&  \multirow{1}{*}{3}            
& \begin{tabular}[c]{@{}c@{}}{\textbf{56.50}\hspace{-0.20cm}}\end{tabular}
& \begin{tabular}[c]{@{}c@{}}{\textbf{66.79}\hspace{-0.20cm}}\end{tabular}
& \begin{tabular}[c]{@{}c@{}}{\textbf{28.39}\hspace{-0.20cm}}\end{tabular}
& \begin{tabular}[c]{@{}c@{}}{\textbf{39.57}\hspace{-0.20cm}}\end{tabular}
\\
\arrayrulecolor{black}\specialrule{1.2pt}{1pt}{1.0pt}
\end{tabular} }
\vspace*{-0.75cm}

\label{tbl:effct_stab}
\end{wraptable}
Table~\ref{tbl:effct_stab} presents the effect of the iterative procedure in Figure~\ref{fig:our-flow}. 
In general, iteratively alternating between pseudo reference generation and model optimization progressively improves pseudo reference accuracy and retrieval performance, with convergence observed at the third update. This result indicates the importance of the iterative procedure in exploiting the synergy between pseudo reference passage quality and model optimization.

\subsubsection{Effect of Query-Forming Template}


Table~\ref{tbl:ablation_template} compares the effect of the query-forming template with its variants. \textit{Variant~1} omits the query-forming process with Prompt~\ref{prompt_temp}, directly using raw responses in Eq.~(\ref{eq:response-refinement}). \textit{Variant~2} excludes the response ``\((\{a_t\})\)'' in Prompt~\ref{prompt_temp}. \textit{Variant~3} and \textit{Variant~4} deactivate the effects of the phrases ``state the main points'' and ``last response \((\{a_t\})\)'', respectively, from the refined response generated by the complete version of Prompt \ref{prompt_temp}.

\vspace*{-0.1cm}
\begin{wraptable}{r}{0.47\textwidth}
\vspace{-1.3em} 
\caption{Effect of the query-transforming template on pseudo reference accuracy.} 
\def\arraystretch{1.1}
\small
\centering
\vspace*{-0.25cm}
\resizebox{0.9999\linewidth}{!}{%
\begin{tabular}[c]
{@{}l|cc|cc|cc@{}}
\arrayrulecolor{black}\specialrule{1.2pt}{0.75pt}{2.5pt}
\multirow{2}{*}{\hspace{-0.0cm}{{Variants}}} 
& \multicolumn{2}{c|}{{ \textbf{SciConvQA}}} 
& \multicolumn{2}{c|}{{ \textbf{TopiOCQA}}} 
& \multirow{2}{*}{\hspace{-0.0cm}{\textit{Degrade}}} 
\\ 
& \multicolumn{1}{c}{{MRR}}
& \multicolumn{1}{c|}{{R@5}}
& \multicolumn{1}{c}{{MRR}}
& \multicolumn{1}{c|}{{R@5}}
\\
\arrayrulecolor{black}\specialrule{1pt}{1.5pt}{1pt}
\arrayrulecolor{black}\specialrule{1pt}{1pt}{3.0pt}
\multirow{1}{*}{{1. w/o Prompt 1\!\!\!}}                                
& \begin{tabular}[c]{@{}c@{}}{{45.38}}\end{tabular}
& \begin{tabular}[c]{@{}c@{}}{{54.35}}\end{tabular}
& \begin{tabular}[c]{@{}c@{}}{{54.57}}\end{tabular}
& \begin{tabular}[c]{@{}c@{}}{{64.16}}\end{tabular}
& \begin{tabular}[c]{@{}c@{}}{\textit{6.97}\%}\end{tabular}
\\
\multirow{1}{*}{{2. w/o ``$(\{a_t\})$'' in Prompt 1\!\!\!}}                             
& \begin{tabular}[c]{@{}c@{}}{{47.29}}\end{tabular}
& \begin{tabular}[c]{@{}c@{}}{{57.35}}\end{tabular}
& \begin{tabular}[c]{@{}c@{}}{{50.00}}\end{tabular}
& \begin{tabular}[c]{@{}c@{}}{{60.26}}\end{tabular}
& \begin{tabular}[c]{@{}c@{}}{\textit{8.46}\%}\end{tabular}
\\
\multirow{1}{*}{{3. w/o ``state the main points''\!\!\!}}                             
& \begin{tabular}[c]{@{}c@{}}{{43.62}}\end{tabular}
& \begin{tabular}[c]{@{}c@{}}{{52.10}}\end{tabular}
& \begin{tabular}[c]{@{}c@{}}{{54.29}}\end{tabular}
& \begin{tabular}[c]{@{}c@{}}{{63.42}}\end{tabular}
& \begin{tabular}[c]{@{}c@{}}{\textit{9.70}\%}\end{tabular}
\\
\multirow{1}{*}{{4. w/o ``last response $\{a_t\}$''\!\!\!}}                                
& \begin{tabular}[c]{@{}c@{}}{{46.64}}\end{tabular}
& \begin{tabular}[c]{@{}c@{}}{{55.69}}\end{tabular}
& \begin{tabular}[c]{@{}c@{}}{{47.16}}\end{tabular}
& \begin{tabular}[c]{@{}c@{}}{{56.41}}\end{tabular}
& \begin{tabular}[c]{@{}c@{}}{\textit{13.20}\%}\end{tabular}
\\
\midrule
\multirow{1}{*}{{\textbf{\algname{}}}}            
& \begin{tabular}[c]{@{}c@{}}{\textbf{50.05}}\end{tabular}
& \begin{tabular}[c]{@{}c@{}}{\textbf{59.75}}\end{tabular}
& \begin{tabular}[c]{@{}c@{}}{\textbf{56.50}}\end{tabular}
& \begin{tabular}[c]{@{}c@{}}{\textbf{66.79}}\end{tabular}
& \begin{tabular}[c]{@{}c@{}}{{-}}\end{tabular}
\\
\arrayrulecolor{black}\specialrule{1.2pt}{1pt}{1.0pt}
\end{tabular} }
\vspace*{-0.2cm}

\label{tbl:ablation_template}
\vspace*{-0.2cm}
\end{wraptable}

Across all variants, performance consistently degrades compared to \algname{}. The decline in {Variant~1} shows the importance of structuring responses into query forms to exploit CQR's effective reformulation. The result for {Variant~2} highlights the significance of explicitly integrating the raw response into the template for response-relevant context extraction. Finally, the degradation in {Variant~3} and {Variant~4} indicates that the two phrases contribute complementary information crucial for effective refinement. Additional results are presented in Appendix~\ref{sec:app_template}.

\begin{figure*}[t!]
\includegraphics[width=1.00\linewidth]{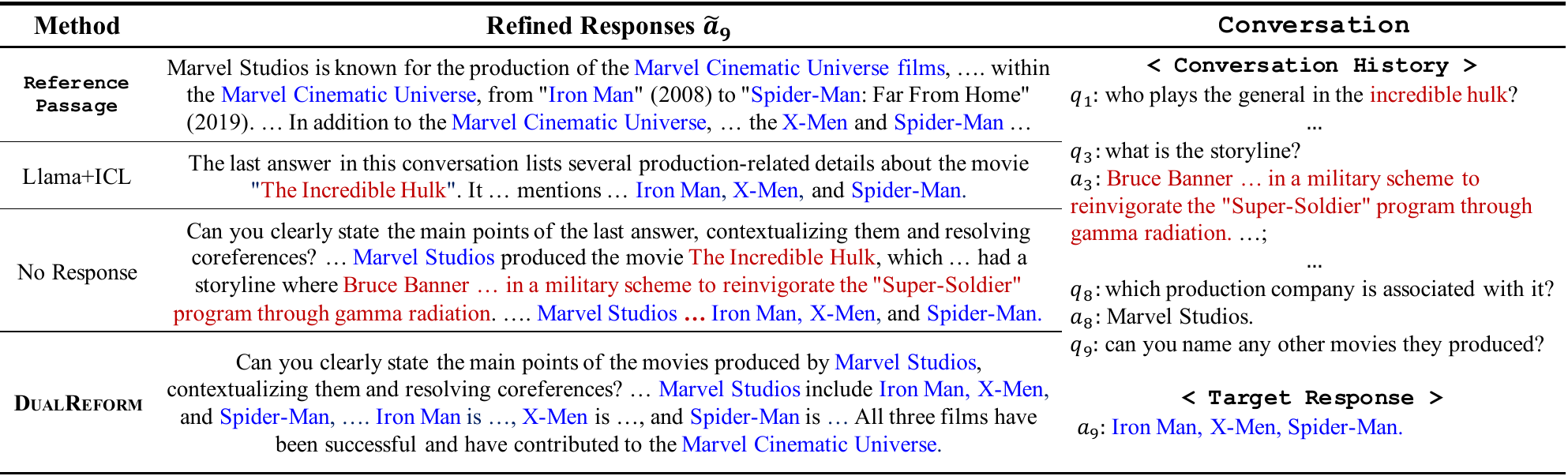}
\vspace*{-0.6cm}
\caption{ 
Examples of refined responses generated by different methods on TopiOCQA. Fragments strongly aligned with the reference passage are highlighted in \textcolor{blue}{blue}, while fragments with weaker connections~(e.g., off-topic elements referring to previous conversation topics) are marked in \textcolor{red}{red}.
}
\vspace*{-0.5cm}
\label{fig:refine_exs}
\end{figure*}

\subsubsection{Qualitative Analysis}

Figure~\ref{fig:refine_exs} presents refined responses generated by \algname{} and its variants for a conversation from TopiOCQA. Compared to other variants, \algname{} demonstrates superior contextual understanding by extracting relevant context\,(e.g., ``Marvel Studios'') and adding details regarding the response\,(e.g., ``Marvel Cinematic Universe''). In contrast, Llama+ICL and No Response often rely on the less relevant context\,(e.g., ``Hulk'').
Additional results are provided in Appendix~\ref{sec:app_examples}.


\subsection{Response Generation Accuracy}
\begin{wraptable}{r}{0.45\textwidth}
\vspace{-2.25em} 
\caption{Response generation accuracy with passages retrieved by different CQR methods.} 
\def\arraystretch{0.8}
\small
\centering
\vspace*{-0.2cm}
\resizebox{0.97\linewidth}{!}{%
\begin{tabular}[c]
{@{}c|ccccc@{}}
\arrayrulecolor{black}\specialrule{1.2pt}{0.75pt}{2.5pt}
\multirow{2}{*}{\hspace{-0.0cm}{{CQR Methods}}} 
& \multicolumn{4}{c}{{ \textbf{Generation Accuracy}}} 
\\ 
& \multicolumn{1}{c}{{LLMEval}}
& \multicolumn{1}{c}{{ROUGE-1}}
& \multicolumn{1}{c}{{ROUGE-L}}
& \multicolumn{1}{c}{{BertScore}}
\\
\arrayrulecolor{black}\specialrule{1pt}{1.5pt}{1pt}
\arrayrulecolor{black}\specialrule{1pt}{1pt}{3.0pt}
\multirow{1}{*}{{LLM-IQR}}                                
& \begin{tabular}[c]{@{}c@{}}{ {27.76} }\end{tabular}
& \begin{tabular}[c]{@{}c@{}}{ {20.07} }\end{tabular}
& \begin{tabular}[c]{@{}c@{}}{ {17.64} }\end{tabular}
& \begin{tabular}[c]{@{}c@{}}{ {86.01} }\end{tabular}
\\
\multirow{1}{*}{{HyDE-LLM}}                                
& \begin{tabular}[c]{@{}c@{}}{ {29.77} }\end{tabular}
& \begin{tabular}[c]{@{}c@{}}{ {22.18} }\end{tabular}
& \begin{tabular}[c]{@{}c@{}}{ {19.39} }\end{tabular}
& \begin{tabular}[c]{@{}c@{}}{ {86.25} }\end{tabular}
\\
\multirow{1}{*}{{LLM4CS-CoT}}                                
& \begin{tabular}[c]{@{}c@{}}{ {26.42} }\end{tabular}
& \begin{tabular}[c]{@{}c@{}}{ {20.23} }\end{tabular}
& \begin{tabular}[c]{@{}c@{}}{ {17.58} }\end{tabular}
& \begin{tabular}[c]{@{}c@{}}{ {85.95} }\end{tabular}
\\
\multirow{1}{*}{{T5QR}}                                
& \begin{tabular}[c]{@{}c@{}}{ {28.43} }\end{tabular}
& \begin{tabular}[c]{@{}c@{}}{ {19.96} }\end{tabular}
& \begin{tabular}[c]{@{}c@{}}{ {17.76} }\end{tabular}
& \begin{tabular}[c]{@{}c@{}}{ {85.74} }\end{tabular}
\\
\multirow{1}{*}{{ConvGQR}}                                
& \begin{tabular}[c]{@{}c@{}}{ {23.08} }\end{tabular}
& \begin{tabular}[c]{@{}c@{}}{ {20.81} }\end{tabular}
& \begin{tabular}[c]{@{}c@{}}{ {18.37} }\end{tabular}
& \begin{tabular}[c]{@{}c@{}}{ {85.95} }\end{tabular}
\\
\multirow{1}{*}{{HyDE-FT}}                                
& \begin{tabular}[c]{@{}c@{}}{ {23.17} }\end{tabular}
& \begin{tabular}[c]{@{}c@{}}{ {19.81} }\end{tabular}
& \begin{tabular}[c]{@{}c@{}}{ {17.40} }\end{tabular}
& \begin{tabular}[c]{@{}c@{}}{ {86.02} }\end{tabular}
\\
\multirow{1}{*}{{RetPo}}                                
& \begin{tabular}[c]{@{}c@{}}{ {27.09} }\end{tabular}
& \begin{tabular}[c]{@{}c@{}}{ {21.63} }\end{tabular}
& \begin{tabular}[c]{@{}c@{}}{ {18.59} }\end{tabular}
& \begin{tabular}[c]{@{}c@{}}{ {86.17} }\end{tabular}
\\
\midrule
\multirow{1}{*}{{\textbf{\algname{}}}}            
& \begin{tabular}[c]{@{}c@{}}{ \textbf{34.45} }\end{tabular}
& \begin{tabular}[c]{@{}c@{}}{ \textbf{24.37} }\end{tabular}
& \begin{tabular}[c]{@{}c@{}}{ \textbf{21.41} }\end{tabular}
& \begin{tabular}[c]{@{}c@{}}{ \textbf{86.33} }\end{tabular}
\\
\arrayrulecolor{black}\specialrule{1.2pt}{1pt}{1.0pt}
\end{tabular} }
\vspace*{-0.2cm}
\label{tbl:gen_acc}
\vspace*{-0.2cm}
\end{wraptable}
Table~\ref{tbl:gen_acc} reports the generation accuracy using passages retrieved by different CQR baselines on SciConvQA, with Llama-3.1-8b-instruct as the generator and BM25 as the retriever. \algname{} achieves superior accuracy compared to baselines, demonstrating the consequence of its enhanced retrieval performance for downstream generation. Additional results on TopiOCQA are provided in Appendix~\ref{sec:app_gen}. 





\section{Conclusion}
\label{sec:conclusion}

We propose \algname{}, a novel reference-free preference optimization framework for CQR, which eliminates the reliance on reference passages. Fully taking advantage of the \textit{dual-role} of CQR, \algname{} generates accurate \textit{pseudo} reference passages to guide preference optimization. Empirical results demonstrate the broad applicability of \algname{} across diverse conversational domains, without reliance on reference passages. Overall, we believe that our work sheds light on the importance of practical CQR approaches for diverse real-world conversational scenarios.

\section*{Limitations}
One limitation of the proposed \algname{} framework is its reliance on a fixed set of candidate reformulated queries generated by ChatGPT to derive preference feedback during training. Incorporating candidate queries produced by the trained CQR model itself could introduce greater diversity of candidate queries and enhance the quality of preference feedback. Investigating the performance gains from this augmentation-based strategy is left for future work.

Additionally, integrating more advanced preference optimization techniques presents a potential important direction for improvement. In the context of CQR, retrieval effectiveness may not be the sole determinant of preferences, and multi-dimensional feedback, e.g., conciseness of queries, can play a critical role. While DPO is widely adopted for preference optimization, it is limited in handling multi-dimensional feedback due to its reliance on a single-dimensional preference structure. Recent work, such as CPO\,\cite{cpo} and Sequential Alignment\,\cite{spo}, offers promising alternatives that may better address this limitation. 
Future work will investigate their potential with CQR preference learning.

\section*{Ethics Statement}
This work primarily aims at generating pseudo reference passages directly from data itself, without relying on human annotators, posing no ethical concerns during training. In creating the SciConvQA benchmark, we adhere to a common LLM-based conversation generation protocol detailed in the prior literature\,\cite{topiocqa}. Therefore, we do not anticipate any ethical violations or neagtive societal consequences resulting from this work.

\bibliography{iclr2026_conference}
\bibliographystyle{iclr2026_conference}



\medskip





\clearpage

\appendix

\begin{center}
\Large \textsc{References Indeed Matter?\\Reference-Free Preference Optimization for\\Conversational Query Reformulation} \\
\vspace{0.1cm}
\end{center}

\section{Proofs of Theoretical Results}
\label{sec:app-proof}

\subsection{Proofs of Lemmas~\ref{lemma:err-single} and~\ref{lemma:err-dual}}
\label{subsec:lemma}
The formal statement of Assumption~\ref{assume:expansion} is given in Assumptions~\ref{assume:app_expansion} and~\ref{assume:app_separation}. Let $\mathcal{B}(x)$ denote the set of all possible augmentations of an input $x$ produced by a data augmentation function $\mathcal{A}(\cdot)$.

\begin{assumption}[\sc Expansion]
An example can reach c neighboring instances on average. Formally, $\mathbb{E}_{x} [ |\mathcal{N}(x)| ]=c$ with $c > 3$, where $\mathcal{N}(x) = \{x^\prime: \mathcal{B}(x) \cap \mathcal{B}(x^\prime) \neq \emptyset\}$. 
\label{assume:app_expansion}
\end{assumption}

\begin{assumption}[\sc Separation]
The average proportion of neighboring instances belonging to a different reference passage is $\mu$, which is negligibly small. Formally, $\mathbb{E}_{x} [ \mathbbm{1}_{[\exists  x^\prime \in \mathcal{B}(x) \text{   such that   } G^*(x) \neq G^*(x^\prime)] } ] = \mu$, where $G^*(x)$ indicates the ground-truth reference passage for $x$. 
\label{assume:app_separation}
\end{assumption}

\noindent\textbf{The Proof of Lemma~\ref{lemma:err-single}.} 
Under these assumptions, the pseudo labeling theory\,\citep{wei2021theoretical} provides pseudo label correction guarantees, improving training accuracy by rectifying erroneous pseudo labels, as presented in Lemma~\ref{lemma:err-pl}.

\begin{lemma}[\sc {Pseudo Label Denoising Bound}\,\citep{wei2021theoretical}]
\label{lemma:err-pl}
Suppose that Assumptions~\ref{assume:app_expansion} and~\ref{assume:app_separation} hold. Then, the training error of any minimizer $\hat{\theta}$ under the pseudo labeler $\theta_{PL}$ is bounded by
\begin{equation}
\begin{gathered}
{Err}(\hat{\theta}) \le {{2}\over{c-1}} {Err}(\theta_{PL}) + {{2c}\over{c-1}} \mu.
\end{gathered}
\label{eq:pl_bound}
\end{equation}
\end{lemma} 

The full proof of this result is presented in Appendix A.1 of \citep{wei2021theoretical}. In our setting, the minimizer corresponds to the CQR model under a single-role configuration, which employs an LLM as the pseudo labeler. Assigning $\hat{\theta}=\theta_{single}$ and $\theta_{PL}=\theta_{LLM}$ in Eq.~(\ref{eq:pl_bound}) gives ${Err}(\theta_{single}) \le {{2}\over{c-1}} {Err}(\theta_{LLM}) + {{2c}\over{c-1}} \mu$, which matches the single-role configuration bound in Lemma~\ref{lemma:err-single}.

Additionally, following \citep{wei2021theoretical}, the general preference optimization framework for the CQR model employs \textit{input consistency regularization} in building preference pairs, ensuring consistent preferences across input transformations. Aligned with standard data augmentation practices in language-based tasks\,\citep{jiao2025preference, retpo}, this framework employs {LLM-based augmentation} to produce various paraphrases of query reformulations by prompting an LLM multiple times, as shown in Eq.~(\ref{eq:pref}). To enforce a substantial gap between winning and losing examples in preference pairs, high-effectiveness reformulations\,(winners) and low-effectiveness reformulations\,(losers) are sampled, following \citep{song-etal-2023-enhancing, yuan2024selfrewarding, xu2023some}. Each preference pair thus includes differently rephrased queries while preserving consistent preference ordering over these rephrased ones. Consequently, through preference optimization in Eq.~(\ref{eq:dpo}), the CQR model learns to maintain consistent preference rankings across these input transformations, thereby demonstrating the principle of input consistency regularization.

\smallskip
\noindent\textbf{Proof of Lemma~\ref{lemma:err-dual}.} 
Recall that the dual-role configuration employs the CQR model, trained under the single-role configuration, as its pseudo labeler, while the single-role configuration uses the backbone LLM model as its pseudo labeler. By leveraging this cascaded structure, we establish the error bound of the dual-role configuration in terms of the backbone LLM's error bound. 

First, the error bound for the dual-role configuration relative to the single-role CQR model is given by
 \begin{equation}
    {Err}(\theta_{dual}) \le {{2}\over{c-1}} {Err}(\theta_{single}) + {{2 c}\over{c-1}}\mu.
    \label{eq:dual_single}
\end{equation}
Next, applying the error bound of the single-role configuration in Eq.~(\ref{eq:error_bound_single}) to $Err(\theta_{single})$ in the right-hand side of Eq.~(\ref{eq:dual_single}), we obtain
\begin{equation}
\begin{split}
{Err}(\theta_{dual}) &\le {{2}\over{c-1}} [ {{2}\over{c-1}} {Err}(\theta_{LLM}) + {{2 c}\over{c-1}}\mu ] + {{2 c}\over{c-1}} \mu \\
&\le \left({{2}\over{c-1}}\right)^2 {Err}(\theta_{LLM}) + \left({{2 c}\over{c-1}}\right)\left({{c+1}\over{c-1}}\right) \mu,
\end{split}
\label{eq:error_bound_rearrange}
\end{equation}
which completes the proof of Lemma~\ref{lemma:err-dual}.

\subsection{Proof of Theorem~\ref{thm:comparison}}
\label{subsec:thm}
Let $\overline{\mathrm{Err}}(\theta)$ denote the theoretical upper bound on the training error for a configuration $\theta$. 
Under Assumption~\ref{assume:expansion}, where $\mu \approx 0$, we obtain approximate upper bounds for the respective errors. From Lemma~\ref{lemma:err-single}\,(Single-Role Bound), Eq.~(\ref{eq:error_bound_single}) becomes
\begin{equation}
    \overline{Err}(\theta_{{single}}) \approx 
    \frac{2}{c-1} \, {Err}(\theta_{{LLM}}).
\end{equation}
From Lemma~\ref{lemma:err-dual}\,(Dual-Role Bound), Eq.~(\ref{eq:error_bound_dual}) becomes
\begin{equation}
    \overline{Err}(\theta_{{dual}}) \approx 
    \left(\frac{2}{c-1}\right)^2 {Err}(\theta_{{LLM}}).
\end{equation}
Here, ``$\approx$'' indicates approximate upper bounds for the respective error terms. Comparing the two approximate upper bounds, we obtain 
\begin{equation}
    \left(\frac{2}{c - 1}\right)^2 \overline{Err}(\theta_{{LLM}}) 
    < \frac{2}{c - 1} \, \overline{Err}(\theta_{{LLM}}) 
    \quad\text{if and only if}\quad 
    c > 3.
\end{equation}
Therefore, for \(c>3\), as stipulated in Assumption~\ref{assume:expansion}, the error bound under the dual-role configuration is strictly smaller than the error bound of the single-role configuration. This completes the proof of Theorem~\ref{thm:comparison}.
\section{Definition of Retrieval Score}
\label{sec:ret-score}

We evaluate candidate query reformulations \(\{\tilde{q}_t^i\}_{i=1}^M\) using pseudo reference passages \(\tilde{G}_t\) based on their retrieval scores. The retrieval score \(s(\tilde{q}_t^i \mid \tilde{G}_t)\) of a candidate indicates how accurately \(\tilde{q}_t^i\) retrieves the passages to contain \(\tilde{G}_t\). Specifically, our assessment focuses on three dimensions: (1) \textit{coverage}\,(Definition \ref{def:coverage}), reflecting how \textit{comprehensively} reference passages are retrieved within specific cutoffs; (2) \textit{immediacy}\,(Definition \ref{def:immediacy}), reflecting how \textit{early} a reference passage appears in the ranking; and (3) \textit{concordance}\,(Definition \ref{def:relevance}), reflecting how well the ranking \textit{aligns} with the ideal relevance ordering of reference passages. Finally, these three scores are combined in a single weighted value\,(Definition \ref{def:retrieval-score}). 

\begin{definition}{\sc (Coverage Score)}
\label{def:coverage}
The coverage score \(s_{\mathrm{cov}}(\tilde{q}_t^i \mid \tilde{G}_t)\) is computed as
\begin{equation}
\label{eq:coverage}
s_{\mathrm{cov}}(\tilde{q}_t^i \mid \tilde{G}_t) = \frac{1}{|K|} \sum_{k \in K} \mathrm{Recall@k} \big(R(\tilde{q}_t^i), \tilde{G}_t\big),
\end{equation}
where \(K\) is a predefined set of cutoff values.
\end{definition}

\begin{definition}{\sc (Immediacy Score)}
\label{def:immediacy}
The immediacy score \(s_{\mathrm{imm}}(\tilde{q}_t^i \mid \tilde{G}_t)\) is computed as
\begin{equation}
\label{eq:immediacy}
s_{\mathrm{imm}}(\tilde{q}_t^i \mid \tilde{G}_t) = \mathrm{MRR}\bigl(R(\tilde{q}_t^i), \tilde{G}_t\bigr).
\end{equation}
\end{definition}

\begin{definition}{\sc (Concordance Score)}
\label{def:relevance}
The concordance score \(s_{\mathrm{con}}(\tilde{q}_t^i \mid \tilde{G}_t)\) is computed as,
\begin{equation}
\label{eq:relevance}
s_{\mathrm{con}}(\tilde{q}_t^i \mid \tilde{G}_t) = \mathrm{NDCG}\bigl(R(\tilde{q}_t^i), \tilde{G}_t\bigr).
\end{equation}
\end{definition}


\begin{definition}{\sc (Retrieval Score)}
\label{def:retrieval-score} 
For each candidate \(\tilde{q}_t^i\), the retrieval score \(s(\tilde{q}_t^i \mid \tilde{G}_t)\) is a weighted sum of the coverage, immediacy, and concordance scores, as
\begin{equation}
\label{eq:retrieval-score}
s(\tilde{q}_t^i\mid \tilde{G}_t) = \omega_1 s_{\mathrm{cov}}(\tilde{q}_t^i\mid \tilde{G}_t) + \omega_2 s_{\mathrm{imm}}(\tilde{q}_t^i\mid \tilde{G}_t) + \omega_3 s_{\mathrm{con}}(\tilde{q}_t^i\mid \tilde{G}_t),
\end{equation}  
where \(\omega_1, \omega_2, \omega_3 \ge 0\) and \(\omega_1 + \omega_2 + \omega_3 = 1. \)
\end{definition}

Details on the retrieval evaluation metrics are provided in Appendix~\ref{sec:app_metrics}.

\begin{table*}[]
\resizebox{\textwidth}{!}{%
\begin{tabular}{@{}l@{}}
\toprule
\textbf{SciConvQA conversation} \\ \midrule
Q: How do genomic tools enhance animal breeding programs? \\
A: Genomic tools, such as single nucleotide polymorphism (SNP), have led to a new method known as “genomic selection.” \\ \quad This method utilizes dense SNP genotypes covering the entire genome to predict the breeding value. \\
Q: Which application is used in evaluating these programs? \\
A: ZPLAN+ is used to evaluate and optimize these programs. \\
Q: What parameters does it consider? \\
A: It considers genetic and economic parameters. \\
Q: Can you tell me more about the types of strategies it models? \\
A: It models four selection strategies: the current conventional program and three based on genomic information. \\
Q: What's unique about the final approach among them? \\
A: The final approach, GS3, is unique because it combines pedigree, genomic enhanced breeding values, \\ \quad performance, and progeny information. \\
Q: How were the male candidates evaluated in this scheme? \\
A: They were genotyped and their selection was based on performance tests and GEBV. \\
Q: Do we have information about the costs associated with these methods? \\
A: Yes, the cost of genotyping was assumed to be \$120 per pig. \\
Q: How does this cost compare to more traditional testing methods? \\
A: Performance testing costs \$55 per tested pig, which is less than genotyping. \\
Q: How many candidates undergo this evaluation when adventurers start the process? \\
A: Initially, 1,000 male candidates were considered in the genomic selection process. \\
Q: In the final phase of selection, how many of these males are retained? \\
A: In the final phase, 23 senior boars were retained. \\
Q: How does the conventional program handle progeny information differently than genomics? \\
A: The conventional program uses progeny records without considering genetic marker information. \\
Q: Can you summarize which traits measured the field performance? \\
A: Traits measured include average daily gain, back fat thickness, and feed conversion rate. \\
Q: Are these the same for testing the station? \\
A: No, station tests focus on meat quality traits like pH, meat color (L*), and intra-muscular fat. \\
Q: In this optimized workflow, what benefit is sought above all? \\
A: The primary goal is high genetic gains with low breeding costs. \\ \bottomrule
\end{tabular}%
}
\caption{Example of a SciConvQA conversation. The conversation is generated from \citet{lopez2016optimization}, published in a renowned journal and stored as `JAKO201614137726690.json' in the scientific journal dataset described in Section~\ref{subsec: sciconvqa}.}
\label{table:SciConvQA_full_conv}
\end{table*}

\section{Dataset Details}
\label{sec:sciconvqa}

\subsection{General-Domain: QReCC and TopiOCQA}
The QReCC dataset\,\citep{qrecc} contains 14K multi-turn conversations with a total of 80K question-answer pairs, aiming to retrieve reference passages from a large corpus of 54 million passages. Similarly, the TopiOCQA dataset\,\citep{topiocqa} includes 3.9K conversations featuring topic shifts, comprising 51K question-answer pairs. Its passage collection is derived from Wikipedia and consists of approximately 20 million passages. For both datasets, small random subsets of the training data were used to construct the validation sets. While these datasets are well-suited for general-domain conversational contexts, they lack focus on domain-specific applications such as scientific question answering.

\subsection{Specialized-Domain: SciConvQA}
\label{subsec: sciconvqa}
Information-seeking conversations span a wide range of domains, from general topics to specialized areas like science, reflecting diverse user interests. To evaluate existing CQR methods and \algname{}, we introduce the SciConvQA dataset, composed of information-seeking conversations generated from renowned scientific journals.

The conversation generation process follows the protocol described in Appendix A of the TopiOCQA\,\citep{topiocqa} paper, which provides the methodology for creating conversational datasets. While TopiOCQA relies on crowd-sourced annotations, it incurs high costs or risks of diminished quality when applied to specialized scientific domains. Hence, we employ gpt-4o-2024-08-06\,\citep{openai2023chatgpt} for automated conversation generation, followed by post-hoc manual quality validation. Overall, the conversation generation process involves two steps: (1) selecting a scientific journal as the seed topic and (2) generating questioner-answerer interactions. Table~\ref{table:SciConvQA_full_conv} shows a representative conversation from SciConvQA.

\smallskip
\noindent\textbf{Seed Topics and Document Collection.} 
SciConvQA is constructed using scientific journal data provided by the Korea Institute of Science and Technology Information\,(KISTI), a government-funded research institute. The scientific journal dataset, accessible at \url{https://aida.kisti.re.kr/data/b22c73ed-fa19-47b0-87b3-a509df8380e5}, includes a total of 481,578 academic articles, comprising both Korean and English publications. Detailed information about the dataset construction is available at the linked source. For our study, we utilize 120,916 English articles to construct the external database corpus, from which a subset is sampled to generate conversations. These articles span 749 diverse scientific fields, including biology, medicine, and architecture.

\smallskip
\noindent\textbf{Conversation Generation.} 
We modify the conversation annotation protocol of TopiOCQA to design a prompt for gpt-4o-2024-08-06, including an in-context demonstration to illustrate the conversation generation process based on a seed topic. The prompt template with its demonstration is shown in Figures~\ref{fig:prompt_SciCovQA}--\ref{fig:conversation_example}. During the conversation generation process, each article serves as a seed topic, and the reference passages for conversation turns are selected from the article. Specifically, for each conversation turn, a reference passage~(``rationale'' in the prompt) is selected as a substring of the article's content that justifies the answer, recorded directly below the corresponding answer, as the demostrastive conversation in Figure~\ref{fig:conversation_example}.

\begin{figure*}[t]
    \centering
    \includegraphics[width=\textwidth]{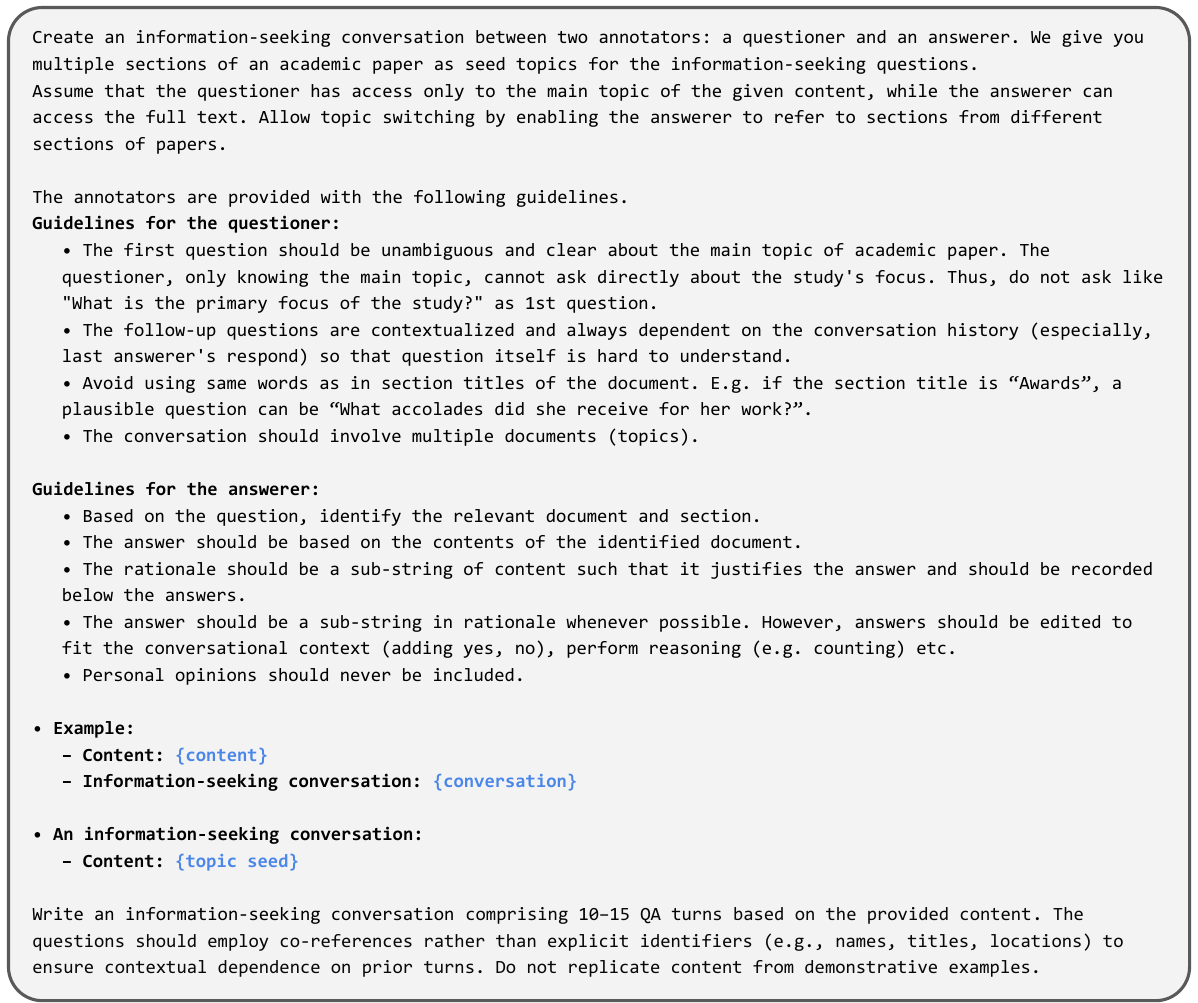}
    \vspace*{-0.5cm}
    \caption{Prompt used for SciConvQA dataset generation. The example of {\color[rgb]{0.2, 0.5, 0.8}\{content\}} is provided in Figure~\ref{fig:content_example}, and the example of {\color[rgb]{0.2, 0.5, 0.8}\{conversation\}} can be found in Figure~\ref{fig:conversation_example}.}
    \label{fig:prompt_SciCovQA}
    \vspace*{-0.3cm}
\end{figure*}

\begin{figure*}[t]
    \centering
    \includegraphics[width=\textwidth]{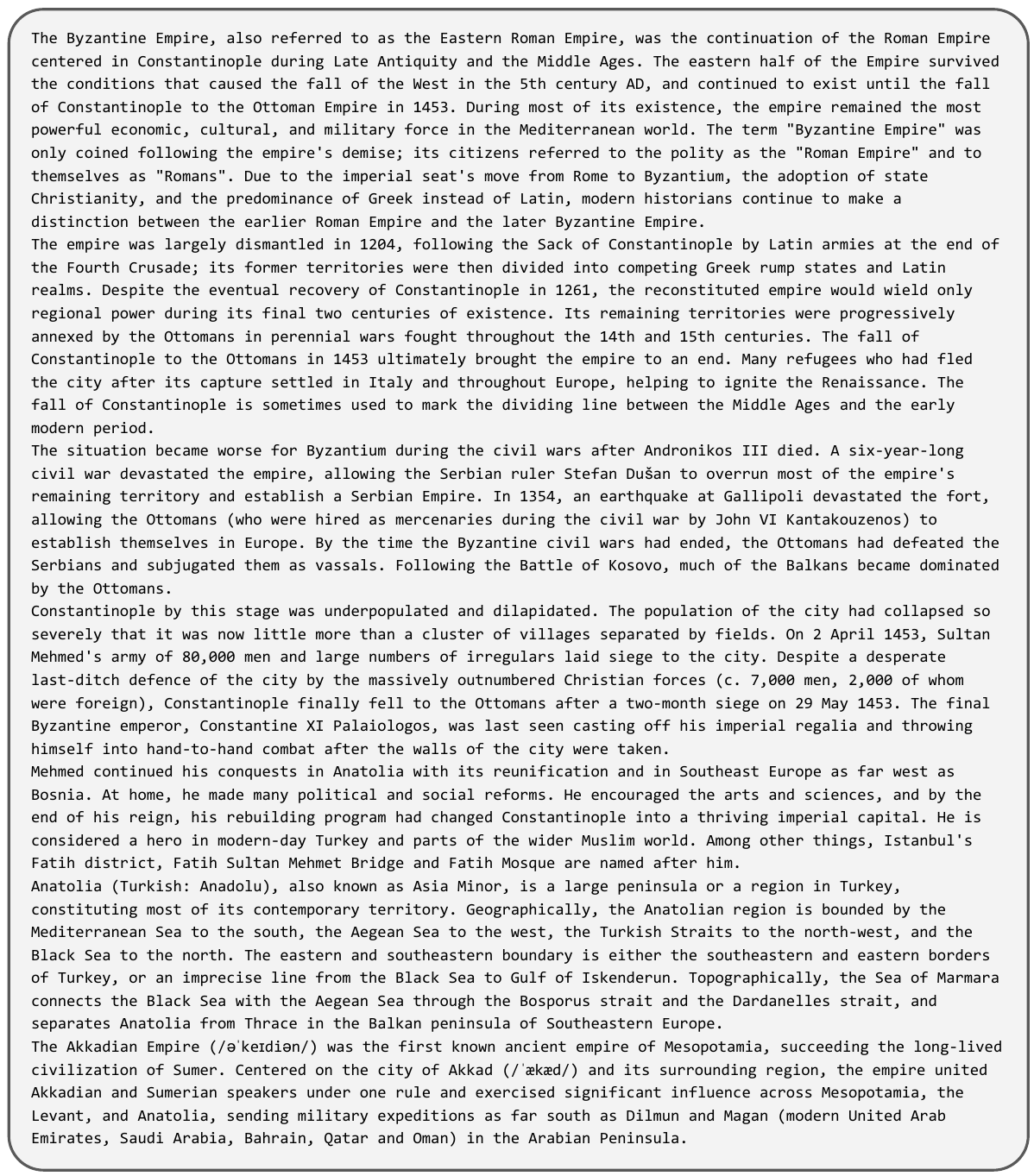}
    \vspace*{-0.5cm}
    \caption{Example of {\color[rgb]{0.2, 0.5, 0.8}\{content\}} in Figure \ref{fig:prompt_SciCovQA}.}
    \label{fig:content_example}
    \vspace*{-0.3cm}
\end{figure*}

\begin{figure*}[t]
    \centering
    \includegraphics[width=\textwidth]{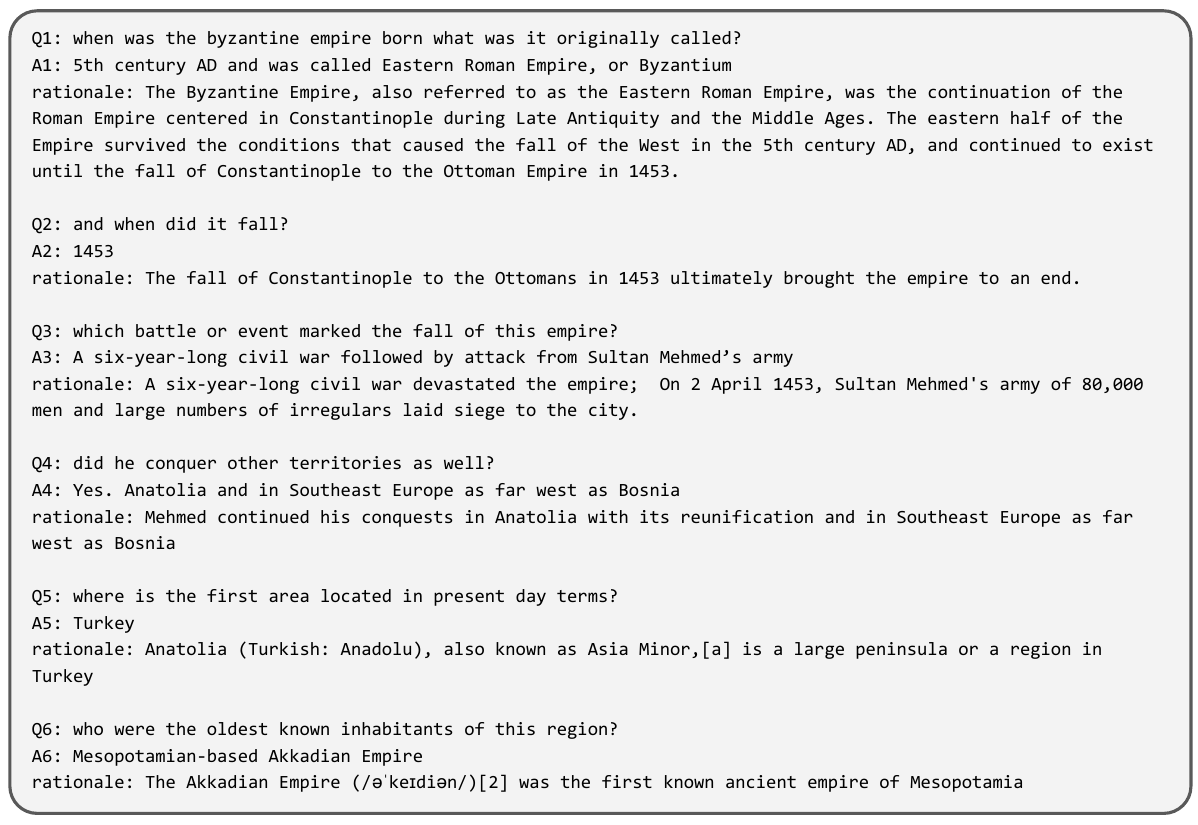}
    \vspace*{-0.6cm}
    \caption{Examples of {\color[rgb]{0.2, 0.5, 0.8}\{conversation\}} in Figure \ref{fig:prompt_SciCovQA}.}
    \label{fig:conversation_example}
\end{figure*}

\smallskip
\noindent\textbf{Passage Database Construction.} 
The passage database is constructed using 120,916 English articles as the retrieval target. Specifically, we employ Langchain's $\mathtt{RecursiveCharacterTextSplitter}$ \!\!\!\!\!\! with a chunk size of 500 and a chunk overlap of 100\,\citep{langchain}, resulting in a database consisting of 1,909,524 passages.

\smallskip
\noindent{\textbf{Post-Processing.}} 
To ensure compatibility with existing datasets~(e.g., TopiOCQA), we standardize the format of the raw conversations generated during the conversation generation process. Due to inconsistencies in the output structure of chat completions, we extract only the relevant content using a custom post-processing pipeline. Furthermore, the passage ID for the ideal reference passage corresponding to each query is assigned by identifying the longest common substring between the generated ``rationale'' and passages in the external passage database. The entire implementation of the post-processing procedure is provided in the \algname{}'s code repository.

\subsection{Exploratory Analysis for SciConvQA}
\noindent\textbf{Data Statistics.} Table~\ref{tbl:dataset_statistics} presents the statistics of the SciConvQA dataset. In summary, there are 11,953 turns across 900 conversations, with an average of 11.53 words per query, 15.59 words per response, and 13.28 turns per conversation.

\begin{table}[h!]
\centering
\resizebox{0.5\linewidth}{!}{
\begin{tabular}{l|cc|c}
\toprule
\textbf{Dataset}              & \textbf{Train} & \textbf{Test} & \textbf{Overall} \\
\midrule
\# Turns                      & 9,999               & 1,954         & 11,953           \\
\# Conversations              & 750                & 150           & 900            \\
\# Words / Query          & 11.51                & 11.62          & 11.53             \\
\# Words / Response            & 15.54               & 15.83         & 15.59            \\
\# Turns / Conversation       & 13.33              & 13.03            & 13.28               \\
\bottomrule
\end{tabular}
}
\vspace{1em}
\caption{Dataset statistics of SciConvQA}
\vspace{-1em}
\label{tbl:dataset_statistics}
\end{table}

\smallskip 
\noindent\textbf{Domain Similarity Comparison of Conversational Datasets.} 
The t-SNE visualization in Figure~\ref{fig:vis_dts} offers a qualitative insight into the similarity among three conversational datasets: QReCC, TopiOCQA, and SciConvQA. In general, QReCC and TopiOCQA show significant overlap in the embedding space, suggesting high semantic similarity between these datasets. This is attributed to their common focus on general-domain conversational topics. In contrast, SciConvQA forms a clearly separate cluster due to its specialized domain~(e.g., scientific topics), reflecting the domain difference from the other two datasets. In real-world applications, CQR models are required to handle such diverse datasets across general-domain and specialized-domain conversations.

\begin{figure}[!t]
    \centering
    \includegraphics[width=0.55\linewidth]{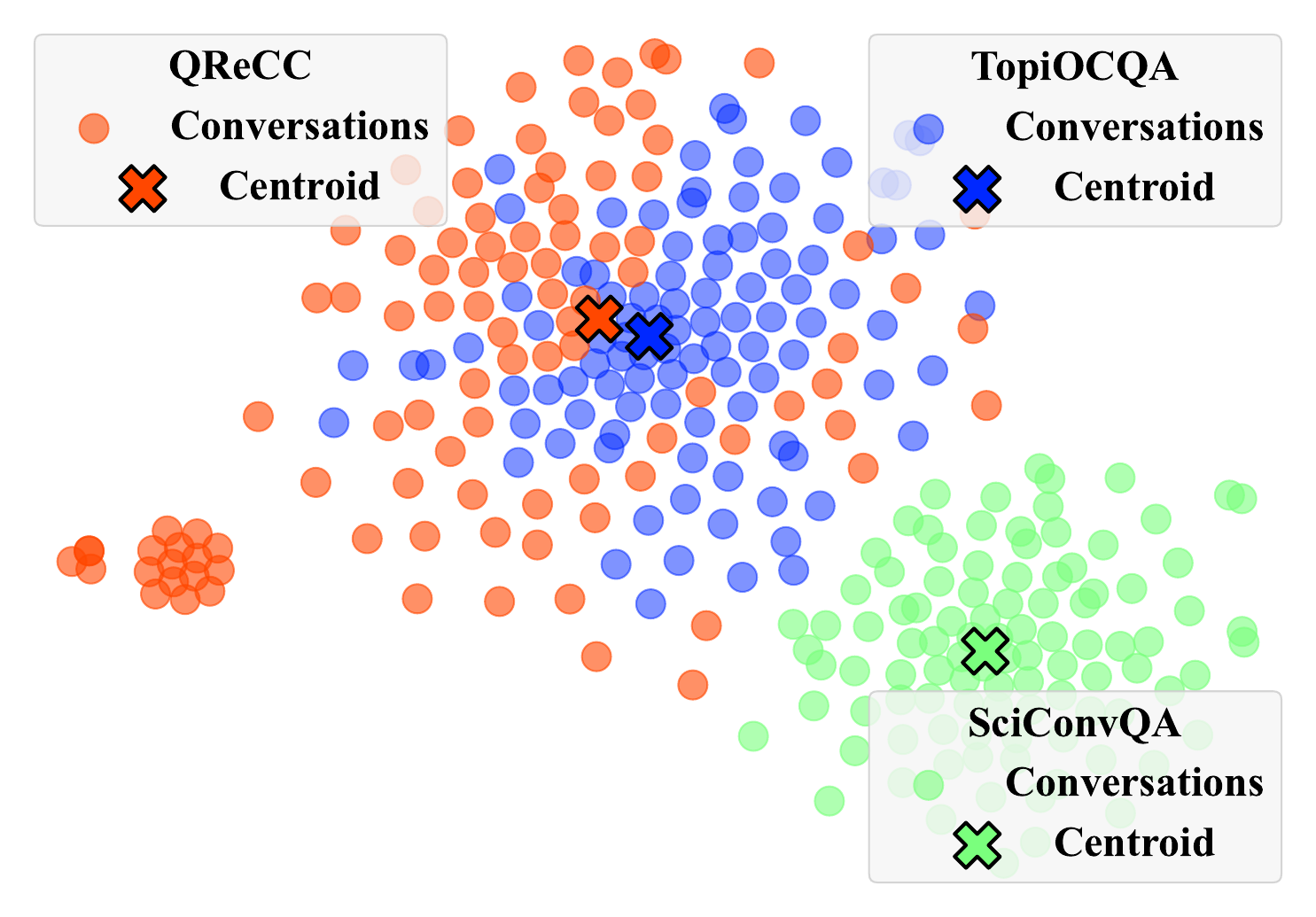}
    \caption{t-SNE visualization\,\citep{tsne} of three conversational datasets: QReCC, TopiOCQA, and SciConvQA. For each dataset, 100 randomly sampled conversations were encoded using the pretrained Sentence Transformer\,\citep{reimers-2019-sentence-bert, gtr, infocqr} and projected into 2D embedding space via t-SNE. The conversations (denoted by \scalebox{0.99}{\color{gray}\CIRCLE} symbols) from the same dataset are in the same color, where the centroid of each dataset's conversations is denoted by a \textbf{\texttimes{}} symbol.}
    \label{fig:vis_dts}
\end{figure}



\section{Experiment Details}
\label{sec:exp_details}

\subsection{Evaluation Metrics}
\label{sec:app_metrics}
\noindent\textbf{Metrics for Retrieval Accuracy.} 
We employ three widely used metrics\,\citep{cdr1, cdr2, cdr3, infocqr}: MRR, NDCG@3, and Recall@$k$. MRR evaluates how well the system ranks the \textit{first} relevant result, with higher scores indicating that a reference passage appears earlier in the ranked list. NDCG@3 measures the overall \textit{alignment} of ranking with with the ideal relevance ordering of reference passages by prioritizing that they are positioned closer to the top of the retrieved passage list. Recall@$k$ assesses the \textit{coverage}, capturing the fraction of reference passages retrieved within the top \(k\) results. Together, these metrics provide a holistic assessment of the system's ability to perform accurate and relevant passage retrieval.

\smallskip 
\noindent\textbf{Metrics for Generation Accuracy.} 
We employ widely used metrics\,\citep{adaptive-rag, baek2023knowledge, selfrag, mallen2022not, adaptive-rag2, bergen_prime}: LLMeval, ROUGE, and BertScore. LLMEval employs gpt-4o-2024-08-06\,\citep{openai2024chatgpt} as the evaluator, providing human-aligned relevance and generative quality assessments. ROUGE, specifically ROUGE-1 and ROUGE-L, evaluates lexical and structural alignment through unigram overlap and longest common subsequences. BERTScore\,\citep{zhang2019bertscore, bert} uses contextual embeddings to measure semantic similarity, enabling robust evaluation beyond surface-level matching.

\subsection{Implementation Details}
\label{sec:app_impl}

All baseline models are trained using their official repositories\,\citep{t5qr, hyde, convgqr, infocqr, llm4cs}, with the exception of RetPo, which we implement because no public code is available.

RetPo's hyperparameters are tuned via grid search for both SFT and DPO: SFT is trained for one epoch with a learning rate of \(2\times10^{-5}\) and a batch size of 32, while DPO uses \(\beta=0.1\) (chosen from \(\{0.1, 0.2, 0.3, 0.4, 0.5\}\)) and trains for two epochs under the same learning rate and batch size. Our implementation of RetPo achieves improved performance over the authors' results.

We adopt the same SFT configuration for \algname{} but set \(\beta=0.5\) during DPO to mitigate overfitting to initial pseudo reference passages. For other hyperparameters, \algname{} updates pseudo reference passages every epoch, repeating this process three times following \citep{noisy_student}, and selects the top-3 relevant passages per query-response turn. Hyperparameter sensitivity analyis is provided in Appendix \ref{sec:app_sensitivity}.

For backbone models, LLM-IQR, HyDE-LLM, and LLM4CS-CoT use gpt-3.5-turbo-0125\,\citep{openai2022chatgpt}, while RetPo, HyDE-FT, and \algname{} use Llama3.1-8b-instruct\,\citep{llama3}. T5QR and ConvGQR follow their official T5-base implementations.

All models were implemented in PyTorch 2.1.2 and trained on NVIDIA RTX A6000 Ada GPUs. The source code is publicly available at \url{https://anonymous.4open.science/r/DualReform}.



\smallskip 
\noindent\textbf{Details of Retrieval Systems.} 
We use Pyserini\,\citep{faiss} and Faiss\,\citep{faiss} for BM25\,\citep{bm25} and GTR\,\citep{gtr} retrieval systems, respectively. For BM25, we adopt the parameter settings from previous studies\,\citep{convgqr, retpo, infocqr}, configuring \(k_1 = 0.82\), \(b = 0.68\) for QReCC, and \(k_1 = 0.9\), \(b = 0.4\) for TopiOCQA and SciConvQA, where \(k_1\) adjusts term frequency normalization and \(b\) controls the impact of document frequency. For GTR\footnote{\url{https://huggingface.co/sentence-transformers/gtr-t5-large}}, the maximum token length is set to 384 for both the reformulated query and passage. 

Both sparse and dense retrieval systems retrieve the top-100 relevant passages per query, and the aforementioned metrics are computed using \textit{pytrec-eval}\,\citep{pytrec}.

\smallskip 
\noindent\textbf{Details of Response Generation.} 
We employ Llama-3.1-8b-instruct\,\citep{llama3} as the response generator. The top-4 most relevant passages, retrieved using BM25, are appended to the original query. The input query to BM25 is obtained by applying various CQR methods. 

\smallskip 
\noindent\textbf{Candidate Query Generation.}
To generate diverse candidate queries, we build on prior studies\,\citep{retpo, adacqr}. Specifically, we utilize the gpt-3.5-turbo-0125\,\citep{openai2022chatgpt} via the OpenAI API\footnote{\url{https://platform.openai.com/docs/models/gpt-3-5-turbo}} to transform user queries in conversational datasets into diverse candidate queries. The model is configured with a temperature of 0.8 and a top-$p$ value of 0.8 to promote diversity, with a maximum token limit set to 2560.

We adopt two prompting strategies: Question Rewriting and Query Expansion. Question Rewriting generates 12 candidate queries, whereas Query Expansion produces 3 additional candidates by applying Llama3.1-8b-instruct to the outputs of Question Rewriting. Both strategies are applied consistently across all datasets. The prompt templates for Question Rewriting\,(adapted from \citep{infocqr}) and Query Expansion\,(adapted from \citep{retpo}) are illustrated in Figure~\ref{fig:prompt_qr} and Figure~\ref{fig:prompt_qe}, respectively. 
\section{Complete Experiment Results}
\label{sec:exp_results}

\begin{figure*}[!t]
    \centering
    \includegraphics[width=\textwidth]{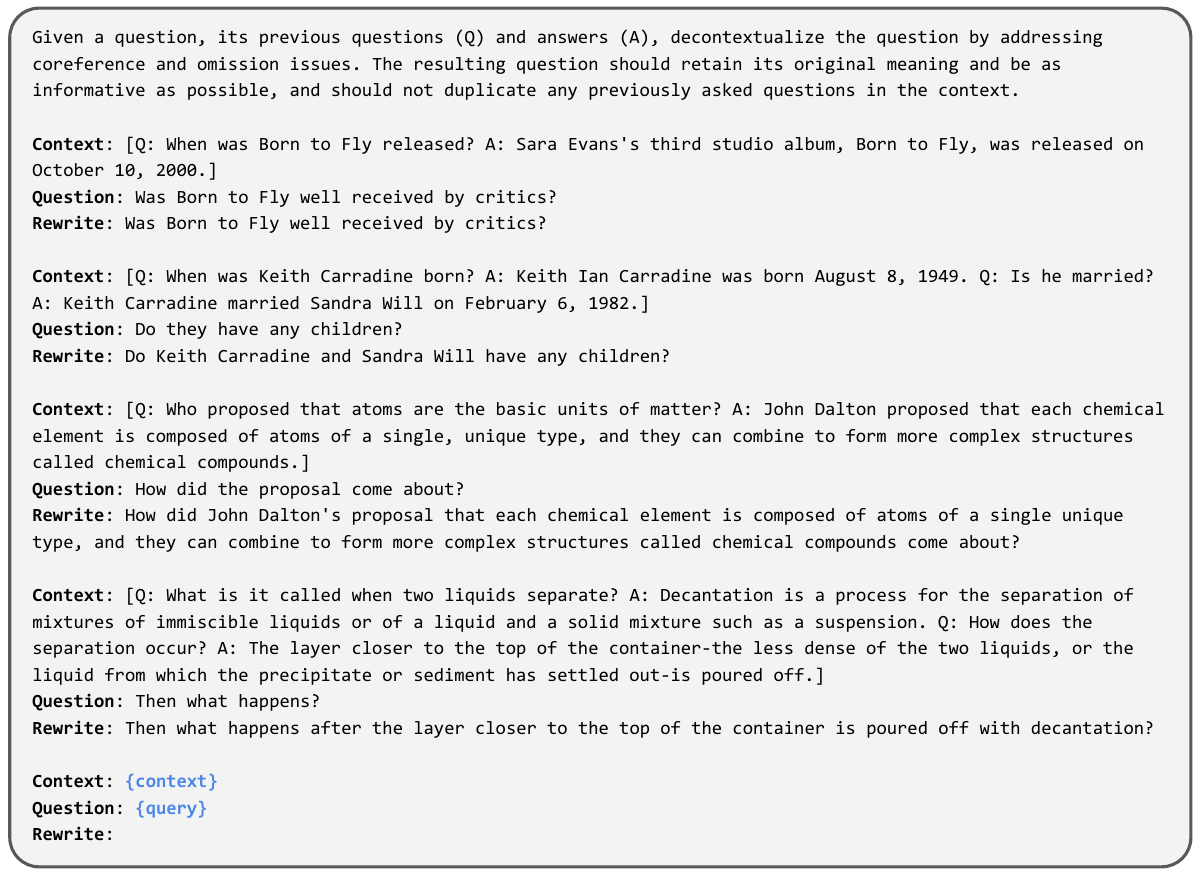}
    \vspace*{-0.5cm}
    \caption{Prompt used for Query Rewriting.}
    \label{fig:prompt_qr}
    \vspace*{-0.3cm}
\end{figure*}

\begin{figure*}[!h]
    \centering
    \includegraphics[width=\textwidth]{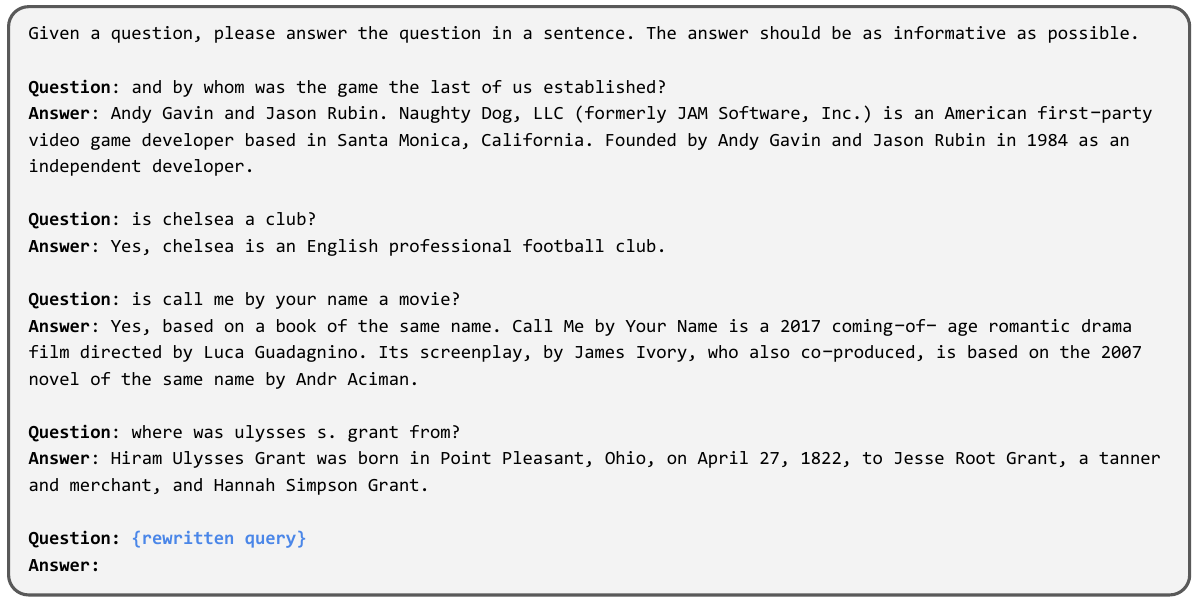}
    \vspace*{-0.5cm}
    \caption{Prompt used for Query Expansion.}
    \label{fig:prompt_qe}
    \vspace*{-0.3cm}
\end{figure*}

\subsection{Extended Results: QReCC}
\label{sec:app_qrecc}
Table~\ref{tbl:overall_perf_qrecc} extends the results of Table~\ref{tbl:overall_perf} by additionally using the QReCC dataset as the target dataset.

The results further highlight the importance of reference-free preference optimization by demonstrating that Upper Bound and RetPo, which leverage preference optimization, exhibit divergent performance depending on the availability of reference passages from the target dataset. Notably, \algname{} establishes itself as an effective approach for \emph{reference-free} preference optimization, outperforming the baseline methods and achieving performance levels comparable to Upper Bound.
\def\arraystretch{0.96}
\begin{table*}[h]
\caption{Retrieval performance comparison of \algname{} against representative CQR baselines on the target dataset, QReCC. The best results~(excluding Upper Bound) are highlighted in bold.}
\centering
\resizebox{0.95\linewidth}{!}{%
\begin{tabular}[c]
{@{}ccc|cccc|ccccc@{}}
\arrayrulecolor{black}\specialrule{2pt}{0.75pt}{1.0pt}
\multirow{2}{*}{\hspace{-0.0cm}\makecell[c]{\textbf{Target}\\\textbf{Dataset}}\hspace{-0.0cm}}&\multicolumn{2}{c|}{\multirow{2}{*}{\makecell[c]{\textbf{Query}\\\textbf{Reformulations}}}}
& \multicolumn{4}{c|}{{ \textbf{Sparse Retriever}}} 
& \multicolumn{4}{c}{{ \textbf{Dense Retriever}}}\\
& & & \multicolumn{1}{c}{\textbf{MRR}}
& \multicolumn{1}{c}{\textbf{\!\!\!NDCG\!\!\!}}
& \multicolumn{1}{c}{\textbf{R@5}}
& \multicolumn{1}{c|}{\textbf{R@20}}
& \multicolumn{1}{c}{\textbf{MRR}}
& \multicolumn{1}{c}{\textbf{\!\!\!NDCG\!\!\!}}
& \multicolumn{1}{c}{\textbf{R@5}}
& \multicolumn{1}{c}{\textbf{R@20}}
\\ 
\arrayrulecolor{black}\specialrule{1pt}{1.5pt}{0.7pt} 
\arrayrulecolor{black}\specialrule{1pt}{0.7pt}{3.0pt}
\multirow{10}{*}{\hspace{-0.0cm}{\textbf{QReCC}}} 
& \multicolumn{2}{c|}{{ {Upper Bound}}} 
& \begin{tabular}[c]{@{}c@{}}{ {50.32} }\end{tabular}
& \begin{tabular}[c]{@{}c@{}}{ {46.94} }\end{tabular}
& \begin{tabular}[c]{@{}c@{}}{ {58.97} }\end{tabular}
& \begin{tabular}[c]{@{}c@{}}{ {78.96} }\end{tabular}
& \begin{tabular}[c]{@{}c@{}}{ {51.63} }\end{tabular}
& \begin{tabular}[c]{@{}c@{}}{ {49.11} }\end{tabular}
& \begin{tabular}[c]{@{}c@{}}{ {63.13} }\end{tabular}
& \begin{tabular}[c]{@{}c@{}}{ {79.64} }\end{tabular}
\\ \addlinespace[0.1ex]\cline{2-11}\addlinespace[0.5ex]
& \multicolumn{2}{c|}{{ {LLM-IQR}}} 
& \begin{tabular}[c]{@{}c@{}}{ {41.82} }\end{tabular}
& \begin{tabular}[c]{@{}c@{}}{ {38.88} }\end{tabular}
& \begin{tabular}[c]{@{}c@{}}{ {52.58} }\end{tabular}
& \begin{tabular}[c]{@{}c@{}}{ {71.95} }\end{tabular}
& \begin{tabular}[c]{@{}c@{}}{ {48.09} }\end{tabular}
& \begin{tabular}[c]{@{}c@{}}{ {45.25} }\end{tabular}
& \begin{tabular}[c]{@{}c@{}}{ \textbf{62.13} }\end{tabular}
& \begin{tabular}[c]{@{}c@{}}{ {80.24} }\end{tabular}
\\
& \multicolumn{2}{c|}{{ {HyDE-LLM}}} 
& \begin{tabular}[c]{@{}c@{}}{ {42.26} }\end{tabular}
& \begin{tabular}[c]{@{}c@{}}{ {39.21} }\end{tabular}
& \begin{tabular}[c]{@{}c@{}}{ {52.93} }\end{tabular}
& \begin{tabular}[c]{@{}c@{}}{ {72.44} }\end{tabular}
& \begin{tabular}[c]{@{}c@{}}{ {48.20} }\end{tabular}
& \begin{tabular}[c]{@{}c@{}}{ {45.46} }\end{tabular}
& \begin{tabular}[c]{@{}c@{}}{ {61.97} }\end{tabular}
& \begin{tabular}[c]{@{}c@{}}{ {79.85} }\end{tabular}
\\ 
& \multicolumn{2}{c|}{{ {LLM4CS-CoT}}} 
& \begin{tabular}[c]{@{}c@{}}{ {47.51} }\end{tabular}
& \begin{tabular}[c]{@{}c@{}}{ {44.25} }\end{tabular}
& \begin{tabular}[c]{@{}c@{}}{ {56.64} }\end{tabular}
& \begin{tabular}[c]{@{}c@{}}{ {78.96} }\end{tabular}
& \begin{tabular}[c]{@{}c@{}}{ \textbf{48.53} }\end{tabular}
& \begin{tabular}[c]{@{}c@{}}{ {45.49} }\end{tabular}
& \begin{tabular}[c]{@{}c@{}}{ {58.54} }\end{tabular}
& \begin{tabular}[c]{@{}c@{}}{ {79.64} }\end{tabular}
\\ 
& \multicolumn{2}{c|}{{ {T5QR}}} 
& \begin{tabular}[c]{@{}c@{}}{ {32.19} }\end{tabular}
& \begin{tabular}[c]{@{}c@{}}{ {29.17} }\end{tabular}
& \begin{tabular}[c]{@{}c@{}}{ {40.18} }\end{tabular}
& \begin{tabular}[c]{@{}c@{}}{ {61.94} }\end{tabular}
& \begin{tabular}[c]{@{}c@{}}{ {41.21} }\end{tabular}
& \begin{tabular}[c]{@{}c@{}}{ {38.18} }\end{tabular}
& \begin{tabular}[c]{@{}c@{}}{ {54.63} }\end{tabular}
& \begin{tabular}[c]{@{}c@{}}{ {73.45} }\end{tabular}
\\ \addlinespace[0.5ex]\cdashline{2-11}\addlinespace[0.7ex]
& \multirow{3}{*}{\hspace{-0.0cm}{\makecell[c]{\small\textit{SciConvQA}\\ \small $\downarrow$ \\ \small\textit{QReCC}}}\hspace{-0cm}} 
& \multirow{1}{*}{\hspace{-0.1cm}\makecell[c]{ConvGQR}\hspace{-0.cm}}
& \begin{tabular}[c]{@{}c@{}}{ {32.49} }\end{tabular}
& \begin{tabular}[c]{@{}c@{}}{ {29.66} }\end{tabular}
& \begin{tabular}[c]{@{}c@{}}{ {41.36} }\end{tabular}
& \begin{tabular}[c]{@{}c@{}}{ {59.47} }\end{tabular}
& \begin{tabular}[c]{@{}c@{}}{ {36.86} }\end{tabular}
& \begin{tabular}[c]{@{}c@{}}{ {34.16} }\end{tabular}
& \begin{tabular}[c]{@{}c@{}}{ {48.94} }\end{tabular}
& \begin{tabular}[c]{@{}c@{}}{ {67.72} }\end{tabular}
\\
& & \multirow{1}{*}{\hspace{-0.1cm}\makecell[c]{HyDE-FT}\hspace{-0.cm}}
& \begin{tabular}[c]{@{}c@{}}{ {38.87} }\end{tabular}
& \begin{tabular}[c]{@{}c@{}}{ {36.09} }\end{tabular}
& \begin{tabular}[c]{@{}c@{}}{ {47.41} }\end{tabular}
& \begin{tabular}[c]{@{}c@{}}{ {64.62} }\end{tabular}
& \begin{tabular}[c]{@{}c@{}}{ {41.23} }\end{tabular}
& \begin{tabular}[c]{@{}c@{}}{ {37.92} }\end{tabular}
& \begin{tabular}[c]{@{}c@{}}{ {52.55} }\end{tabular}
& \begin{tabular}[c]{@{}c@{}}{ {68.76} }\end{tabular}
\\
& & \multirow{1}{*}{\hspace{-0.1cm}\makecell[c]{RetPo}\hspace{-0.cm}}
& \begin{tabular}[c]{@{}c@{}}{ {39.22} }\end{tabular}
& \begin{tabular}[c]{@{}c@{}}{ {36.49} }\end{tabular}
& \begin{tabular}[c]{@{}c@{}}{ {48.76} }\end{tabular}
& \begin{tabular}[c]{@{}c@{}}{ {65.07} }\end{tabular}
& \begin{tabular}[c]{@{}c@{}}{ {45.44} }\end{tabular}
& \begin{tabular}[c]{@{}c@{}}{ {42.73} }\end{tabular}
& \begin{tabular}[c]{@{}c@{}}{ {59.51} }\end{tabular}
& \begin{tabular}[c]{@{}c@{}}{ {78.32} }\end{tabular}
\\ \addlinespace[0.3ex]\cdashline{2-11}\addlinespace[0.8ex]
& \multicolumn{2}{c|}{{ \textbf{\algname{}}}} 
& \begin{tabular}[c]{@{}c@{}}{ \textbf{48.40} }\end{tabular}
& \begin{tabular}[c]{@{}c@{}}{ \textbf{45.30} }\end{tabular}
& \begin{tabular}[c]{@{}c@{}}{ \textbf{58.91} }\end{tabular}
& \begin{tabular}[c]{@{}c@{}}{ \textbf{77.90} }\end{tabular}
& \begin{tabular}[c]{@{}c@{}}{ {47.58} }\end{tabular}
& \begin{tabular}[c]{@{}c@{}}{ \textbf{46.98} }\end{tabular}
& \begin{tabular}[c]{@{}c@{}}{ {59.60} }\end{tabular}
& \begin{tabular}[c]{@{}c@{}}{ \textbf{80.33} }\end{tabular}
\\ 
\arrayrulecolor{black}\specialrule{2pt}{1.0pt}{1.0pt}
\end{tabular}
}
\label{tbl:overall_perf_qrecc}
\end{table*}



\subsection{Extended Results: Effect of Pseudo Reference Refinement via CQR}
\label{sec:app_refine_methods}


Table~\ref{tbl:effct_cqr_app} extends the results of  Table~\ref{tbl:effct_cqr} by additionally using other evaluation metrics,  NDCG@3 and Recall@20, in order to offer a more comprehensive assessment of retrieval performance.
\begin{table}[h]
\vspace{-1.1em} 
\caption{Comparison of response refinement methods, evaluated using the sparse retriever. The highest values are emphasized in bold.} 
\def\arraystretch{1.28}
\centering
\resizebox{0.47\linewidth}{!}{%
\begin{tabular}[c]
{@{}cc|cc|ccc@{}}
\arrayrulecolor{black}\specialrule{1.2pt}{0.75pt}{2.5pt}
\multirow{2}{*}{\hspace{-0.0cm}{\makecell[c]{ Data }}\hspace{-0cm}}  & \multirow{2}{*}{\hspace{-0.0cm}{\makecell[c]{\textbf{Refine}\\ \textbf{Methods}}}\hspace{-0cm}} 

& \multicolumn{2}{c|}{{ \textbf{Pseudo Ref. Acc.}}} 
& \multicolumn{2}{c}{{ \textbf{Retrieval Acc.}}} 
\\ 
& & \multicolumn{1}{c}{{NDCG}}
& \multicolumn{1}{c|}{{R@20}}
& \multicolumn{1}{c}{{NDCG}}
& \multicolumn{1}{c}{{R@20}}
\\
\arrayrulecolor{black}\specialrule{1pt}{1.5pt}{1pt}
\arrayrulecolor{black}\specialrule{1pt}{1pt}{3.0pt}
\multirow{3}{*}{\vspace{0.4cm}\rotatebox{90}{\small SciConvQA}} 
& \multirow{1}{*}{{Llama}}
& \begin{tabular}[c]{@{}c@{}}{ {35.62} }\end{tabular}
& \begin{tabular}[c]{@{}c@{}}{ {59.37} }\end{tabular}
& \begin{tabular}[c]{@{}c@{}}{ {16.53} }\end{tabular}
& \begin{tabular}[c]{@{}c@{}}{ {34.03} }\end{tabular}
\\
&  \multirow{1}{*}{{Llama+ICL}}                                
& \begin{tabular}[c]{@{}c@{}}{ {43.99} }\end{tabular}
& \begin{tabular}[c]{@{}c@{}}{ {66.00} }\end{tabular}
& \begin{tabular}[c]{@{}c@{}}{ {17.56} }\end{tabular}
& \begin{tabular}[c]{@{}c@{}}{ {42.32} }\end{tabular}
\\ 
&  \multirow{1}{*}{{{\algname{}}}}            
& \begin{tabular}[c]{@{}c@{}}{ \textbf{49.36} }\end{tabular}
& \begin{tabular}[c]{@{}c@{}}{ \textbf{71.27} }\end{tabular}
& \begin{tabular}[c]{@{}c@{}}{ \textbf{18.88} }\end{tabular}
& \begin{tabular}[c]{@{}c@{}}{ \textbf{42.78} }\end{tabular}
\\ \addlinespace[0.3ex]\cdashline{1-7}\addlinespace[0.8ex] 
\multirow{3}{*}{\vspace{0.4cm}\rotatebox{90}{\small TopiOCQA}} 
& \multirow{1}{*}{{Llama}}              
& \begin{tabular}[c]{@{}c@{}}{ {35.42} }\end{tabular}
& \begin{tabular}[c]{@{}c@{}}{ {50.40} }\end{tabular}
& \begin{tabular}[c]{@{}c@{}}{ {22.98} }\end{tabular}
& \begin{tabular}[c]{@{}c@{}}{ {53.62} }\end{tabular}
\\
&  \multirow{1}{*}{{Llama+ICL}}                                
& \begin{tabular}[c]{@{}c@{}}{ {50.24} }\end{tabular}
& \begin{tabular}[c]{@{}c@{}}{ {67.86} }\end{tabular}
& \begin{tabular}[c]{@{}c@{}}{ {25.76} }\end{tabular}
& \begin{tabular}[c]{@{}c@{}}{ {54.02} }\end{tabular}
\\ 
&  \multirow{1}{*}{{{\algname{}}}}            
& \begin{tabular}[c]{@{}c@{}}{ \textbf{56.23} }\end{tabular}
& \begin{tabular}[c]{@{}c@{}}{ \textbf{76.82} }\end{tabular}
& \begin{tabular}[c]{@{}c@{}}{ \textbf{26.57} }\end{tabular}
& \begin{tabular}[c]{@{}c@{}}{ \textbf{59.03} }\end{tabular}
\\
\arrayrulecolor{black}\specialrule{1.2pt}{1pt}{1.0pt}
\end{tabular} 
}
\label{tbl:effct_cqr_app}
\vspace*{-0.1cm}
\end{table}

\smallskip
\noindent\textbf{Comparison with the Substantially Larger Backbone Language Model, Llama3.1-70b-inst.} 
In the Llama and Llama+ICL variants, we replace their backbones with a larger backbone, \textit{Llama3.1-\textit{70b}-inst}, whereas \algname{} keeps using \textit{Llama3.1-{8b}-inst} as the backbone for its CQR model. As shown in Table~\ref{tbl:effct_cqr_llama70B}, \algname{} achieves comparable or marginally superior performance despite using a substantially smaller language model as its backbone. This finding demonstrates that CQR can effectively substitute widely adopted LLMs for response refinement without introducing additional computational overhead.
\begin{table}[h!]
\vspace{-0.5em} 
\caption{Comparison with the substantially larger Llama3.1-70b-inst backbone. Both Llama and Llama+ICL use Llama3.1-70b-inst, while \algname{} employs the smaller Llama3.1-8b-inst for CQR. The highest values are emphasized in bold.} 
\def\arraystretch{1.25}
\centering
\resizebox{0.65\linewidth}{!}{%
\begin{tabular}[c]
{@{}ccc|ccccc@{}}
\arrayrulecolor{black}\specialrule{1.2pt}{0.75pt}{2.5pt}
\multirow{2}{*}{\hspace{-0.0cm}{\makecell[c]{ Data }}\hspace{-0cm}}  
& \multirow{2}{*}{\hspace{-0.0cm}{\makecell[c]{\textbf{Backbone}\\ \textbf{Models}}}\hspace{-0cm}} 
& \multirow{2}{*}{\hspace{-0.0cm}{\makecell[c]{{Refine}\\ {Methods}}}\hspace{-0cm}} 
& \multicolumn{4}{c}{{ \textbf{Pseudo Reference Accuracy}}} 
\\ 
& & & \multicolumn{1}{c}{{MRR}}
& \multicolumn{1}{c}{{NDCG}}
& \multicolumn{1}{c}{{R@5}}
& \multicolumn{1}{c}{{R@20}}
\\
\arrayrulecolor{black}\specialrule{1pt}{1.5pt}{1pt}
\arrayrulecolor{black}\specialrule{1pt}{1pt}{3.0pt}
\multirow{3}{*}{\vspace{-0.3cm}\rotatebox{90}{\small SciConvQA}} 
& \multirow{2}{*}{\textbf{Llama3.1-70B-inst}}
& \multirow{1}{*}{{Llama}}
& \begin{tabular}[c]{@{}c@{}}{{47.61}}\end{tabular}
& \begin{tabular}[c]{@{}c@{}}{{46.91}}\end{tabular}
& \begin{tabular}[c]{@{}c@{}}{{57.14}}\end{tabular}
& \begin{tabular}[c]{@{}c@{}}{{68.28}}\end{tabular}
\\
&
&  \multirow{1}{*}{{Llama+ICL}}                                
& \begin{tabular}[c]{@{}c@{}}{{49.64}}\end{tabular}
& \begin{tabular}[c]{@{}c@{}}{{48.87}}\end{tabular}
& \begin{tabular}[c]{@{}c@{}}{{58.46}}\end{tabular}
& \begin{tabular}[c]{@{}c@{}}{{69.90}}\end{tabular}
\\ \addlinespace[0.3ex]\cdashline{2-7}\addlinespace[0.8ex]
& \multirow{1}{*}{\textbf{Llama3.1-8B-inst}}
& \multirow{1}{*}{{\textbf{\algname{}}}}            
& \begin{tabular}[c]{@{}c@{}}{\textbf{50.05}}\end{tabular}
& \begin{tabular}[c]{@{}c@{}}{\textbf{49.36}}\end{tabular}
& \begin{tabular}[c]{@{}c@{}}{\textbf{59.75}}\end{tabular}
& \begin{tabular}[c]{@{}c@{}}{\textbf{71.27}}\end{tabular}
\\ \addlinespace[0.3ex]\cline{1-7}\addlinespace[0.8ex] 
\multirow{3}{*}{\vspace{-0.3cm}\rotatebox{90}{\small TopiOCQA}} 
& \multirow{2}{*}{\textbf{Llama3.1-70B-inst}}
& \multirow{1}{*}{{Llama}}              
& \begin{tabular}[c]{@{}c@{}}{{51.00}}\end{tabular}
& \begin{tabular}[c]{@{}c@{}}{{50.51}}\end{tabular}
& \begin{tabular}[c]{@{}c@{}}{{60.74}}\end{tabular}
& \begin{tabular}[c]{@{}c@{}}{{71.09}}\end{tabular}
\\
&
&  \multirow{1}{*}{{Llama+ICL}}                                
& \begin{tabular}[c]{@{}c@{}}{{55.38}}\end{tabular}
& \begin{tabular}[c]{@{}c@{}}{{55.10}}\end{tabular}
& \begin{tabular}[c]{@{}c@{}}{{64.70}}\end{tabular}
& \begin{tabular}[c]{@{}c@{}}{{73.74}}\end{tabular}
\\ \addlinespace[0.3ex]\cdashline{2-7}\addlinespace[0.8ex]
& \multirow{1}{*}{\textbf{Llama3.1-8B-inst}}
&  \multirow{1}{*}{{\textbf{\algname{}}}}            
& \begin{tabular}[c]{@{}c@{}}{\textbf{56.50}}\end{tabular}
& \begin{tabular}[c]{@{}c@{}}{\textbf{56.23}}\end{tabular}
& \begin{tabular}[c]{@{}c@{}}{\textbf{66.79}}\end{tabular}
& \begin{tabular}[c]{@{}c@{}}{\textbf{76.82}}\end{tabular}
\\
\arrayrulecolor{black}\specialrule{1.2pt}{1pt}{1.0pt}
\end{tabular} 
}
\label{tbl:effct_cqr_llama70B}
\end{table}

\smallskip
\noindent\textbf{Demonstrations for Llama+ICL.} 
\textit{Llama+ICL} differs from \textit{Llama} by utilizing in-context demonstrations, as shown in Figure~\ref{fig:demonstration_llama_icl}.
\begin{figure*}[!h]
    \centering
    \includegraphics[width=\textwidth]{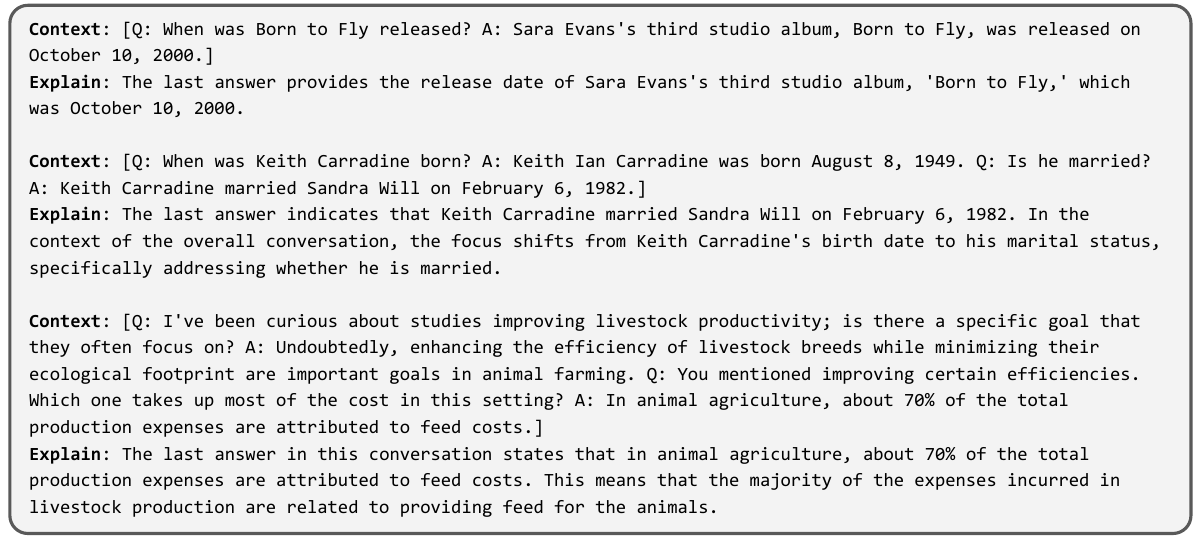}
    \vspace*{-0.5cm}
    \caption{Demonstrations used for the Llama+ICL variant.}
    \label{fig:demonstration_llama_icl}
    \vspace*{-0.3cm}
\end{figure*}

\subsection{Extended Results: Effect of Iterative Refinement for Pseudo Reference}
Table~\ref{tbl:effct_stab_app} extends the results of  Table~\ref{tbl:effct_stab} by additionally using other evaluation metrics, NDCG@3 and Recall@20. These results again demonstrate that \algname{} benefits from the iterative refinement process of pseudo reference passages.
\def\arraystretch{1.3}
\begin{table}[h]
\caption{
Effect of iterative optimization within \algname{}, evaluated using the sparse retriever. The highest values are emphasized in bold.
}

\centering
\resizebox{0.43\linewidth}{!}{%
\begin{tabular}[c]
{@{}cc|cc|ccc@{}}
\arrayrulecolor{black}\specialrule{1.2pt}{0.75pt}{2.5pt}
\multirow{2}{*}{\hspace{-0.cm}{ Data }\hspace{-0.cm}}  & \multirow{2}{*}{\hspace{-0.2cm}\makecell[c]{\textbf{Pseudo Ref.}\\\textbf{Updates}}\hspace{-0.1cm}} 
& \multicolumn{2}{c|}{{ \textbf{Pseudo Ref. Acc.}}} 
& \multicolumn{2}{c}{{ \textbf{Retrieval Acc.}}} 
\\ 
& & \multicolumn{1}{c}{{NDCG}}
& \multicolumn{1}{c|}{{R@20}}
& \multicolumn{1}{c}{{NDCG}}
& \multicolumn{1}{c}{{R@20}}
\\
\arrayrulecolor{black}\specialrule{1pt}{1.5pt}{1pt}
\arrayrulecolor{black}\specialrule{1pt}{1pt}{3.0pt}
\multirow{4}{*}{\vspace{0.4cm}\rotatebox{90}{\small SciConvQA}} 
& \multirow{1}{*}{1}                                
& \begin{tabular}[c]{@{}c@{}}{ {37.48} }\end{tabular}
& \begin{tabular}[c]{@{}c@{}}{ {58.98} }\end{tabular}
& \begin{tabular}[c]{@{}c@{}}{ {15.50} }\end{tabular}
& \begin{tabular}[c]{@{}c@{}}{ {40.23} }\end{tabular}
\\
&  \multirow{1}{*}{2}                                
& \begin{tabular}[c]{@{}c@{}}{ {45.42} }\end{tabular}
& \begin{tabular}[c]{@{}c@{}}{ {65.88} }\end{tabular}
& \begin{tabular}[c]{@{}c@{}}{ {17.96} }\end{tabular}
& \begin{tabular}[c]{@{}c@{}}{ {40.38} }\end{tabular}
\\ 
&  \multirow{1}{*}{{3}}            
& \begin{tabular}[c]{@{}c@{}}{ \textbf{49.36} }\end{tabular}
& \begin{tabular}[c]{@{}c@{}}{ \textbf{71.27} }\end{tabular}
& \begin{tabular}[c]{@{}c@{}}{ \textbf{18.88} }\end{tabular}
& \begin{tabular}[c]{@{}c@{}}{ \textbf{42.78} }\end{tabular}
\\ \addlinespace[0.3ex]\cdashline{1-7}\addlinespace[0.8ex] 
\multirow{4}{*}{\vspace{0.4cm}\rotatebox{90}{\small TopiOCQA}} 
& \multirow{1}{*}{1}                                
& \begin{tabular}[c]{@{}c@{}}{ {38.82} }\end{tabular}
& \begin{tabular}[c]{@{}c@{}}{ {50.75} }\end{tabular}
& \begin{tabular}[c]{@{}c@{}}{ {23.66} }\end{tabular}
& \begin{tabular}[c]{@{}c@{}}{ {52.11} }\end{tabular}
\\
&  \multirow{1}{*}{2}                                
& \begin{tabular}[c]{@{}c@{}}{ {54.86} }\end{tabular}
& \begin{tabular}[c]{@{}c@{}}{ {75.28} }\end{tabular}
& \begin{tabular}[c]{@{}c@{}}{ {26.42} }\end{tabular}
& \begin{tabular}[c]{@{}c@{}}{ {57.25} }\end{tabular}
\\ 
&  \multirow{1}{*}{3}            
& \begin{tabular}[c]{@{}c@{}}{ \textbf{56.23} }\end{tabular}
& \begin{tabular}[c]{@{}c@{}}{ \textbf{76.82} }\end{tabular}
& \begin{tabular}[c]{@{}c@{}}{ \textbf{26.57} }\end{tabular}
& \begin{tabular}[c]{@{}c@{}}{ \textbf{59.03} }\end{tabular}
\\
\arrayrulecolor{black}\specialrule{1.2pt}{1pt}{1.0pt}
\end{tabular} }
\label{tbl:effct_stab_app}
\vspace*{-0.1cm}
\end{table}

\subsection{Effect of Query-Forming Template}
\label{sec:app_template}

Table~\ref{tbl:ablation_template_app} extends the results of  Table~\ref{tbl:ablation_template} by additionally using other evaluation metrics, NDCG@3 and Recall@20. These results further validate that the performance of the variants consistently declines compared to \algname{}, confirming the effectiveness of its query-forming template. 
\def\arraystretch{1.1}
\begin{table}[h]
\caption{
Effect of the query-forming template on pseudo reference accuracy using the sparse retriever. The highest values are emphasized in bold.}
\small
\centering
\resizebox{0.5\linewidth}{!}{%
\begin{tabular}[c]
{@{}c|cc|cc|cc@{}}
\arrayrulecolor{black}\specialrule{1.2pt}{0.75pt}{2.5pt}
\multirow{2}{*}{\hspace{-0.0cm}{{Variants}}} 
& \multicolumn{2}{c|}{{ \textbf{SciConvQA}}} 
& \multicolumn{2}{c|}{{ \textbf{TopiOCQA}}} 
& \multirow{2}{*}{\hspace{-0.0cm}{\textit{Degrade}}} 
\\ 
& \multicolumn{1}{c}{{NDCG}}
& \multicolumn{1}{c|}{{R@20}}
& \multicolumn{1}{c}{{NDCG}}
& \multicolumn{1}{c|}{{R@20}}
\\
\arrayrulecolor{black}\specialrule{1pt}{1.5pt}{1pt}
\arrayrulecolor{black}\specialrule{1pt}{1pt}{3.0pt}
\multirow{1}{*}{\textbf{No Template}}                                
& \begin{tabular}[c]{@{}c@{}}{ {44.56} }\end{tabular}
& \begin{tabular}[c]{@{}c@{}}{ {65.70} }\end{tabular}
& \begin{tabular}[c]{@{}c@{}}{ {54.23} }\end{tabular}
& \begin{tabular}[c]{@{}c@{}}{ {73.45} }\end{tabular}
& \begin{tabular}[c]{@{}c@{}}{ \textit{6.88}\% }\end{tabular}
\\
\multirow{1}{*}{\textbf{No Response}}                                
& \begin{tabular}[c]{@{}c@{}}{ {46.59} }\end{tabular}
& \begin{tabular}[c]{@{}c@{}}{ {68.72} }\end{tabular}
& \begin{tabular}[c]{@{}c@{}}{ {49.93} }\end{tabular}
& \begin{tabular}[c]{@{}c@{}}{ {71.71} }\end{tabular}
& \begin{tabular}[c]{@{}c@{}}{ \textit{7.35}\% }\end{tabular}
\\
\multirow{1}{*}{\textbf{No QE}}                                
& \begin{tabular}[c]{@{}c@{}}{ {42.84} }\end{tabular}
& \begin{tabular}[c]{@{}c@{}}{ {62.75} }\end{tabular}
& \begin{tabular}[c]{@{}c@{}}{ {53.94} }\end{tabular}
& \begin{tabular}[c]{@{}c@{}}{ {72.57} }\end{tabular}
& \begin{tabular}[c]{@{}c@{}}{ \textit{9.72}\% }\end{tabular}
\\
\multirow{1}{*}{\textbf{No QR}}                                
& \begin{tabular}[c]{@{}c@{}}{ {45.90} }\end{tabular}
& \begin{tabular}[c]{@{}c@{}}{ {66.33} }\end{tabular}
& \begin{tabular}[c]{@{}c@{}}{ {46.58} }\end{tabular}
& \begin{tabular}[c]{@{}c@{}}{ {66.58} }\end{tabular}
& \begin{tabular}[c]{@{}c@{}}{ \textit{12.77}\% }\end{tabular}
\\
\midrule
\multirow{1}{*}{{\textbf{\algname{}}}}            
& \begin{tabular}[c]{@{}c@{}}{ \textbf{49.36} }\end{tabular}
& \begin{tabular}[c]{@{}c@{}}{ \textbf{71.27} }\end{tabular}
& \begin{tabular}[c]{@{}c@{}}{ \textbf{56.23} }\end{tabular}
& \begin{tabular}[c]{@{}c@{}}{ \textbf{76.82} }\end{tabular}
& \begin{tabular}[c]{@{}c@{}}{{-}}\end{tabular}
\\
\arrayrulecolor{black}\specialrule{1.2pt}{1pt}{1.0pt}
\end{tabular} }

\label{tbl:ablation_template_app}
\vspace*{-0.1cm}
\end{table}

\subsection{More Examples of Refined Responses}
\label{sec:app_examples}
Figure~\ref{fig:refine_exs_app_marvel} complements Figure~\ref{fig:refine_exs}, presenting the complete refined responses generated by various methods. Figure~\ref{fig:refine_attns} illustrates the top attention weights assigned by \algname{} during response refinement, demonstrating its ability to correctly attend to the relevant context\,(e.g., ``Marvel Studios'') and the target response\,(e.g., ``Iron Man, X-Men, Spider-Man''). More examples for other conversations are provided in Figures~\ref{fig:refine_exs_app_atari}--\ref{fig:refine_exs_app}.

\begin{figure*}[t!]
\includegraphics[width=1.00\linewidth]{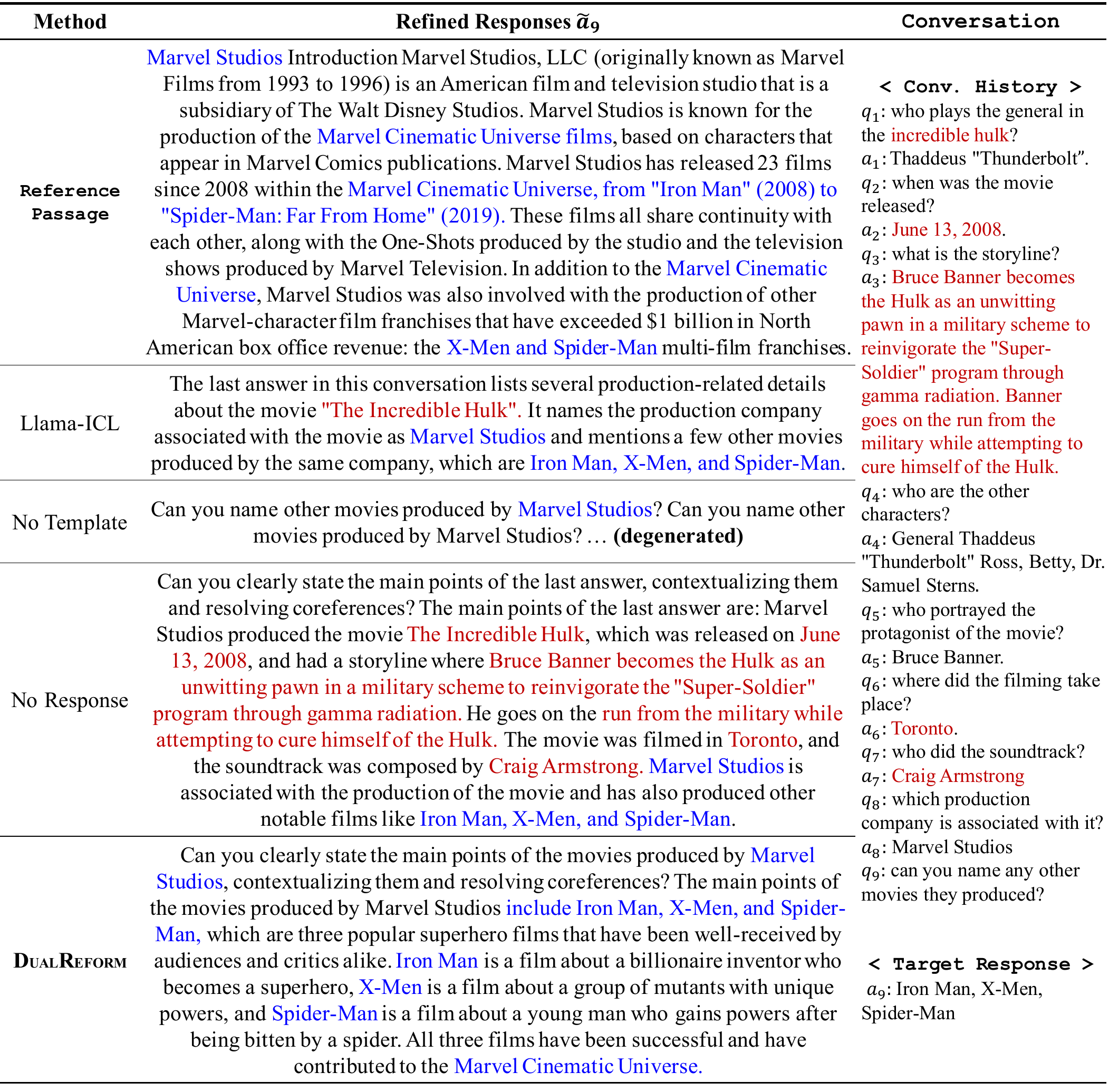}
\vspace*{-0.6cm}
\caption{ 
Examples of refined responses generated by different methods on TopiOCQA. Fragments strongly aligned with the reference passage are highlighted in \textcolor{blue}{blue}, while fragments with weaker connections~(e.g., off-topic elements referring to previous conversation topics) are marked in \textcolor{red}{red}.
}
\vspace*{-0.4cm}
\label{fig:refine_exs_app_marvel}
\end{figure*}

\begin{figure}[h!]
\centering
\includegraphics[width=0.6\columnwidth]{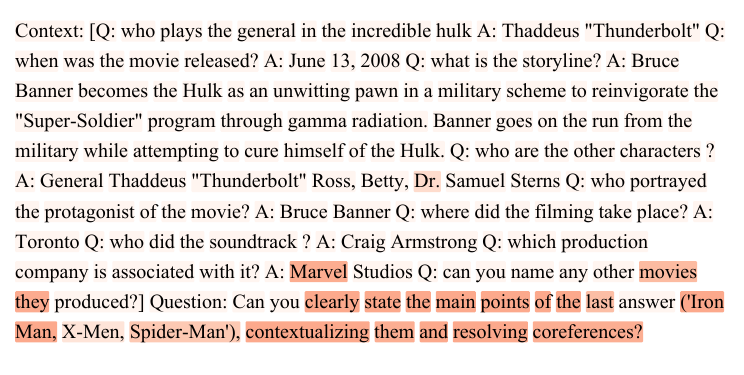}
\vspace*{-0.3cm}
\caption{Top attention weights assigned by \algname{} during response refinement for the conversation in Figure~\ref{fig:refine_exs}, with high-weight regions highlighted in \textcolor{red}{red}.
}
\label{fig:refine_attns}
\end{figure}

\clearpage
\begin{figure*}[t!]
\includegraphics[width=1.00\linewidth]{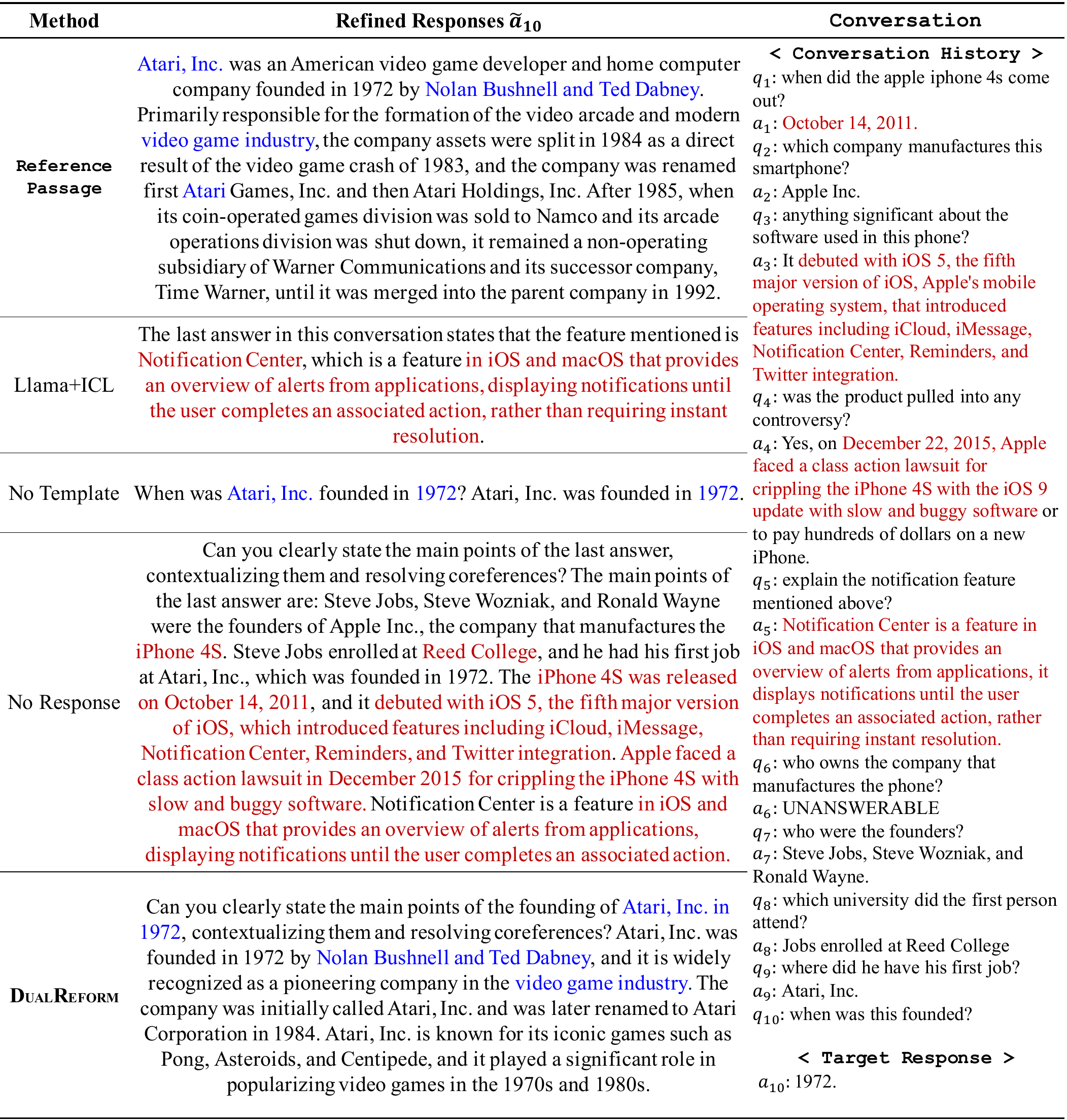}
\vspace*{-0.6cm}
\caption{ 
Examples of refined responses generated by different methods on TopiOCQA. Fragments strongly aligned with the reference passage are highlighted in \textcolor{blue}{blue}, while fragments with weaker connections~(e.g., off-topic elements referring to previous conversation topics) are marked in \textcolor{red}{red}.
}
\vspace*{-0.4cm}
\label{fig:refine_exs_app_atari}
\end{figure*}

\clearpage
\begin{figure*}[t!]
\includegraphics[width=1.00\linewidth]{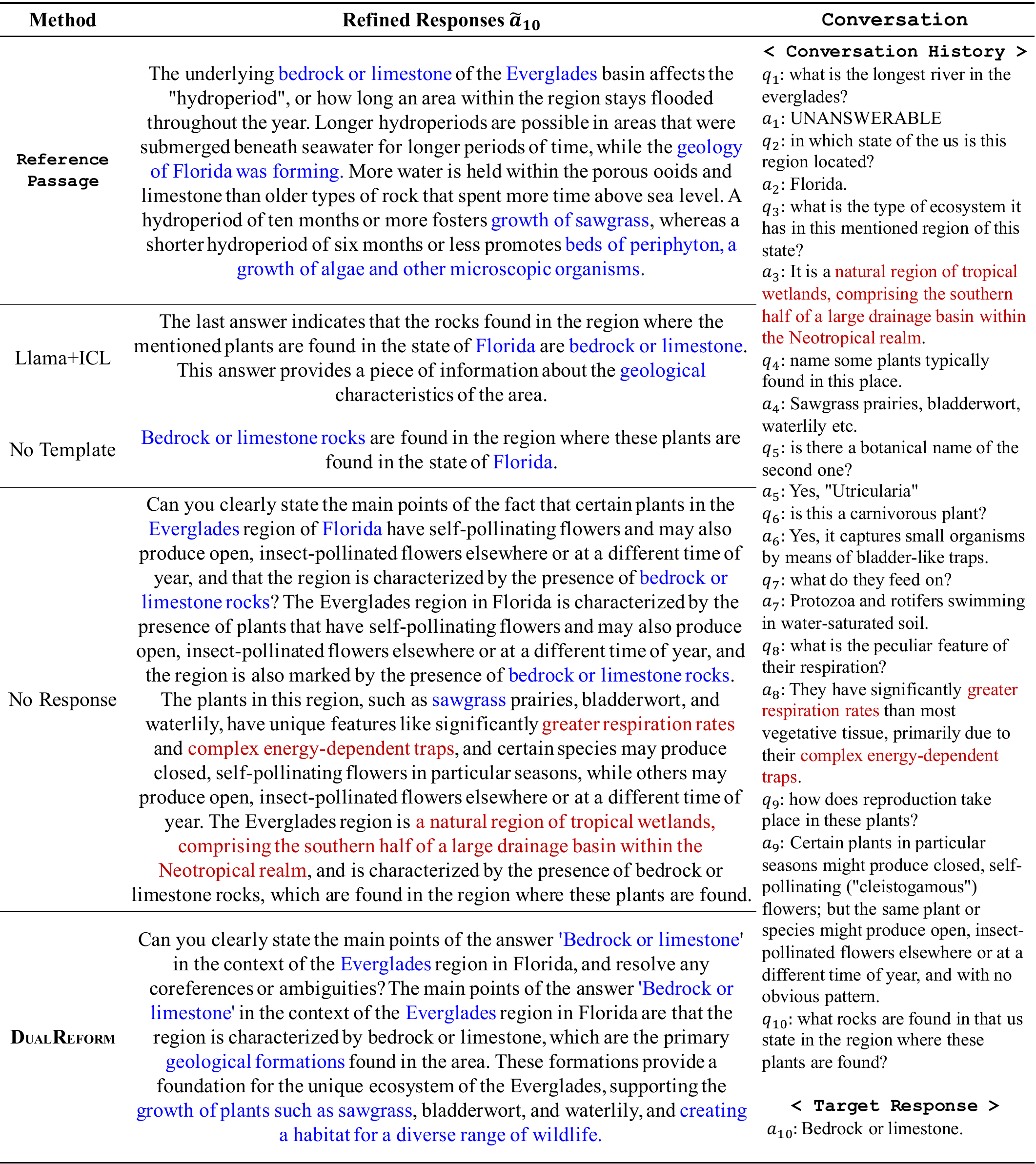}
\vspace*{-0.6cm}
\caption{ 
Examples of refined responses generated by different methods on TopiOCQA. Fragments strongly aligned with the reference passage are highlighted in \textcolor{blue}{blue}, while fragments with weaker connections~(e.g., off-topic elements referring to previous conversation topics) are marked in \textcolor{red}{red}.
}
\vspace*{-0.4cm}
\label{fig:refine_exs_app_everglade}
\end{figure*}

\clearpage

\begin{figure*}[t!]
\includegraphics[width=1.00\linewidth]{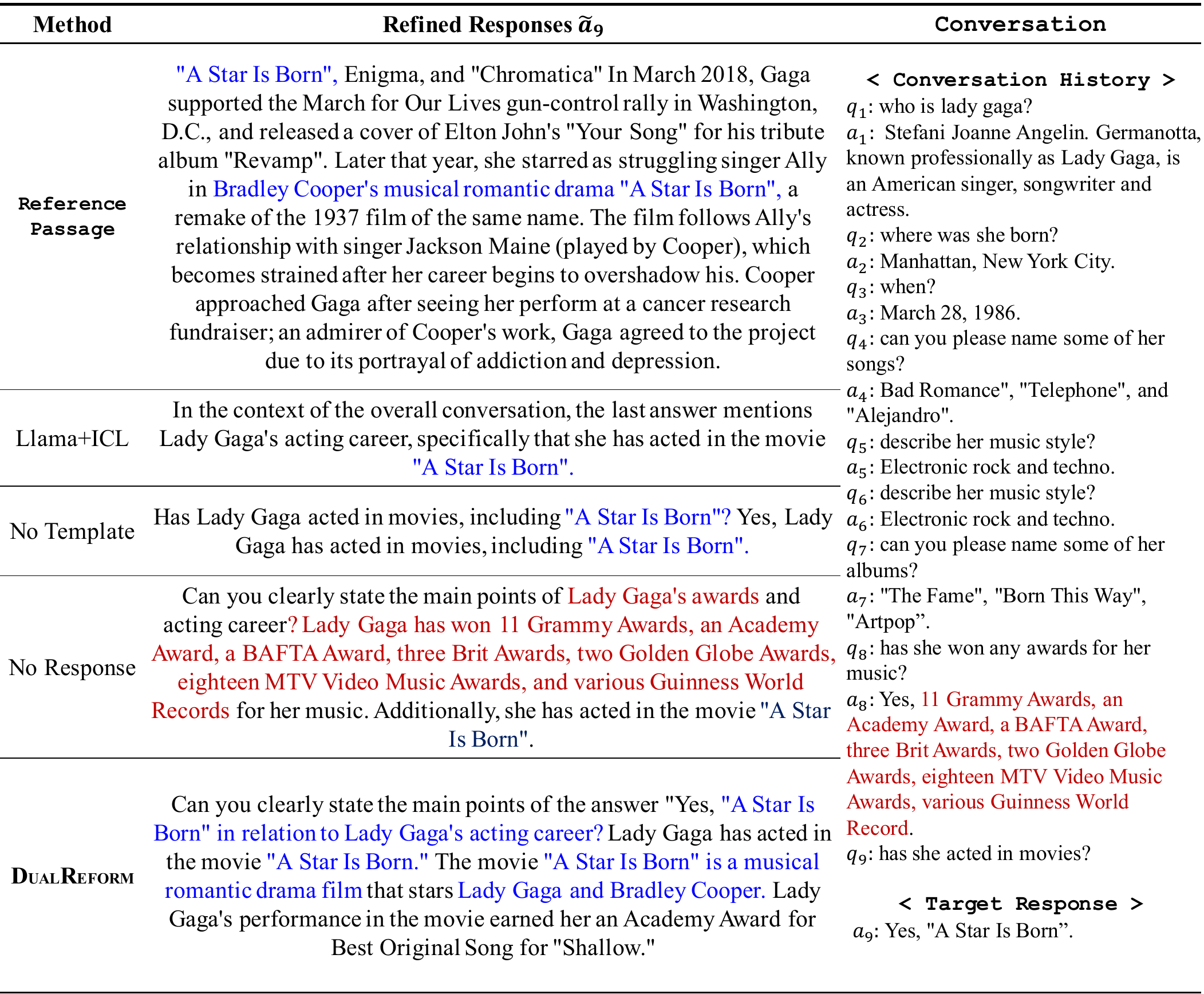}
\vspace*{-0.6cm}
\caption{ 
Examples of refined responses generated by different methods on TopiOCQA. Fragments strongly aligned with the reference passage are highlighted in \textcolor{blue}{blue}, while fragments with weaker connections~(e.g., off-topic elements referring to previous conversation topics) are marked in \textcolor{red}{red}.
}
\vspace*{-0.4cm}
\label{fig:refine_exs_app}
\end{figure*}

\subsection{Generation Accuracy on TopiOCQA}
\label{sec:app_gen}
Table~\ref{tbl:gen_acc_app} extends the results of Table~\ref{tbl:gen_acc} by additionally using the TopiOCQA dataset. These results show that \algname{} consistently outperforms the baselines by achieving higher generation accuracy through accurate passage retrieval across diverse conversational domains.
\def\arraystretch{1.0}
\begin{table}[h]
\caption{
Response generation accuracy with passages retrieved via different CQR methods on TopiOCQA. The highest values are emphasized in bold.}
\small
\centering
\resizebox{0.5\linewidth}{!}{%
\begin{tabular}[c]
{@{}c|ccccc@{}}
\arrayrulecolor{black}\specialrule{1.2pt}{0.75pt}{2.5pt}
\multirow{2}{*}{\hspace{-0.0cm}{{CQR Methods}}} 
& \multicolumn{4}{c}{{ \textbf{Generation Accuracy}}} 
\\ 
& \multicolumn{1}{c}{{LLMEval}}
& \multicolumn{1}{c}{{ROUGE-1}}
& \multicolumn{1}{c}{{ROUGE-L}}
& \multicolumn{1}{c}{{BertScore}}
\\
\arrayrulecolor{black}\specialrule{1pt}{1.5pt}{1pt}
\arrayrulecolor{black}\specialrule{1pt}{1pt}{3.0pt}
\multirow{1}{*}{{LLM-IQR}}                                
& \begin{tabular}[c]{@{}c@{}}{ {25.34} }\end{tabular}
& \begin{tabular}[c]{@{}c@{}}{ {20.36} }\end{tabular}
& \begin{tabular}[c]{@{}c@{}}{ {19.09} }\end{tabular}
& \begin{tabular}[c]{@{}c@{}}{ {84.32} }\end{tabular}
\\
\multirow{1}{*}{{HyDE-LLM}}                                
& \begin{tabular}[c]{@{}c@{}}{ {27.36} }\end{tabular}
& \begin{tabular}[c]{@{}c@{}}{ {24.83} }\end{tabular}
& \begin{tabular}[c]{@{}c@{}}{ {23.30} }\end{tabular}
& \begin{tabular}[c]{@{}c@{}}{ {85.28} }\end{tabular}
\\
\multirow{1}{*}{{LLM4CS-CoT}}                                
& \begin{tabular}[c]{@{}c@{}}{ {33.78} }\end{tabular}
& \begin{tabular}[c]{@{}c@{}}{ {29.20} }\end{tabular}
& \begin{tabular}[c]{@{}c@{}}{ {27.78} }\end{tabular}
& \begin{tabular}[c]{@{}c@{}}{ {85.88} }\end{tabular}
\\
\multirow{1}{*}{{T5QR}}                                
& \begin{tabular}[c]{@{}c@{}}{ {18.92} }\end{tabular}
& \begin{tabular}[c]{@{}c@{}}{ {17.07} }\end{tabular}
& \begin{tabular}[c]{@{}c@{}}{ {16.50} }\end{tabular}
& \begin{tabular}[c]{@{}c@{}}{ {83.67} }\end{tabular}
\\
\multirow{1}{*}{{ConvGQR}}                                
& \begin{tabular}[c]{@{}c@{}}{ {21.62} }\end{tabular}
& \begin{tabular}[c]{@{}c@{}}{ {22.02} }\end{tabular}
& \begin{tabular}[c]{@{}c@{}}{ {20.89} }\end{tabular}
& \begin{tabular}[c]{@{}c@{}}{ {84.84} }\end{tabular}
\\
\multirow{1}{*}{{HyDE-FT}}                                
& \begin{tabular}[c]{@{}c@{}}{ {15.98} }\end{tabular}
& \begin{tabular}[c]{@{}c@{}}{ {15.37} }\end{tabular}
& \begin{tabular}[c]{@{}c@{}}{ {14.72} }\end{tabular}
& \begin{tabular}[c]{@{}c@{}}{ {84.01} }\end{tabular}
\\
\multirow{1}{*}{{RetPo}}                                
& \begin{tabular}[c]{@{}c@{}}{ {33.11} }\end{tabular}
& \begin{tabular}[c]{@{}c@{}}{ {26.01} }\end{tabular}
& \begin{tabular}[c]{@{}c@{}}{ {24.39} }\end{tabular}
& \begin{tabular}[c]{@{}c@{}}{ {85.74} }\end{tabular}
\\
\midrule
\multirow{1}{*}{{\textbf{\algname{}}}}            
& \begin{tabular}[c]{@{}c@{}}{ \textbf{35.81} }\end{tabular}
& \begin{tabular}[c]{@{}c@{}}{ \textbf{31.63} }\end{tabular}
& \begin{tabular}[c]{@{}c@{}}{ \textbf{30.10} }\end{tabular}
& \begin{tabular}[c]{@{}c@{}}{ \textbf{86.48} }\end{tabular}
\\
\arrayrulecolor{black}\specialrule{1.2pt}{1pt}{1.0pt}
\end{tabular} }
\label{tbl:gen_acc_app}
\vspace*{-0.1cm}
\end{table}

\subsection{Parameter Sensitivity Analysis}
\label{sec:app_sensitivity}
We conduct the sensitivity analysis of \algname{}'s hyperparameters, specifically the number of pseudo reference passages $k$ in Definition \ref{def:pseudo-ref} and the regularization parameter $\beta$ of DPO in Eq.~(\ref{eq:pseudo-optim-obj}).

\begin{figure*}[h]
    \centering
    \includegraphics[width=0.5\columnwidth]{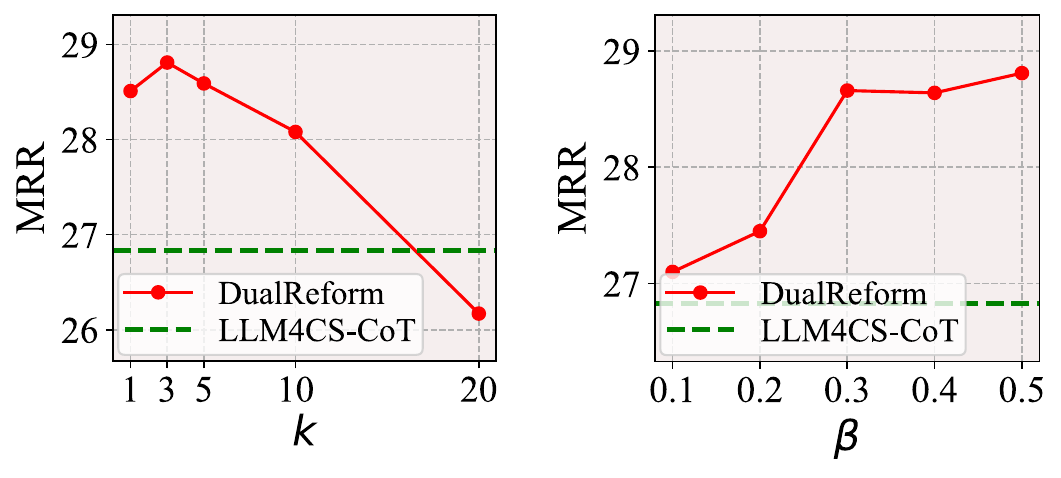}
    \vspace{-0.1cm} \\
    \centering{\small ~~~~~~~~ (a) \# of Pseudo Ref. ~~~~~~~ (b) DPO Regularization.} 
    \caption{Effects of the hyperparameters of \algname{} on MRR. The green dashed line denotes the performance of LLM4CS-CoT, the strongest baseline.}
    \label{fig:param_sensitivity}
\end{figure*}

\smallskip
\noindent\textbf{Number of Pseudo Reference Passages $k$.} 
Figure~\ref{fig:param_sensitivity}(a) presents the impact of \(k\) on retrieval performance. The parameter \(k\) controls the number of top-ranked passages used as pseudo reference passages. Lower values selectively include only the most relevant passages, while higher values introduce additional but less relevant passages. In general, the retrieval accuracy stabilizes between 1 and 5, after which it declines, indicating a negative impact from less relevant passages beyond a specific threshold.

\smallskip
\noindent\textbf{DPO Regularization Parameter $\beta$.} 
Figure~\ref{fig:param_sensitivity}(b) presents retrieval accuracy across different values of \(\beta \in \{0.1,0.2,0.3,0.4,0.5\}\), as guided by prior works\,\citep{dpo_beta, furuta2024geometric}.
This parameter controls the trade-off between aligning the model with user preferences and retaining the behavior of the pre-trained model. Lower values prioritize the former, while higher values place greater emphasis on the latter. In general, increasing \(\beta\) improves performance, with a plateau observed around 0.3 and a peak at 0.5. This trend is attributed to the ability of higher \(\beta\) values to reduce overfitting to the initial pseudo reference passages, further enabling model refinements in \algname{} through iterative optimization.

\clearpage

\end{document}